\documentclass{article}

\usepackage[T1]{fontenc}
\usepackage{microtype}
\usepackage{graphicx}
\usepackage{subcaption}
\usepackage{booktabs} 

\usepackage{hyperref}

\usepackage[accepted]{icml2021}
\usepackage{math}
\usepackage{amsmath,amssymb,amsthm,mathrsfs,amsfonts,dsfont,bm,bbm}
\usepackage{nicefrac}       
\usepackage{microtype}      
\usepackage{algorithm}
\usepackage[noend]{algpseudocode}
\usepackage[table]{xcolor}
\usepackage{mathtools}
\usepackage{tabularx}
\usepackage{multirow}
\usepackage{enumitem}
\setlist[enumerate]{topsep=0pt,itemsep=-1ex,partopsep=1ex,parsep=1ex,leftmargin=*}
\usepackage[page,header]{appendix}
\usepackage{hhline}
\usepackage{xspace}
\usepackage{mathtools}
\usepackage{thmtools,thm-restate}

\newcommand{\M}{\mathcal{M}}

\newcommand{\A}{\mathcal{A}}
\newcommand{\poly}{\text{poly}}
\newcommand{\Cov}{\operatorname{Cov}}

\newtheorem{theorem}{Theorem}[section]
\newtheorem{lemma}{Lemma}[section]
\newtheorem{corollary}{Corollary}[section]

\newtheorem{assumption}{Assumption}[section]
\theoremstyle{definition}
\newtheorem{definition}{Definition}[section]

\newtheorem{remark}{Remark}[section]
\newtheorem*{lemma*}{Lemma}
\newtheorem{proposition}{Proposition}[section]

\newcommand{\Kcal}{\mathcal{K}}
\renewcommand{\S}{\mathcal{S}}

\renewcommand{\E}[2]{\mathbb{E}_{#1}\left[#2\right]}
\newcommand{\interior}[1]{%
  {\kern0pt#1}^{\mathrm{o}}%
}

\newcommand{\1}{\mathbbm{1}}

\newcommand{\Scal}{\mathcal{S}}
\newcommand{\Acal}{\mathcal{A}}
\newcommand{\EE}{\ensuremath{\mathbb{E}}}
\newcommand{\sa}{(s,a)}

\newcommand{\sL}{\mathcal{L}}
\newcommand{\filter}{{\textsc{Filter}}}
\newcommand{\one}{\mathbf{1}}
\DeclarePairedDelimiter{\ceil}{\lceil}{\rceil}

\newcommand{\bR}{\mathbb{R}}

\newcommand{\sever}{{\textsc{Sever}}}

\newcommand{\Ef}{\bar{f}}

\newcommand{\goodset}{I_{\mathrm{good}}}
\newcommand{\Sgood}{S_{\mathrm{good}}}
\newcommand{\Sbad}{S_{\mathrm{bad}}}
\newcommand{\mugood}{\mu_{\mathrm{good}}}
\newcommand{\mubad}{\mu_{\mathrm{bad}}}

\newcommand{\dom}{\mathcal{H}}

\allowdisplaybreaks
\icmltitlerunning{Robust Policy Gradient against Strong Data Corruption}

\begin{document}
\twocolumn[
\icmltitle{Robust Policy Gradient against Strong Data Corruption}




\begin{icmlauthorlist}
\icmlauthor{Xuezhou Zhang}{uwm}
\icmlauthor{Yiding Chen}{uwm}
\icmlauthor{Xiaojin Zhu}{uwm}
\icmlauthor{Wen Sun}{cornell}
\end{icmlauthorlist}

\icmlaffiliation{uwm}{University of Wisconsin--Madison}
\icmlaffiliation{cornell}{Cornell University}

\icmlcorrespondingauthor{Xuezhou Zhang}{xzhang784@wisc.edu}
\icmlcorrespondingauthor{Wen Sun}{ws455@cornell.com}
\icmlkeywords{Machine Learning, ICML}

\vskip 0.3in
]



\printAffiliationsAndNotice{}  
\begin{abstract}
	We study the problem of robust reinforcement learning under adversarial corruption on both rewards and transitions. Our attack model assumes an \textit{adaptive} adversary who can arbitrarily corrupt the reward and transition at every step within an episode, for at most $\epsilon$-fraction of the learning episodes.  Our attack model is strictly stronger than those considered in prior works.
	Our first result shows that no algorithm can find a better than $O(\epsilon)$-optimal policy under our attack model.
	Next, we show that surprisingly the natural policy gradient (NPG) method retains a natural robustness property if the reward corruption is bounded, and can find an $O(\sqrt{\epsilon})$-optimal policy.
	Consequently, we develop a Filtered Policy Gradient (FPG) algorithm that can tolerate even unbounded reward corruption and can find an $O(\epsilon^{1/4})$-optimal policy. We emphasize that FPG is the first that can achieve a meaningful learning guarantee when a constant fraction of episodes are corrupted.
	Complimentary to the theoretical results, we show that a neural implementation of FPG achieves strong robust learning performance on the MuJoCo continuous control benchmarks.
\end{abstract}

\section{Introduction}
Policy gradient methods are a popular class of Reinforcement Learning (RL) methods among practitioners, as they are
amenable to parametric policy classes \cite{schulman2015high, schulman2017proximal}, resilient to modeling assumption mismatches \cite{agarwal2019optimality, agarwal2020pc}, and
they directly optimizing the cost function of interest. However, one current drawback of these methods and most existing RL algorithms is the lack of robustness to data corruption, which severely limits their applications to high-stack decision-making domains with highly noisy data, such as autonomous driving, quantitative trading, or medical diagnosis. 

In fact, data corruption can be a larger threat in the RL paradigm than in traditional supervised learning, because supervised learning is often applied in a controlled environment where data are collected and cleaned by highly-skilled data scientists and domain experts, whereas RL agents are developed to learn in the wild using raw feedbacks from the environment. While the increasing autonomy and less supervision mark a step closer to the goal of general artificial intelligence, they also make the learning system more susceptible to data corruption: autonomous vehicles can misread traffic signs when the signs are contaminated by adversarial stickers \cite{eykholt2018robust}; chatbot can be mistaught by a small group of tweeter users to make misogynistic and racist remarks \cite{neff2016automation}; recommendation systems can be fooled by a small number of fake clicks/reviews/comments to rank products higher than they should be. Despite the many vulnerabilities, \textit{robustness} against data corruption in RL has not been extensively studied only until recently. 

The existing works on \textit{robust} RL are mostly theoretical and can be viewed as a successor of the adversarial bandit literature. However, several drawbacks of this line of approach make them insufficient to modern real-world threats faced by RL agents. We elaborate them below:
\begin{enumerate}[leftmargin=*]
	\item \textbf{Reward vs. transition contamination}: The majority of prior works on adversarial RL
	focus on reward contamination \cite{even2009online, neu2010online, neu2012adversarial, zimin2013online, rosenberg2019online, jin2020learning}, while in reality the adversary often has stronger control during the adversarial interactions. For example, when a chatbot interacts with an adversarial user, the user has full control over both the rewards and transitions during that conversation episode.
	\item \textbf{Density of contamination}: The existing works that do handle adversarial/time-varying transitions can only tolerate \textit{sublinear} number of interactions being corrupted \cite{lykouris2019corruption, cheung2019non, ornik2019learning, ortner2019variational}. These methods would fail when the adversary's attack budget also grows linearly with time, which is often the case in practice.
	\item \textbf{Practicability}: The majority of these work focuses on the setting of tabular MDPs and cannot be applied to real-world RL problems that have large state and action spaces and require function approximations.
\end{enumerate}
In this work, we address the above shortcomings by developing a variant of natural policy gradient (NPG) methods that, under the linear value function assumption, are provably robust against strongly adaptive adversaries, who can \textbf{arbitrarily contaminate} both rewards and transitions in $\epsilon$ fraction of all learning episodes. {Our algorithm does not need to know $\epsilon$, and is adaptive to the contamination level.} Specifically, it guarantees to find an $\tilde O(\epsilon^{1/4})$-optimal policy in a polynomial number of steps. Complementarily, we also present a corresponding lower-bound, showing that no algorithm can consistently find a better than $\Omega(\epsilon)$ optimal policy, even with infinite data. In addition to the theoretical results, we also develop a neural network implementation of our algorithm which is shown to achieve strong robustness performance on the MuJoCo continuous control benchmarks \cite{todorov2012mujoco}, proving that our algorithm can be applied to real-world, high-dimensional RL problems.

\section{Related Work}
\paragraph{RL in standard MDPs.} Learning MDPs with stochastic rewards and transitions is relatively well-studied for the tabular case (that is, a finite number of states and actions).
For example, in the episodic setting, the UCRL2 algorithm \cite{auer2009near} achieves $O(\sqrt{H^4S^2AT})$ regret, where $H$ is the episode length, $S$ is the state space size, $A$ is the action space size, and $T$ is the total number of steps. Later the UCBVI algorithm \cite{azar2017minimax, dann2017unifying} achieves the optimal $O(\sqrt{H^2SAT})$ regret matching the lower-bound \cite{osband2016lower, dann2015sample}. Recent work extends the analysis to various linear setting \cite{jin2020provably,yang2019reinforcement, yang2019sample, zanette2020frequentist, ayoub2020model, zhou2020provably,cai2019provably,du2019provably,kakade2020information} with known linear feature. For unknown feature, \cite{agarwal2020flambe} proposes a sample efficient algorithm that explicitly learns feature representation under the assumption that the transition matrix is low rank.
Beyond the linear settings, there are works assuming the function class has low Eluder dimension which so far is known to be small only for linear functions and generalized linear models \cite{osband2014model}. For more general function approximation, \cite{jiang2017contextual,sun2019model} showed that polynomial sample complexity is achievable as long as the MDP and the given function class together induce low Bellman rank and Witness rank, which include almost all prior models such as tabular MDP, linear MDPs \cite{yang2019reinforcement,jin2020provably}, Kernelized nonlinear regulators \cite{kakade2020information}, low rank MDP \cite{agarwal2020flambe}, and Bellman completion under linear functions \cite{zanette2020frequentist}.

\paragraph{Policy Gradient and Policy Optimization}
Policy Gradient \cite{williams1992simple,sutton1999policy} and Policy optimization methods are widely used in practice \cite{kakade2002approximately,schulman2015high, schulman2017proximal} and have demonstrated amazing performance on challenging applications \cite{berner2019dota,akkaya2019solving}. Unlike model-based approach or Bellman-backup based approaches, PG methods directly optimize the objective function and are often more robust to model-misspecification \cite{agarwal2020pc}. In addition to being robust to model-misspecification, we show in this work that vanilla NPG is also robust to constant fraction and bounded adversarial corruption on both rewards and transitions.

\paragraph{RL with adversarial rewards.}  Almost all prior works on adversarial RL study the setting where the reward functions can be adversarial but the transitions are still stochastic and remain unchanged throughout the learning process. Specifically, at the beginning of each episode, the adversary must decide on a reward function for this episode, and can not change it for the rest of the episode. Also, the majority of these works focus on tabular MDPs. Early works on adversarial MDPs assume a known transition
function and full-information feedback. For example,
\cite{even2009online} proposes the algorithm MDP-E and
proves a regret bound of $\tilde O(\tau \sqrt{T\log A})$ in the non-episodic setting, where $\tau$ is the mixing time of the MDP;
Later, \cite{zimin2013online} consider the episodic
setting and propose the O-REPS algorithm which applies
Online Mirror Descent over the space of occupancy measures, a key component adopted by \cite{rosenberg2019online} and \cite{jin2020learning}. O-REPS achieves the optimal regret
$\tilde O(\sqrt{H^2T\log(SA)})$ in this setting.
Several works consider the harder bandit feedback model while still assuming known transitions. The work \cite{neu2010online}
achieves regret $\tilde O(\sqrt{H^3AT}/\alpha)$ assuming that all states are reachable with some probability $\alpha$ under all policies. Later, \cite{neu2010online} eliminates the dependence on $\alpha$
but only achieves $O(T^{2/3})$ regret. The O-REPS algorithm
of \cite{zimin2013online} again achieves the optimal regret $\tilde O(\sqrt{H^3SAT})$.
To deal with unknown transitions, \cite{neu2012adversarial} proposes
the Follow the Perturbed Optimistic Policy algorithm and achieves $\tilde O(\sqrt{H^2S^2A^2T})$ regret given full-information feedback. Combining the idea of confidence sets and Online Mirror Descent, the UC-O-REPS algorithm of \cite{rosenberg2019online} improves the
regret to $\tilde O(\sqrt{H^2S^2AT})$.
A few recent works start to consider the hardest setting assuming unknown transition as well as bandit feedback. \cite{rosenberg2019online} achieves $O(T^{3/4})$ regret, which is improved by \cite{jin2020learning} to $\tilde O(\sqrt{H^2S^2AT})$, matching the regret of UC-O-REPS in the full information setting. Also, note that the lower
bound of $\Omega(\sqrt{H^2SAT})$ \cite{jin2018q} still applies. In summary, it is found that on tabular MDPs with oblivious reward contamination, an $O(\sqrt{T})$ regret can still be achieved. Recent improvements include best-of-both-worlds algorithms \cite{jin2020simultaneously}, data-dependent bound \cite{lee2020bias} and extension to linear function approximation \cite{neu2020online}.

\paragraph{RL with adversarial transitions and rewards.} Very few prior works study the problem of both adversarial transitions and adversarial rewards, in fact, only one that we are aware of \cite{lykouris2019corruption}. They study a setting where only a constant $C$ number of episodes can be corrupted by the adversary, and most of their technical effort dedicate to designing an algorithm that is agnostic to $C$, i.e. the algorithm doesn't need to know the contamination level ahead of time. As a result, their algorithm takes a multi-layer structure and cannot be easily implemented in practice. Their algorithm achieves a regret of $O(C\sqrt{T})$ for tabular MDPs and $O(C^2\sqrt{T})$ for linear MDPs, which unfortunately becomes vacuous when $C \geq \Omega(\sqrt{T})$ and $C \geq \Omega(T^{1/4})$, respectively. Note that the contamination ratio $C / T$ approaches zero when $T$ increases,  and hence their algorithm cannot handle constant fraction contamination.
Notably, in all of the above works, the adversary can \textit{partially adapt} to the learner's behavior, in the sense that the adversary can pick an adversary MDP $\M_k$ or reward function $r_k$ at the start of episode $k$ based on the history of interactions so far. However, it can no longer adapt its strategy after the episode starts, and therefore, the learner can still use a randomization strategy to trick the adversary.

A separate line of work studies the \textit{online MDP} setting, where the MDP is not adversarial but \textit{slowly} change over time, and the amount of change is bounded under a total-variation metric \cite{cheung2019non, ornik2019learning, ortner2019variational, domingues2020kernel}. Due to the slow-changing nature of the environment, algorithms in these works typically uses a sliding window approach where the algorithm keeps throwing away old data and only learns a policy from recent data, assuming that most of them come from the MDP that the agent is currently experiencing. These methods typically achieve a regret in the form of $O(\Delta^c K^{1-c})$, where $\Delta$ is the total variation bound. It is worth noting that all of these regrets become vacuous when the amount of variation is linear in time, i.e. $\Delta \geq \Omega(T)$. Separately, it is shown that when both the transitions and the rewards are adversarial in every episode, the problem is at least as hard as stochastic parity problem, for which no computationally efficient algorithm exists \cite{yadkori2013online}.

\paragraph{Learning robust controller.} A different type of robustness has also been considered in RL \cite{pinto2017robust, derman2020bayesian} and robust control \cite{zhou1998essentials, petersen2012robust}, where the goal is to learn a control policy that is robust to potential misalignment between the training and deployment environment. Such approaches are often conservative, i.e. the learned polices are sub-optimal even if there is no corruption. In comparison, our approach can learn as effectively as standard RL algorithms without corruption. Interestingly, parallel to our work, a line of concurrent work in the robust control literature \cite{zhang2020policy, zhang2020stability, zhang2021derivative} has also found that policy optimization method enjoys some implicit regularization/robustness property that can automatically converge to robust control policies. An interesting future direction could be to understand the connection between these two kind of robustness.

\paragraph{Robust statistics.} One of the most important discoveries in modern robust statistics is that there exists computationally efficient and robust estimator that can learn near-optimally even under the strongest adaptive adversary. 
For example, in the classic problem of Gaussian mean estimation, the recent works \cite{diakonikolas2016robust, lai2016agnostic} present the first computational and sample-efficient algorithms.
The algorithm in \cite{diakonikolas2016robust} can generate a robust mean estimate $\hat\mu$, such that $\|\hat\mu-\mu\|_2\leq O(\epsilon \sqrt{\log{(1/\epsilon)}})$ under $\epsilon$ corruption. 
Crucially, the error bound does not scale with the dimension $d$ of the problem, suggesting that the estimator remains robust even in high dimensional problems. 
Similar results have since been developed for robust mean estimation under weaker assumptions \cite{diakonikolas2017being}, and for supervised learning and unsupervised learning tasks \cite{charikar2017learning, diakonikolas2019sever}. We refer readers to \cite{diakonikolas2019recent} for a more thorough survey of recent advances in high-dimensional robust statistics.

\section{Problem Definitions}
\label{section:definitions}
A Markov Decision Process (MDP) $\mathcal{M} = (\Scal, \Acal, P, r, \gamma,\mu_0)$
is specified by a state space $\Scal$, an action space $\Acal$, a
transition model $P: \mathcal{S} \times \mathcal{A} \rightarrow \Delta(\mathcal{S})$ (where $\Delta(\mathcal{S})$ denotes a distribution over $\mathcal{S}$),
a (stochastic and possibly unbounded) reward function $r: \Scal\times \Acal \to \Delta(\R)$,
a discounting factor $\gamma \in [0, 1)$, and an initial state
distribution $\mu_0\in\Delta(\Scal)$, i.e. $s_0\sim \mu_0$. In this paper, we assume that $\Acal$ is a small and finite set, and denote $A = \lvert\Acal\rvert$. A policy $\pi: \Scal \to
\Delta(\Acal)$
specifies a decision-making strategy in which the agent chooses
actions based on the current state, i.e., $a \sim\pi(\cdot | s)$.

The value function $V^\pi: \Scal \to \mathbb{R}$ is
defined as the expected discounted sum of future rewards, starting at state $s$
and executing $\pi$, i.e.
$
V^\pi(s) := \EE \left[\sum_{t=0}^\infty \gamma^t  r(s_t, a_t)
| \pi, s_0 = s\right],
$
where the expectation is taken with respect to the randomness of the policy and environment $\mathcal{M}$.
Similarly, the \emph{state-action} value function $Q^\pi: \Scal
\times \Acal \to \mathbb{R}$
is defined as
$
Q^\pi(s,a) := \EE\left[\sum_{t=0}^\infty \gamma^t  r(s_t, a_t) | \pi,
s_0 = s, a_0 = a \right].
$

We define the discounted state-action
distribution $d_{s}^\pi$ of a policy $\pi$:
$
d_{s'}^\pi(s,a) := (1-\gamma) \sum_{t=0}^\infty \gamma^t {\Pr}^\pi(s_t=s,a_t=a|s_0=s'),
$
where $\Pr^\pi(s_t=s,a_t=a|s_0=s')$ is the 
probability that $s_t=s$ and $a_t=a$, after we execute $\pi$ from $t=0$ onwards starting at state
$s'$ in model $\mathcal{M}$. Similarly, we define $d^{\pi}_{s',a'}(s,a)$ as:
$
d^{\pi}_{s',a'}(s,a) := (1-\gamma) \sum_{t=0}^{\infty} \gamma^t {\Pr}^{\pi}(s_t = s, a_t = s | s_0=s', a_0 = a').
$
For any state-action distribution $\nu$, we write $d^{\pi}_{\nu}(s,a):= \sum_{(s',a')\in\mathcal{S}\times\mathcal{A}} \nu(s',a') d^{\pi}_{s',a'}(s,a)$. For ease of presentation, we assume that the agent can reset to $s_0\sim \mu_0$ at any point in the trajectory. We denote $d^{\pi}_{\nu}(s) = \sum_{a}d^{\pi}_{\nu}(s,a)$. 

The goal of the agent is to find a policy $\pi$
that maximizes the expected value from the starting state $s_0$, i.e. the optimization problem is:
$  \max_\pi V^{\pi}(\mu_0) \defeq \EE_{s\sim\mu_0}V^{\pi}(s)$,
where the $\max$ is over some policy class.

For completeness, we specify a $d^{\pi}_{\nu}$-sampler and an unbiased estimator of $Q^{\pi}(s,a)$ in Algorithm~\ref{alg:sampler_est}, which are standard in discounted MDPs~\cite{agarwal2019optimality,agarwal2020pc}. The $d^{\pi}_\nu$ sampler samples $\sa$ i.i.d from $d^{\pi}_{\nu}$, and the $Q^{\pi}$ sampler returns an unbiased estimate of $Q^{\pi}(s,a)$ for a given pair $(s,a)$ by a single roll-out from $\sa$. 
Later, when we define the contamination model and the sample complexity of learning, we treat each call of $d^{\pi}_{\nu}$-sampler (optionally followed by a $Q^{\pi}(s,a)$-estimator) as a \emph{single episode}, as in practice both of these procedures can be achieved in a single roll-out from $\mu_0$.

\begin{assumption}[Linear Q function]\label{ass:linearMDP}
	For the theoretical analysis, we focus on the setting of linear value function approximation. In particular, we assume that there exists a feature map $\phi: \S\times \A\rightarrow \R^d$, such that for any $(s,a)\in \S \times \A$ and any policy $\pi:\S\to \Delta_\A$, we have
	\begin{eqnarray}
		Q^\pi(s,a) = \phi(s,a)^\top w^\pi \text{, for some } \|w^\pi\|\leq W
	\end{eqnarray}
	We also assume that the feature is bounded, i.e. $\max_{s,a}\|\phi(s,a)\|_2\leq 1$, and the reward function has bounded first and second moments, i.e. $\E{}{r(s,a)}\in [0,1]$ and $\mbox{Var}(r(s,a))\leq \sigma^2$ for all $(s,a)$.
\end{assumption}

\begin{remark}
	Assumption \ref{ass:linearMDP} is satisfied, for example, in tabular MDPs and linear MDPs of \cite{jin2020provably} or \cite{yang2019sample}.
	Unlike most theoretical RL literature, we allow the reward to be stochastic and unbounded. Such a setup aligns better with applications with a low signal-to-noise ratio and motivates the requirement for nontrivial robust learning techniques.
\end{remark}
\paragraph{Notation.}
When clear from context, we write $d^\pi(s,a)$ and $d^\pi(s)$ to denote
$d_{\mu_0}^\pi(s,a)$ and $d^{\pi}_{\mu_0}(s)$ respectively.
For iterative algorithms which obtain policies at each episode, we let $V^{i}$,$Q^{i}$ and $A^{i}$ denote the
corresponding quantities associated with episode $i$.
For a vector $v$, we denote $\|v\|_2=\sqrt{\sum_i v_i^2}$, $\|v\|_1=\sum_i |v_i|$,
and $\|v\|_\infty=\max_i |v_i|$. 
We use $\text{Uniform}(\Acal)$ (in short $\text{Unif}_{\Acal}$) to represent a uniform distribution over the set $\Acal$.

\subsection{The Contamination Model}
In this paper, we study the robustness of policy gradient methods under the \textit{$\epsilon$-contamination model}, a widely studied adversarial model in the robust statistics literature, e.g. see \cite{diakonikolas2016robust}. In the classic robust mean estimation problem, given a dataset $D$ and a learning algorithm $f$, the $\epsilon$-contamination model assumes that the adversary has full knowledge of the dataset $D$ and the learning algorithm $f$, and can arbitrarily change $\epsilon$-fraction of the data in the dataset and then send the contaminated data to the learner. The goal of the learner is to identify an $O(\poly(\epsilon))$-optimal estimator of the mean despite the $\epsilon$-contamination. 

Unfortunately, the original $\epsilon$-contamination model is defined for the offline learning setting and does not directly generalize to the online setting, because it doesn't specify the availability of knowledge and the order of actions between the adversary and the learner in the time dimension. 
In this paper, we define the $\epsilon$-contamination model for online learning as follows:
\begin{definition}[$\epsilon$-contamination model for Reinforcement Learning]\label{def: eps_con}
	Given $\epsilon$ and the clean MDP $\M$, an $\epsilon$-contamination adversary operates as follows:
	\begin{enumerate}[leftmargin=*]
		\item The adversary has full knowledge of the MDP $\M$ and the learning algorithm, and observes all the historical interactions.I 
		\item At any time step $t$, the adversary observes the current state-action pair $(s_t,a_t)$, as well as the reward and next state returned by the environment, $(r_t,s_{t+1})$. He then can decide whether to replace $(r_t,s_{t+1})$ with an arbitrary reward and next state $(r^\dagger_t,s^\dagger_{t+1})\in\R \times \S$.
		\item The only constraint on the adversary is that if the learning process terminates after $K$ episodes, he can contaminate in at most $\epsilon K$ episodes.
	\end{enumerate}
\end{definition}
Compared to the standard adversarial models studied in online learning \cite{shalev2011online}, adversarial bandits \cite{bubeck2012regret, lykouris2018stochastic, gupta2019better} and adversarial RL \cite{lykouris2019corruption, jin2020learning}, the $\epsilon$-contamination model in Definition \ref{def: eps_con} is stronger in several ways:
(1) The adversary can adaptively attack after observing the action of the learner as well as the feedback from the clean environments;
(2) the adversary can perturb the data arbitrarily (any real-valued reward and any next state from the state space) rather than sampling it from a pre-specified bounded adversarial reward function or adversarial MDP.

Given the contamination model, our first result is a lower-bound, showing that under the $\epsilon$-contamination model, one can only hope to find an $O(\epsilon)$-optimal policy. Exact optimal policy identification is not possible even with infinite data.
\begin{theorem}[lower bound]\label{thm:lb}
	For any algorithm, there exists an MDP such that the algorithm fails to find an $\left(\frac{\epsilon}{2(1-\gamma)}\right)$-optimal policy under the $\epsilon$-contamination model with a probability of at least $1/4$.
\end{theorem}
The high-level idea is that we can construct two MDPs, $M$ and $M'$, with the following properties: 1. No policy can be $O(\epsilon/(1-\gamma))$ optimal on both MDP simultaneously. 2. An $\epsilon$-contamination adversary can with large probability mimic one MDP via contamination in the other, regardless of the learner's behavior. Therefore, under contamination, the learner will not be able to distinguish $M$ and $M'$ and must suffer $\Omega(\epsilon/(1-\gamma))$ gap on at least one of them.

\subsection{Background on NPG} 
Given a differentiable parameterized policy $\pi_{\theta}: \mathcal{S}\to\Delta(\mathcal{A})$, NPG
can be written in the following actor-critc style update form. With the dataset $\{s_i,a_i, \widehat{Q}^{\pi_{\theta}}(s_i,a_i)\}_{i=1}^N$ where $s_i,a_i\sim d^{\pi_\theta}_{\nu}$, and $\widehat{Q}^{\pi_\theta}(s_i,a_i)$ is unbiased estimate of $Q^{\pi_\theta}(s,a)$ (e.g., via $Q^{\pi}$-estimator), we have
\begin{align}\label{eq:NPG_update}
	&\widehat{w} \in \argmin_{w: \|w\|_2 \leq W} \sum_{i=1}^N \left( w^{\top}\nabla \log\pi_\theta(a_i|s_i) - \widehat{Q}^{\pi_{\theta}}(s_i,a_i) \right)^2\nonumber\\
	&\theta' = \theta + \eta \widehat{w}.
\end{align} 
In theoretical part of this work, we focus on softmax linear policy, i.e., $\pi_{\theta}(a|s) \propto \exp(\theta^{\top}\phi(s,a))$. In this case, note that $\nabla \log\pi_\theta(a|s) = \phi(s,a)$, and
it is not hard to verify that the policy update procedure is equivalent to:
\begin{align*}
	\pi_{\theta'}(a|s) \propto \pi_{\theta}(a|s) \exp\left( \eta \widehat{w}^{\top} \phi(s,a) \right), \quad \forall s,a,
\end{align*} which is equivalent to running Mirror Descent on each state with a reward vector $\widehat{w}^{\top}\phi(s,\cdot)\in\mathbb{R}^{|\A|}$. We refer readers to \cite{agarwal2019optimality} for more detailed explanation of NPG and the equivalence between the form in Eq.~\eqref{eq:NPG_update} and the classic form that uses Fisher information matrix. Similar to \cite{agarwal2019optimality}, we make the following assumption of having access to an exploratory reset distribution, under which it has been shown that NPG can converge to the optimal policy without contamination.

\begin{assumption}[Relative condition number]\label{assum:conditioning} With respect to any state-action distribution $\upsilon$,  define:
	\[
	\Sigma_\upsilon = \EE_{s,a \sim \upsilon}\left[ \phi_{s,a}\phi_{s,a}^\top\right],
	\]
	and define
	\[
	\sup_{w \in \R^d} \ \frac{w^\top \Sigma_{d^\star} w}
	{w^\top \Sigma_\nu w}
	=\kappa\text{, where }d^*(s,a) = d^{\pi^*}_{\mu_0}(s)\circ \text{Unif}_{\A}(a)
	\]
	We assume $\kappa$ is finite and small w.r.t. a reset distribution $\nu$ available to the learner at training time.
\end{assumption}

\section{The Natural Robustness of NPG Against Bounded corruption}\label{sec:npg}
Our first result shows that, surprisingly, NPG can already be robust against $\epsilon$-contamination, if the adversary can only generate small and bounded rewards. In particular, we assume that the adversarial rewards is bounded in $[0,1]$ (the feature $\phi(s,a)$ is already bounded).

\begin{theorem}[Natural robustness of NPG]\label{thm:npg_natural}
	Under assumptions \ref{ass:linearMDP} and \ref{assum:conditioning}, given a desired optimality gap $\alpha$, there exists a set of hyperparameters agnostic to the contamination level $\epsilon$, such that 
	Algorithm \ref{alg:q_npg_sample} guarantees with a $poly(1/\alpha, 1/(1-\gamma), |\A|,W,\sigma,\kappa)$ sample complexity that under $\epsilon$-contamination with adversarial rewards bounded in $[0,1]$, we have
	\begin{eqnarray}
		\EE\left[V^{*}(\mu_0) - V^{\hat\pi}(\mu_0)\right]
		\leq
		\tilde O\left(\max\left[\alpha, W\sqrt{\frac{|\Acal|\kappa\epsilon}{(1-\gamma)^3}} \mbox{ }\right]\right)\nonumber
	\end{eqnarray}
	where $\hat\pi$ is the uniform mixture of $\pi^{(1)}$ through $\pi^{(T)}$.
\end{theorem}
A few remarks are in order.
\begin{remark}[Agnostic to the contamination level $\epsilon$]
	It is worth emphasizing that to achieve the above bound, the hyperparameters of NPG are agnostic to the value of $\epsilon$, and so the algorithm can be applied in the more realistic setting where the agent does not have knowledge of the contamination level $\epsilon$, similar to what's achieved in \cite{lykouris2019corruption} with a complicated nested structure. The same property is also achieved by the FPG algorithm in the next section.
\end{remark}
\begin{remark}[Dimension-independent robustness guarantee]
	Theorem \ref{thm:npg_natural} guarantees that NPG can find an $O(\epsilon^{1/2})$-optimal policy after polynomial number of episodes, provided that $|\A|$ and $\kappa$ are small. Conceptually, the relative condition number $\kappa$ indicates how well-aligned the initial state distribution is to the occupancy distribution of the optimal policy. A good initial distribution can have a $\kappa$ as small as $1$, and so $\kappa$ is independent of $d$. Interested readers can refer to \cite{agarwal2019optimality} (Remark 6.3) for additional discussion on the relative condition number. Here, importantly, the optimality gap does not directly scale with $d$, and so the guarantee will not blow up on high-dimensional problems. This is an important attribute of robust learning algorithms heavily emphasized in the traditional robust statistics literature. 
\end{remark} 

The proof of Theorem \ref{thm:npg_natural} relies on the following NPG regret lemma, first developed by \cite{even2009online} for the MDP-Expert algorithm and later extend to NPG by \cite{agarwal2019optimality,agarwal2020pc}:
\begin{lemma}[NPG Regret Lemma]\label{thm:npg_regret}
	Suppose
	Assumption~\ref{ass:linearMDP} and \ref{assum:conditioning} hold and Algorithm \ref{alg:q_npg_sample}
	starts with $\theta^{(0)}=0$, $\eta
	=\sqrt{2\log |\Acal| /(W^2T)}$. Suppose in addition that the (random) sequence of iterates satisfies
	the assumption that
	\begin{equation}
		\EE\left[\EE_{s,a\sim d^{(t)}}\left[\left(Q^{\pi^{(t)}}(s,a)-\phi(s,a)^\top w^{(t)} \right)^2\right]\right]
		\leq \epsilon_{stat}^{(t)}.\nonumber
	\end{equation}
	Then, we have that
	\begin{align}
		\EE&\left[\sum_{t=1}^T \{V^{*}(\mu_0) - V^{(t)}(\mu_0) \}\right]\\
		&\qquad\leq
		\frac{W}{1-\gamma}\sqrt{2 \log |\Acal|T}
		+\sum_{t=1}^T \sqrt{ \frac{4|\Acal| \kappa \epsilon_{stat}^{(t)}}{(1-\gamma)^3}}.\nonumber
	\end{align}
\end{lemma}
Intuitively, Lemma \ref{thm:npg_regret} decompose the regret of NPG into two terms. The first term corresponds to the regret of standard mirror descent procedure, which scales with $\sqrt{T}$. The second term corresponds to the estimation error on the Q value, which acts as the reward signal for mirror descent. When not under attack,  estimation error $\epsilon_{stat}^{(t)}$ goes to zero as the number of samples $M$ gets larger, which in turn implies the global convergence of NPG. However, when under bounded attack, the generalization error $\epsilon_{stat}^{(t)}$ will not go to zero even with infinite data. Nevertheless, we can show that it is bounded by $O(\epsilon^{(t)})$ when the sample size $M$ is large enough, where $\epsilon^{(t)}$ denotes the fraction of episodes being corrupted in iteration $t$. Note that by definition, we have $\sum_{t} \varepsilon^{(t)} \leq \varepsilon T$.
\begin{lemma}[Robustness of linear regression under bounded contamination]\label{thm:robust_ols}
	Suppose the adversarial rewards are bounded in $[0,1]$, and in a particular iteration $t$, the adversary contaminates $\epsilon^{(t)}$ fraction of the episodes, then given M episodes, it is guaranteed that with probability at least $1-\delta$,
	\begin{align}
		\EE_{s,a\sim d^{(t)}}&\left[\left(Q^{\pi^{(t)}}(s,a)-\phi(s,a)^\top w^{(t)} \right)^2\right] \\
		&\leq 4 \left(W^2+WH\right)\left(\epsilon^{(t)} + \sqrt{\frac{8}{M}\log\frac{4d}{\delta}}\right).\nonumber
	\end{align}
	where $H = (\log \delta-\log M)/\log\gamma$ is the effective horizon.
\end{lemma}
This along with the NPG regret lemma guarantees that the expected regret of NPG is bounded by $O(\sqrt{T}+M^{-1/4}+\sqrt{\epsilon}T)$ which in turn guarantees to identify an $O(\sqrt{\epsilon})$-optimal policy.

In the special case of tabular MDPs, $\phi(s,a)$ will all be one-hot vectors and $W$ will in general by on the order of $O(\sqrt{SA})$, which means that the bound given by Theorem \ref{thm:npg_natural} still scales with the size of the state space. In the following corollary, we show that this dependency can be removed through a tighter analysis. 
\begin{corollary}[Dimension-free Robustness of NPG in tabular MDPs]\label{thm:npg_tabular}
	Given a tabular MDP and assumption \ref{assum:conditioning}, given a desired optimality gap $\alpha$, there exists a set of hyperparameters agnostic to the contamination level $\epsilon$, such that 
	Algorithm \ref{alg:q_npg_sample} guarantees with a $poly(1/\alpha, 1/(1-\gamma), |\A|,W,\sigma,\kappa)$ sample complexity that under $\epsilon$-contamination with adversarial rewards bounded in $[0,1]$, we have
	\begin{eqnarray}
		\EE\left[V^{*}(\mu_0) - V^{\hat\pi}(\mu_0)\right]
		\leq
		\tilde O\left(\max\left[\alpha, \sqrt{\frac{|\Acal|\kappa\epsilon}{(1-\gamma)^5}} \mbox{ }\right]\right)\nonumber
	\end{eqnarray}
	where $\hat\pi$ is the uniform mixture of $\pi^{(1)}$ through $\pi^{(T)}$.
\end{corollary}
In the more general case of linear MDP, $W$ will not necessarily scale with $d$ in an obvious way and thus we leave Theorem \ref{thm:npg_natural} untouched.

\section{FPG: Robust NPG Against Unbounded Corruption}\label{sec:fpg}
\begin{algorithm}[!t]
	\caption{$d_{\nu}^\pi$ sampler and $Q^{\pi}$ estimator}
	\label{alg:sampler_est}
	\begin{algorithmic}[1]
		\setcounter{algorithm}{-1}
		\Function{$d_{\nu}^\pi$-sampler}{}
		\State  \hspace*{-0.1cm}\textbf{Input}:  A reset distribution $\nu\in\Delta(\Scal\times\Acal)$.
		\State Sample $s_0,a_0 \sim \nu$.
		\State Execute $\pi$ from $s_0, a_0$; at any step $t$ with $(s_t,a_t)$, return $(s_t,a_t)$ with probability $1-\gamma$.
		\caption{$d^{\pi}$ sampler and $Q^{\pi}$ estimator}
		\EndFunction
		\Function{$Q^\pi$-estimator}{}
		\State  \hspace*{-0.1cm}\textbf{Input}:  current state-action $\sa$, a policy $\pi$.
		\State Execute $\pi$ from $(s_0,a_0) = (s, a)$; at step $t$ with $(s_t,a_t)$, terminate with probability $1-\gamma$.
		\State  \hspace*{-0.1cm}\textbf{Return}: $\widehat{Q}^{\pi}\sa = \sum_{i=0}^t r(s_i,a_i)$.
		\caption{$d_{\nu}^\pi$ sampler and $Q^{\pi}$ estimator}
		\EndFunction
		\setcounter{algorithm}{1}
	\end{algorithmic}
{\color{red}[In an adversarial episode, the adversary can hijack the $d_{\nu}^\pi$ sampler to return any $(s,a)$ pair and the $Q^\pi$-estimator to return any $\widehat{Q}^{\pi}\sa\in\R$.]}
\end{algorithm}
\begin{algorithm}[!t]
	\begin{algorithmic}[1]	
		\Require Learning rate $\eta$; number of episodes per iteration $M$
		\State Initialize $\theta^{(0)} = 0$.
		\For{$t=0,1,\ldots,T-1$}
		\State Call Algorithm~\ref{alg:sampler_est} $M$ times with $\pi^{(t)}$ to obtain a dataset that consist of $s_i,a_i\sim d^{(t)}_\nu$ and $\widehat{Q}^{(t)}(s_i,a_i)$, $i\in [M]$.
		\State Solve the linear regression problem 
		\[w^{(t)} = \argmin_{ \|w\|_2\leq W} \sum_{i=1}^M \left(\widehat{Q}^{(t)}(s_i,a_i)-w^\top \nabla_\theta \phi(s_i,a_i)\right)^2\]  
		\State Update $\theta^{(t+1)} = \theta^{(t)} + \eta w^{(t)}$.
		\EndFor
	\end{algorithmic}
	\caption{Natural Policy Gradient (NPG)}
	\label{alg:q_npg_sample}
\end{algorithm}
\begin{algorithm}[!t]
	\caption{Robust Linear Regression via \texttt{SEVER}}
	\label{alg:sever}
	\begin{algorithmic}
		\State {\bf{Input:}} Dataset $\{(x_i,y_i)\}_{i=1:M}$, a standard linear regression solver $\sL$, and parameter $\sigma' \in \R_+$.
		\State Initialize $S \gets \{1,\ldots,M\}$, $f_i(w) = \|y_i-w^\top x_i\|^2$.
		\Repeat
		\State ${w} \gets \sL(\{(x_i,y_i)\}_{i\in S})$. $\triangleright$ Run learner on $S$.
		\State Let $\widehat{\nabla} = \frac{1}{|S|} \sum_{i\in S} \nabla f_i(w)$.
		\State Let $G = [\nabla f_i({w}) - \widehat{\nabla}]_{i \in S}$ be the $|S| \times d$ matrix of centered gradients.
		\State Let $v$ be the top right singular vector of $G$.
		\State Compute the vector $\tau$ of \emph{outlier scores} defined via
		$\tau_i = \left((\nabla f_i({w}) - \widehat{\nabla}) \cdot v\right)^2$.
		\State $S' \gets S$
		\If{$\frac{1}{|S|}\sum_{i\in S} \tau_i \leq c_0 \cdot \sigma'^2$, for some constant $c_0>1$}
		\State $S = S'$ $\triangleright$ We only filter out points if the variance is larger than an appropriately chosen threshold.
		\Else
		\State Draw $T$ from Uniform$[0,\max_i \tau_i]$.
		\State  $S=\{i \in S: \tau_i < T \}$.
		\EndIf
		\Until{$S = S'$.}
		\State Return $w$.
	\end{algorithmic}
\end{algorithm}

\begin{figure*}[th!]
	\begin{subfigure}{0.33\textwidth}
		\centering
		\includegraphics[width=.9\columnwidth]{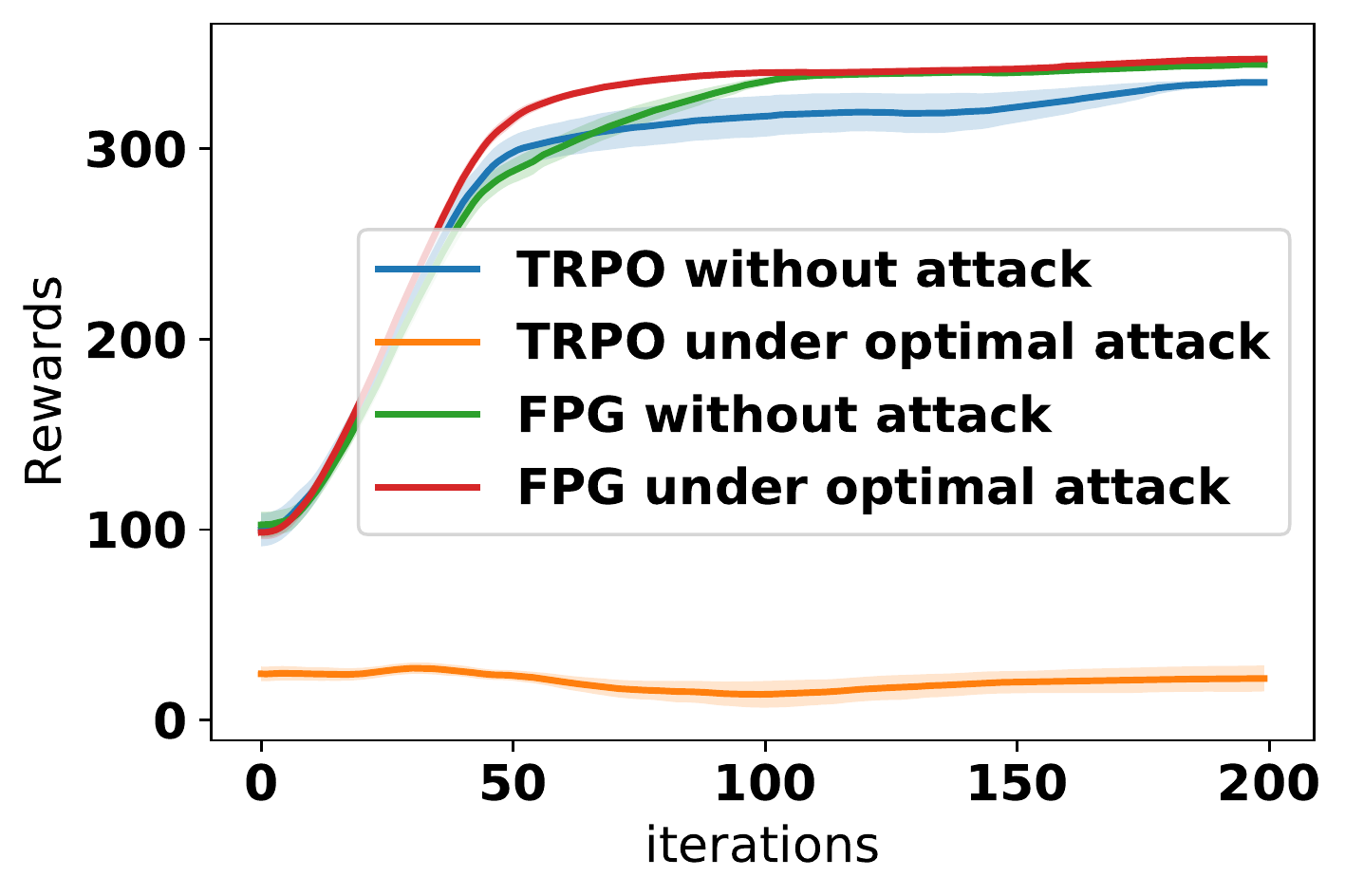}
		\caption{Swimmer-v3}
		\label{fig:Swimmer}
	\end{subfigure}
	\begin{subfigure}{0.33\textwidth}
		\centering
		\includegraphics[width=.9\columnwidth]{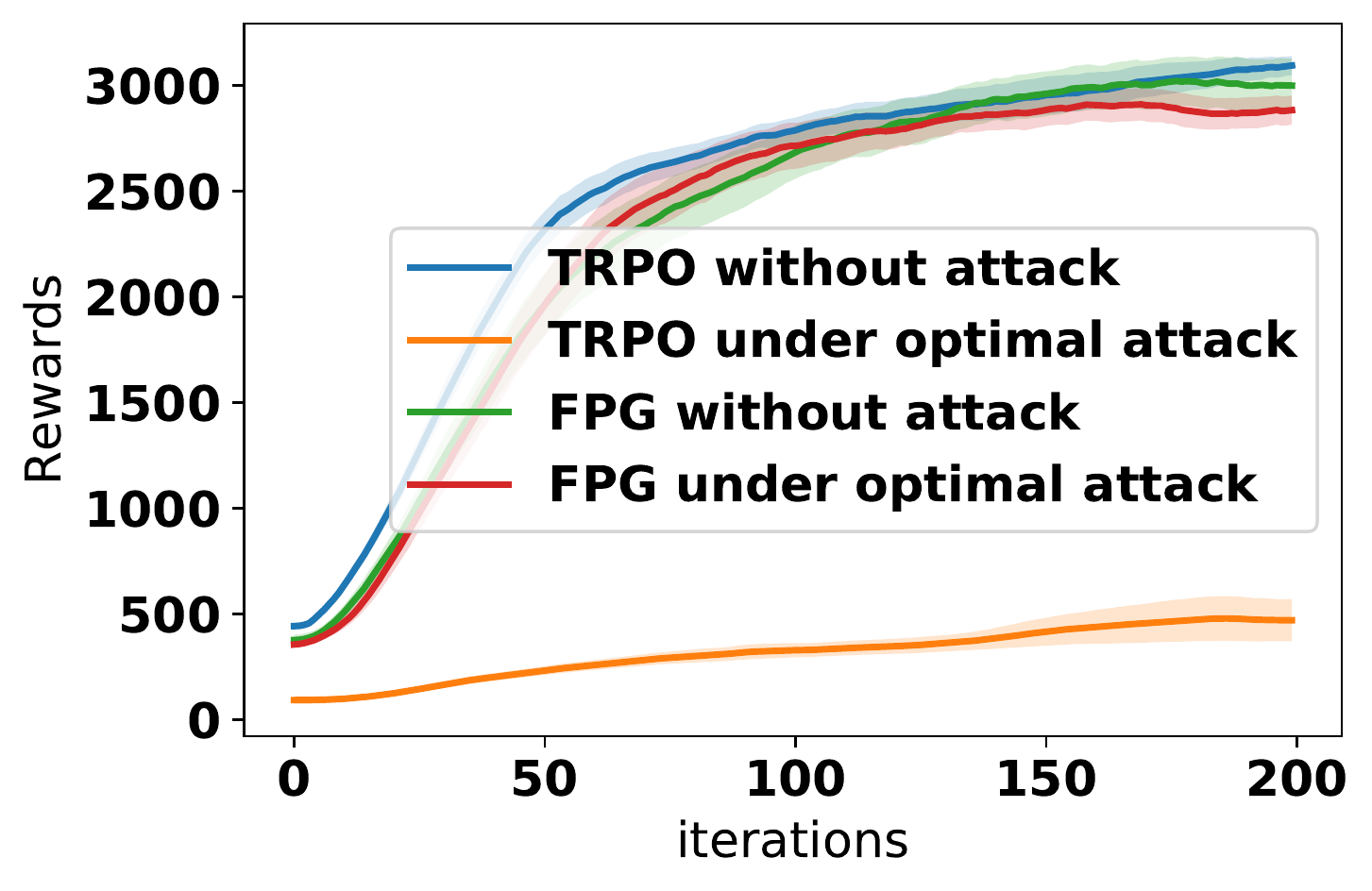}
		\caption{Hopper-v3}
		\label{fig:Hopper}
	\end{subfigure}
	\begin{subfigure}{0.33\textwidth}
		\centering
		\includegraphics[width=.9\columnwidth]{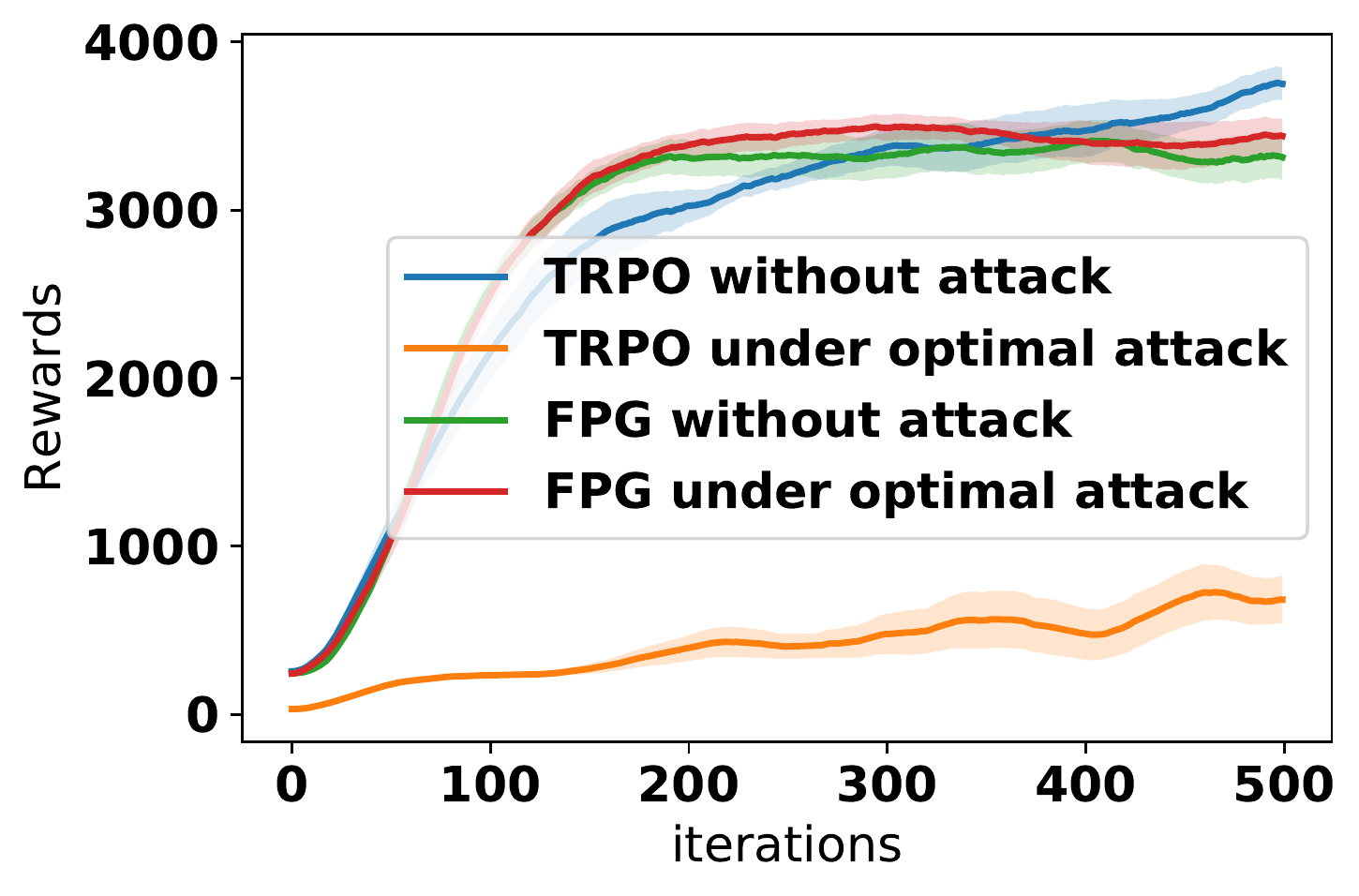}
		\caption{Walker2d-v3}
		\label{fig:Walker2d}
	\end{subfigure}
	\begin{subfigure}{0.33\textwidth}
		\centering
		\includegraphics[width=.9\columnwidth]{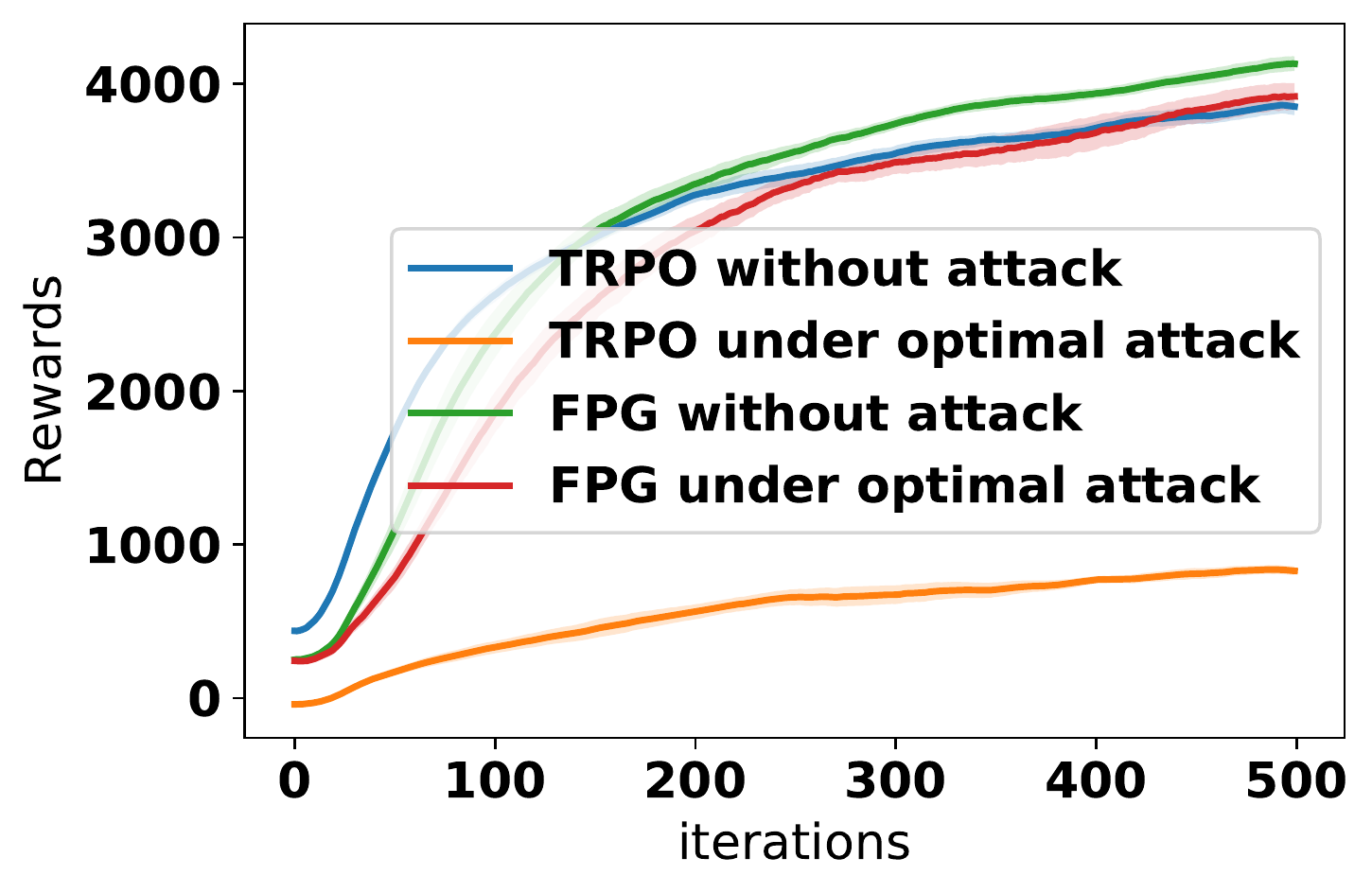}
		\caption{HalfCheetah-v3}
		\label{fig:HalfCheetah}
	\end{subfigure}
	\begin{subfigure}{0.33\textwidth}
		\centering
		\includegraphics[width=.9\columnwidth]{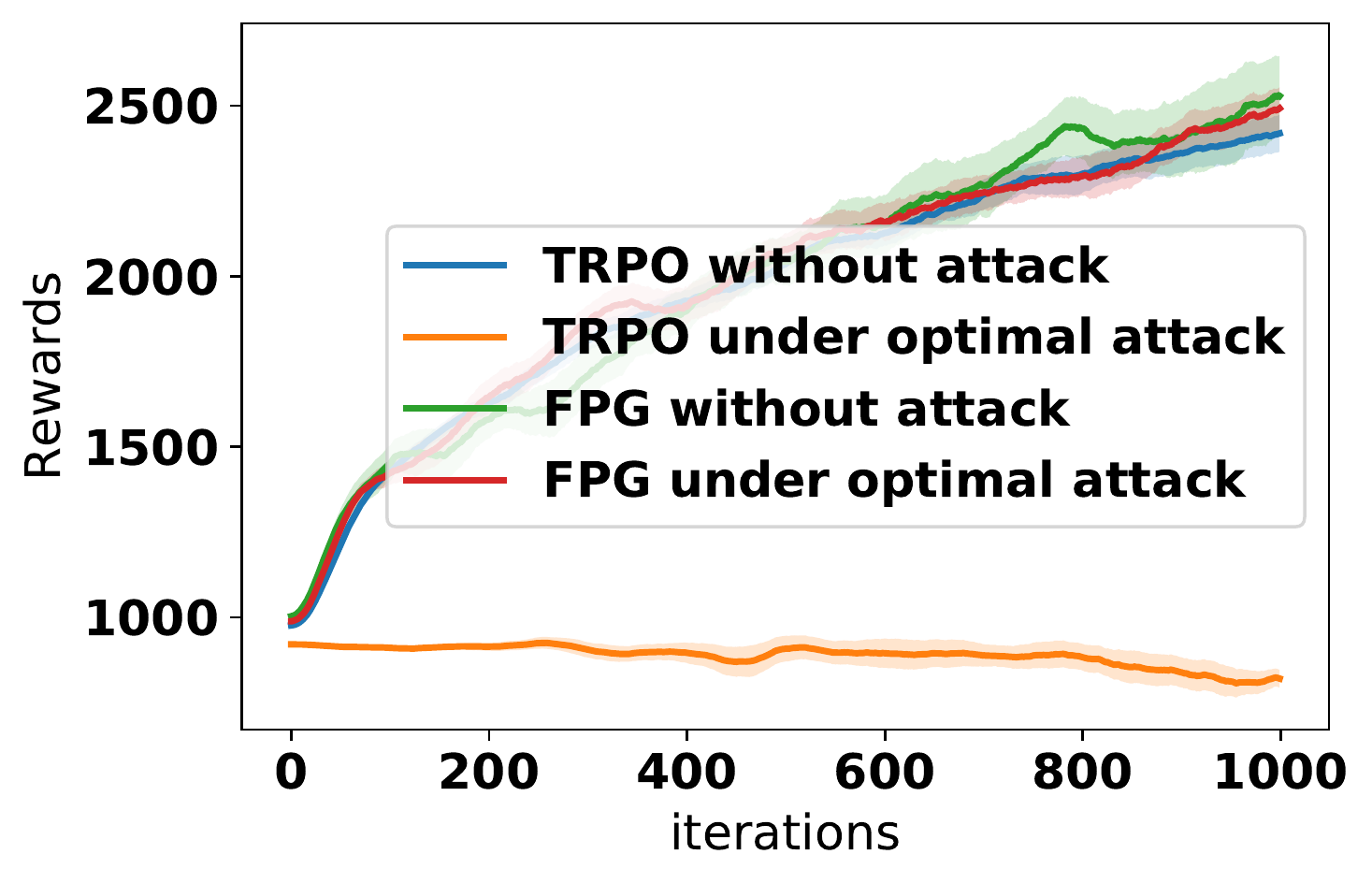}
		\caption{Ant-v3}
		\label{fig:Ant}
	\end{subfigure}
	\begin{subfigure}{0.33\textwidth}
		\centering
		\includegraphics[width=.9\columnwidth]{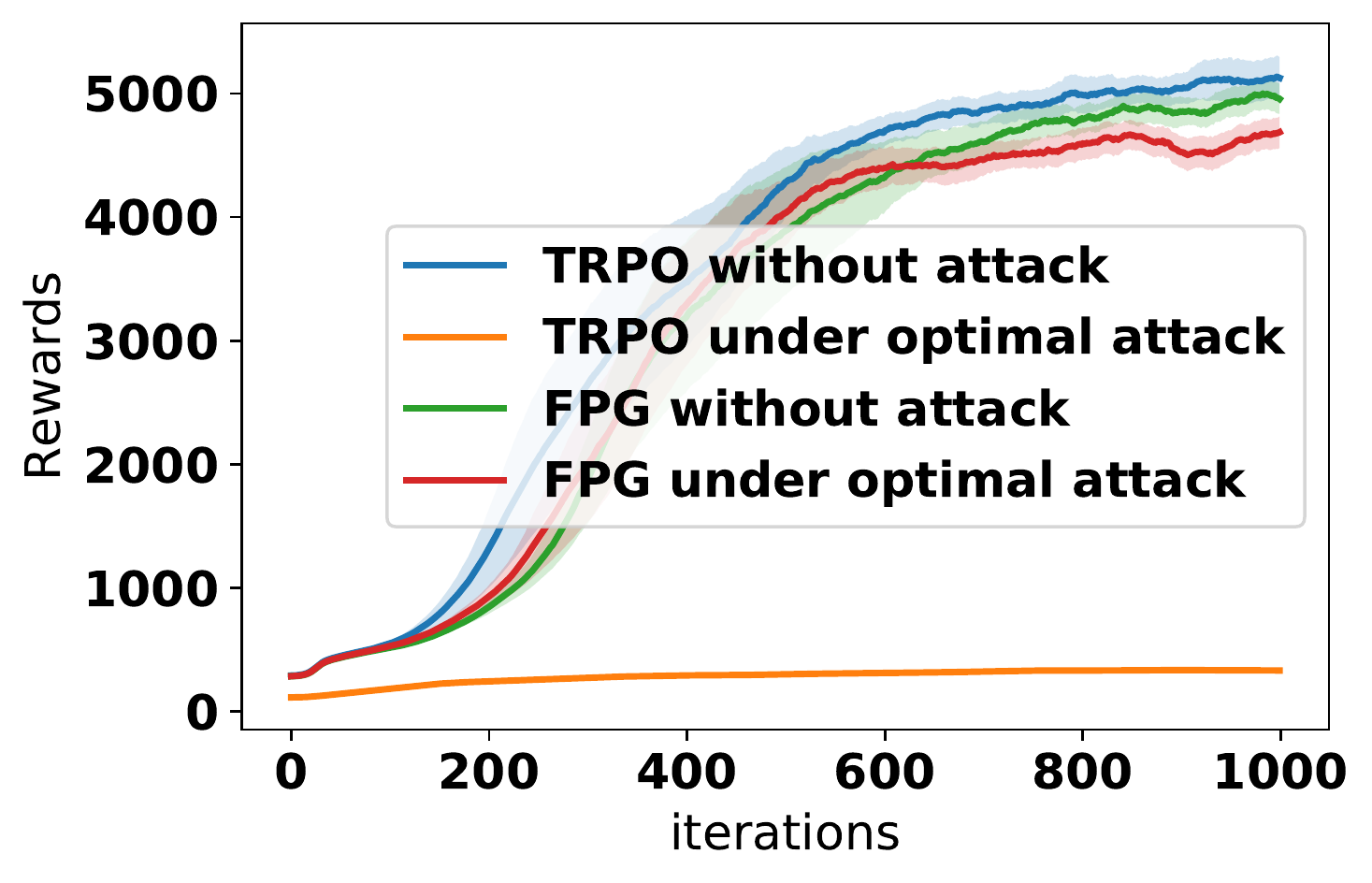}
		\caption{Humanoid-v3}
		\label{fig:Humanoid}
	\end{subfigure}
	\caption{Experiment Results on the 6 MuJoTo benchmarks.}
	\label{fig:exp_result}
\end{figure*}

Our second result is the Filtered Policy Gradient (FPG) algorithm, a robust variant of the NPG algorithm \cite{kakade2001natural, agarwal2019optimality} that can be robust against arbitrary and \emph{potentially unbounded} data corruption. Specifically, FPG replace the standard linear regression solver in NPG with a statistically robust alternative. In this work, we use the \texttt{SEVER} algorithm \cite{diakonikolas2019sever}. In practice, one can substitute it with any computationally efficient robust linear regression solver. We show that FPG can find an $O(\epsilon^{1/4})$-optimal policy under $\epsilon$-contamination with a polynomial number of samples.
\begin{theorem}\label{thm:fpg}
	Under assumptions \ref{ass:linearMDP} and \ref{assum:conditioning}, given a desired optimality gap $\alpha$, there exists a set of hyperparameters agnostic to the contamination level $\epsilon$, such that 
	Algorithm \ref{alg:q_npg_sample}, using Algorithm \ref{alg:sever} as the linear regression solver, guarantees with a  $poly(1/\alpha, 1/(1-\gamma), |\A|,W,\sigma,\kappa)$ sample complexity that under $\epsilon$-contamination, we have
	\begin{align}
		\EE&\left[V^{*}(\mu_0) - V^{\hat\pi}(\mu_0)\right]\\
		&\qquad\leq
		\tilde O\left(\max\left[\alpha, \sqrt{ \frac{|\Acal| \kappa \left(W^2+\sigma W\right)}{(1-\gamma)^4}}\epsilon^{1/4}\right]\right).\nonumber
	\end{align}
	where $\hat\pi$ is the uniform mixture of $\pi^{(1)}$ through $\pi^{(T)}$.
\end{theorem}
The proof of Theorem \ref{thm:fpg} relies on a similar result to Lemma \ref{thm:robust_ols}, which shows that if we use Algorithm \ref{alg:sever} as the linear regression subroutine, then $\epsilon_{stat}^{(t)}$ can be bounded by $O(\sqrt{\epsilon^{(t)}})$ when the sample size $M$ is large enough, even under unbounded $\epsilon$-contamination.
\begin{lemma}[Robustness of \texttt{SEVER} under unbounded contamination]\label{thm:sever_result}
	Suppose the adversarial rewards are unbounded, and in a particular iteration $t$, the adversarial contaminate $\epsilon^{(t)}$ fraction of the episodes, then given M episodes, it is guaranteed that if $\epsilon^{(t)}\leq c$, for some absolute constant c, and any constant $\tau\in[0,1]$, we have 
	\begin{align}
		&\EE \left[\EE_{s,a\sim d^{(t)}}\left[\left(Q^{\pi^{(t)}}(s,a)-\phi(s,a)^\top w^{(t)} \right)^2\right]\right]
		\\
		&\leq
		O\left(\left(W^2+\frac{\sigma W}{1-\gamma}\right)\left(\sqrt{\epsilon^{(t)}}+
		f(d,\tau)M^{-\frac{1}{2}} + \tau\right)\right).\nonumber
	\end{align}
	where $f(d,\tau) = \sqrt{d\log d}+\sqrt{\log(1/\tau)}$.
\end{lemma}
In Lemma \ref{thm:sever_result}, $c$ is the break point of \texttt{SEVER} and is an absolute constant that does not depend on the data, and $(1-\tau)$ is the probability that the clean data satisfies a certain stability condition which suffices for robust learning. 

\section{Robust NPG with Exploration via Policy Cover}
The Policy Cover-Policy Gradient (\texttt{PC-PG}) algorithm, defined in Algorithm \ref{alg:pcpg}, is an exploratory policy gradient methods recently developed by \cite{agarwal2020pc}. Intuitively, PC-PG is a spiritually inheritor of the \texttt{RMax} algorithm \cite{brafman2002r}, and encourages exploration by adding reward bonuses in directions of the feature space that past polices (stored in the policy cover) haven't visited sufficiently.
Similar to the NPG algorithm, we show that \texttt{PC-PG} enjoys a (weaker) natural robustness against bounded data corruption. This gives us the following robustness guarantee:
\begin{algorithm}[!t]
	\begin{algorithmic}[1]
		\State \hspace*{-0.1cm}\textbf{Input}: iterations $N$, threshold $\beta$, regularizer $\lambda$
		\State \hspace*{-0.1cm}Initialize $\pi^0(a|s)$ to be uniform.
		\For{episode $n = 0, \dots N-1$}
		
		\State Define the policy cover's state-action distribution $\rho^n_{\text{cov}}$ as
		\begin{equation*}
			\rho_{\text{cov}}^n(s,a) = \sum_{i=0}^{n} d^i(s,a)/(n+1)
			\label{eq:cover_def}
		\end{equation*}
		\State  Sample $\{s_i,a_i\}_{i=1}^K \sim \rho_{\text{cov}}^n(s,a)$ and estimate the covariance of $\pi^n$ as
		\begin{equation*}
			\widehat{\Sigma}^n = (n+1)\left(\sum_{i=1}^K \phi(s_i,a_i)\phi(s_i,a_i)^{\top}/K\right)+\lambda I
		\end{equation*}\label{line:feature_cov}
		\State Set the exploration bonus $b^n$ to reward infrequently visited state-action under $\rho^n_{\text{cov}}$ \label{line:bonus}
		
		\begin{equation*}
			b^n(s,a) = \frac{\one\{\sa~:~ \phi\sa^\top (\widehat{\Sigma}_{\text{cov}}^n)^{-1}\phi\sa \geq \beta\}}{1-\gamma}.
		\end{equation*}
		\State Update $\pi^{n+1} = \textrm{Robust-NPG-Update}(\rho^n_{\text{cov}}, b^n)$ [Alg. \ref{alg:npg2} in the appendix, similar to Alg. \ref{alg:q_npg_sample}].
		\EndFor
		\State \Return $\hat\pi \defeq \text{Uniform}\{\pi^0,...,\pi^{N-1}\}$.
	\end{algorithmic}
	\caption{Robust Policy Cover-Policy Gradient (PC-PG)}
	\label{alg:pcpg}
\end{algorithm}

\begin{theorem}[Best hyperparameters, assuming known $\epsilon$]\label{thm:pcpg}
	There exists a set of hyperparameters, such that Algorithm \ref{alg:pcpg} guarantees with probability at least $1-\delta$
	\begin{eqnarray}
	\EE\left[V^{*}(\mu_0) - V^{\hat\pi}(\mu_0)\right]\leq
		\tilde O\left(
		d^{2}\epsilon^{1/7}
		\right)
	\end{eqnarray}
	with $\poly\left(d,W,\sigma,\kappa, |\A|,1/(1-\gamma),1/\alpha\right)$ number of episodes.
\end{theorem}

\begin{remark}[The scaling with dimension $d$]Compared to the guarantee of vanilla NPG, \texttt{PC-PG} alleviate the requirement of a good initial distribution with small relative conditional number. However, this process introduce a dependency on $d$. In particular, the gap in Theorem \ref{thm:pcpg} is on the order of 
	$\tilde O\left(d^{2}\epsilon^{1/7}\right)$
	.
	This implies that for any fixed $\epsilon$, the bound becomes vacuous for high dimensional problems where $d\geq \Omega(\epsilon^{-2/3})$. Intuitively, the dependency on $d$ is introduced because \texttt{PC-PG} is trying to find a initial state-action distribution with good coverage, i.e. a distribution whose covariance matrix has a lower-bounded smallest eigenvalue. Under the assumption that $\|\phi(s,a)\|_2\leq 1$, such a distribution will have a covariance matrix whose eigenvalues are all on the order of $O(1/d)$. and so the value of $\kappa$ will be on the order of $O(d)$, which by Theorem \ref{thm:fpg} will similarly introduce a $d$ dependency. We expect that for a robust RL algorithm to avoid the $d$ dependency, it must gradually find a state-action distribution approaching $d^*$. How to design such an algorithm is left as an open problem. 
\end{remark}

\section{Experiments}
In the theoretical analysis, we rely on the assumption of linear Q function, finite action space and exploratory initial state distribution to prove the robustness guarantees for NPG and FPG. In this section, we present a practical implementation of \texttt{FPG}, based on the \textit{Trusted Region Policy Optimization} (\texttt{TRPO}) algorithm \cite{schulman2015trust}, in which the conjugate gradient step (equivalent to the linear regression step in Alg. \ref{alg:q_npg_sample}) is robustified with \texttt{SEVER}. The pseudo-code and implementation details are discussed in appendix \ref{sec:code}\footnote{A Pytorch Implementation of \texttt{FPG-TRPO} can be found at \url{https://github.com/zhangxz1123/FilteredPolicyGradient}}. In this section, we demonstrate its empirical performance on the MuJoCo benchmarks \cite{todorov2012mujoco}, a set of high-dimensional continuous control domains where both assumptions no longer holds, and show that \texttt{FPG} can still consistently performs near-optimally with and without attack.
\begin{figure*}[th!]
	\centering
	\begin{subfigure}{0.19\textwidth}
		\centering
		\includegraphics[width=0.8\columnwidth]{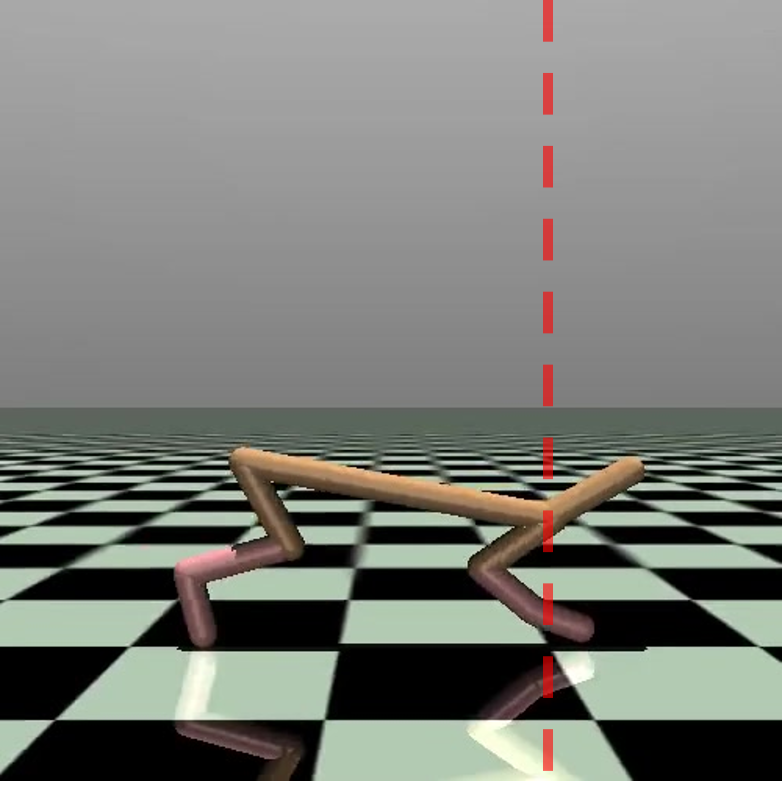}
	\end{subfigure}
	\begin{subfigure}{0.19\textwidth}
		\centering
		\includegraphics[width=0.8\columnwidth]{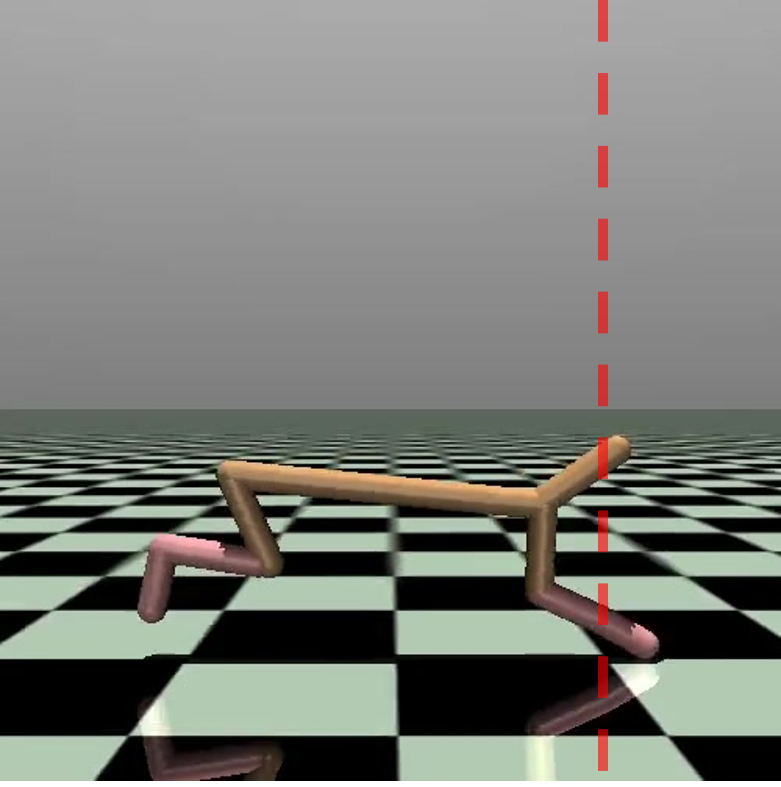}
	\end{subfigure}
	\begin{subfigure}{0.19\textwidth}
		\centering
		\includegraphics[width=0.8\columnwidth]{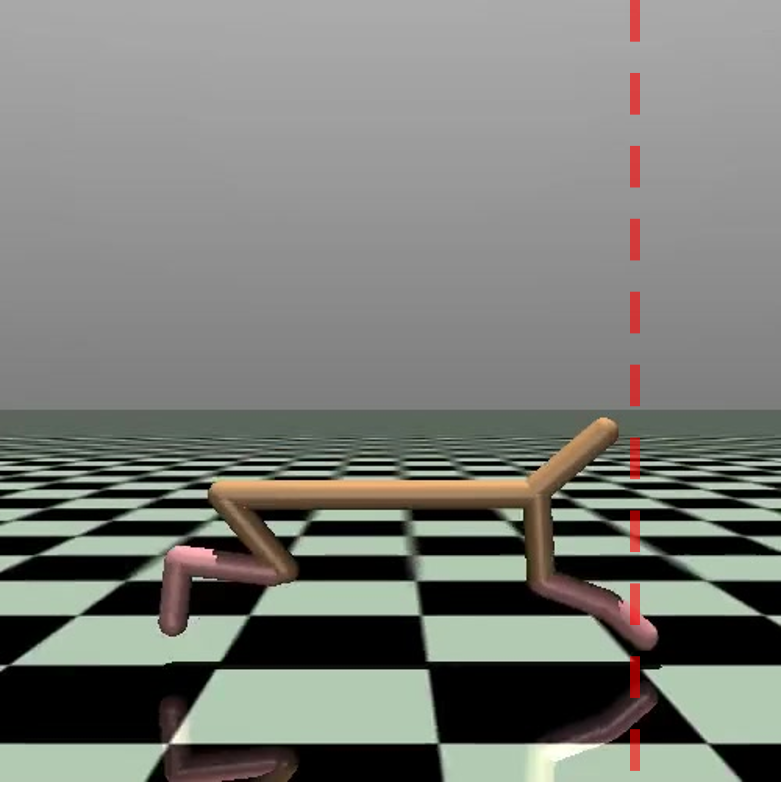}
	\end{subfigure}
	\begin{subfigure}{0.19\textwidth}
		\centering
		\includegraphics[width=0.8\columnwidth]{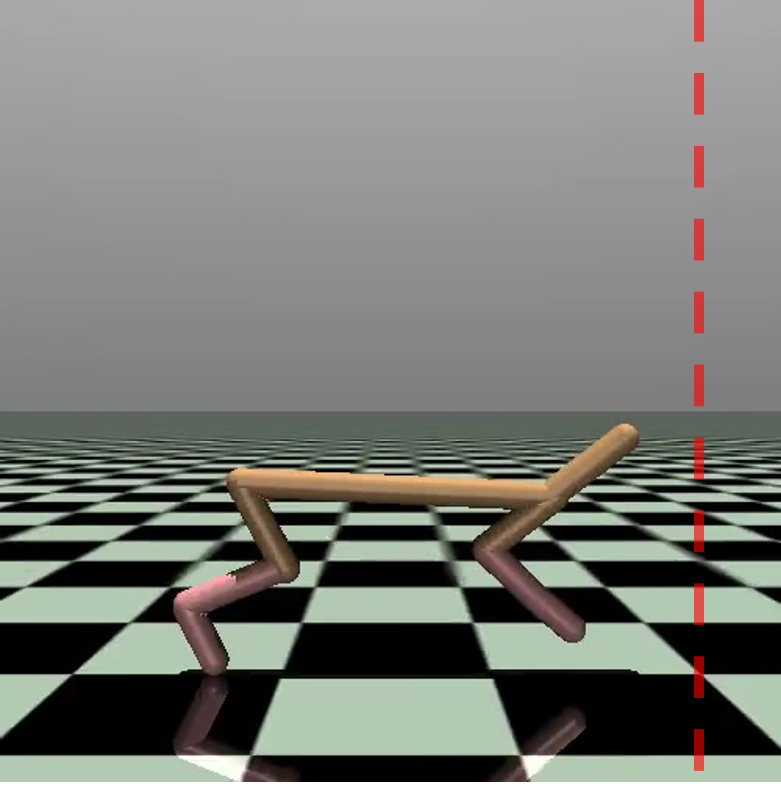}
	\end{subfigure}
	\begin{subfigure}{0.19\textwidth}
		\centering
		\includegraphics[width=0.8\columnwidth]{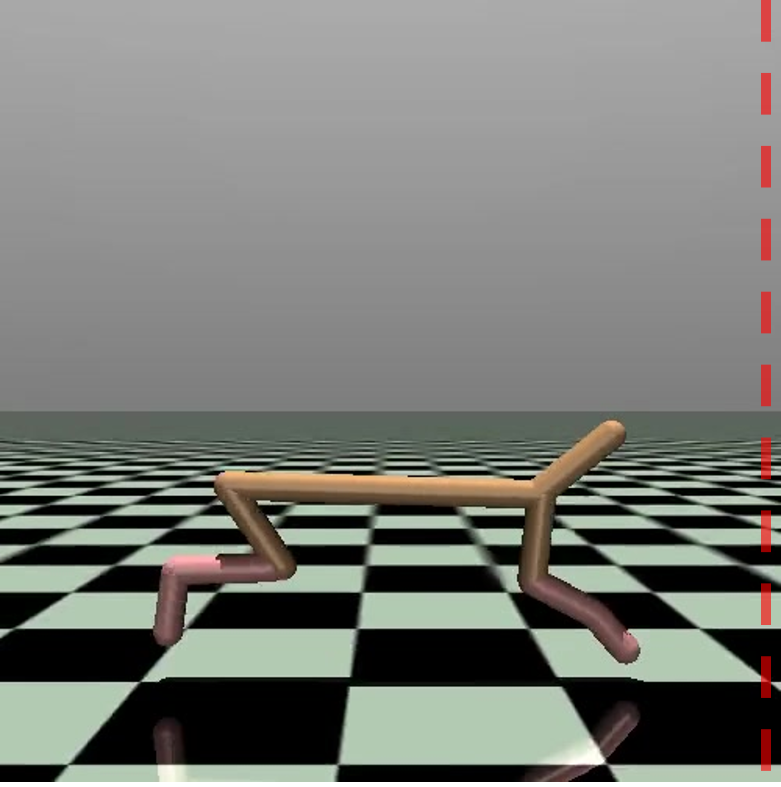}
	\end{subfigure}
	\\
	\begin{subfigure}{0.19\textwidth}
		\centering
		\includegraphics[width=0.8\columnwidth]{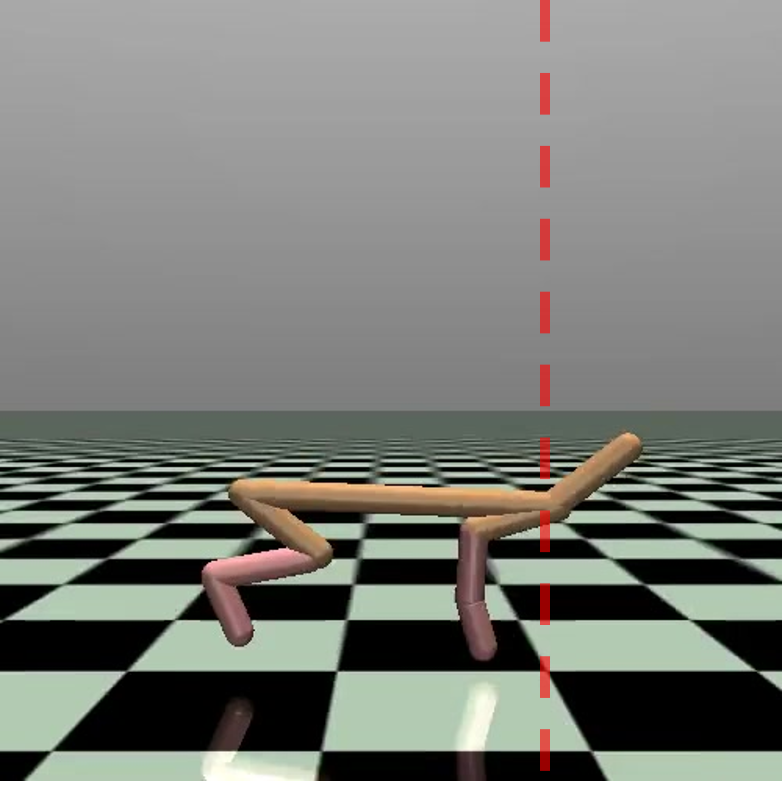}
	\end{subfigure}
	\begin{subfigure}{0.19\textwidth}
		\centering
		\includegraphics[width=0.8\columnwidth]{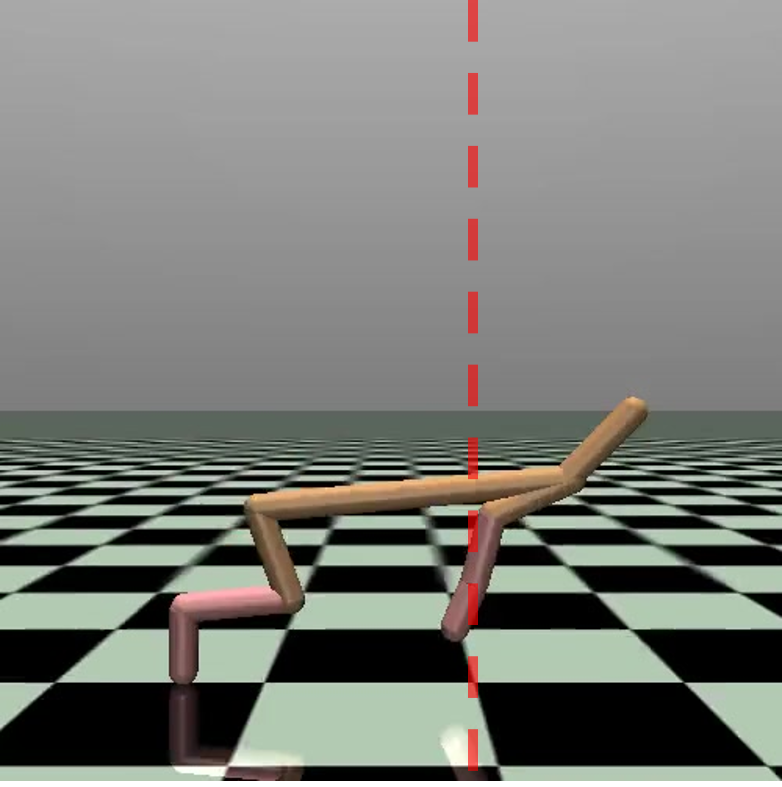}
	\end{subfigure}
	\begin{subfigure}{0.19\textwidth}
		\centering
		\includegraphics[width=0.8\columnwidth]{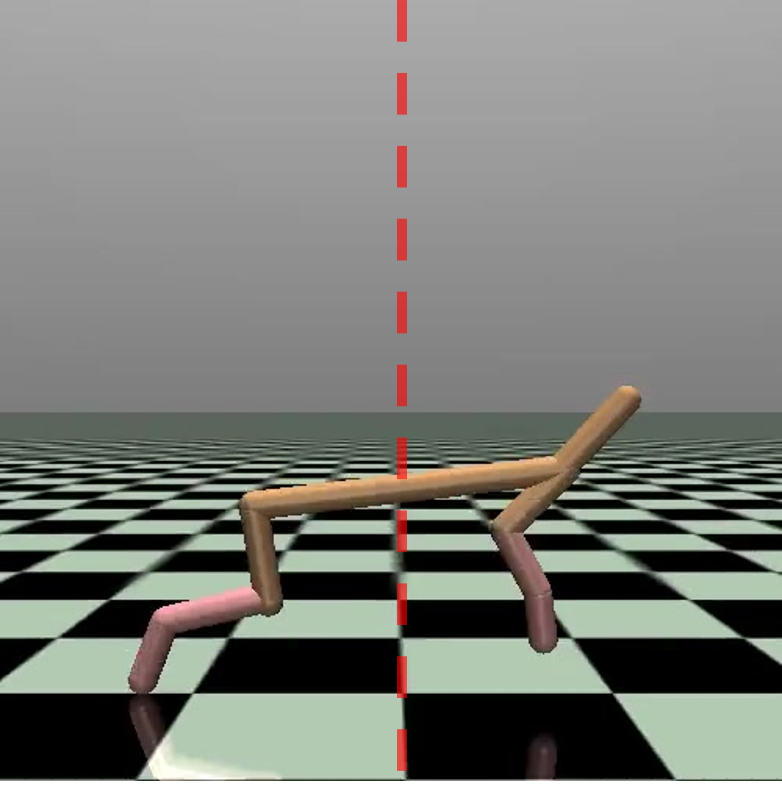}
	\end{subfigure}
	\begin{subfigure}{0.19\textwidth}
		\centering
		\includegraphics[width=0.8\columnwidth]{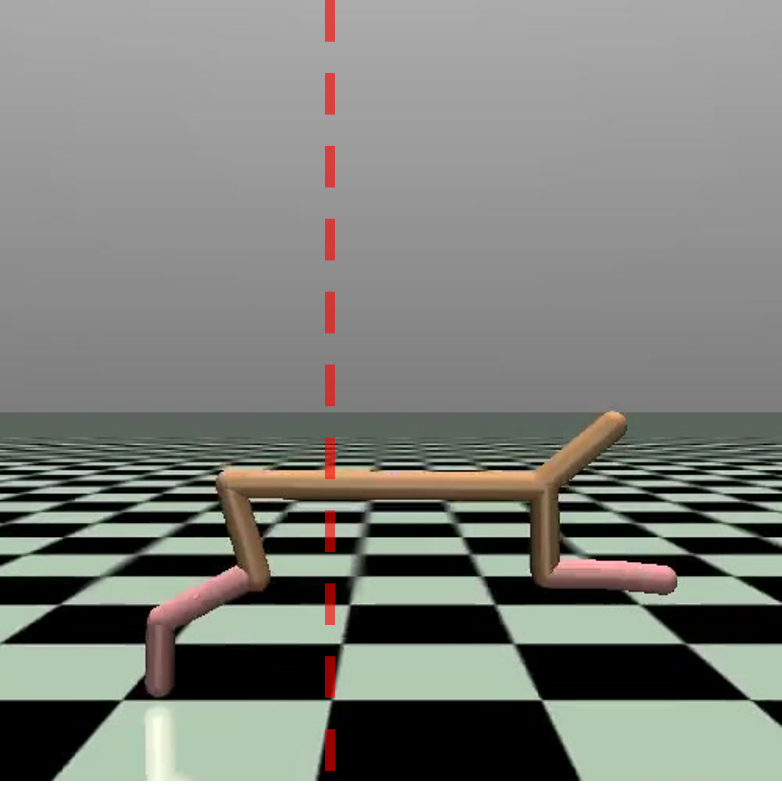}
	\end{subfigure}
	\begin{subfigure}{0.19\textwidth}
		\centering
		\includegraphics[width=0.8\columnwidth]{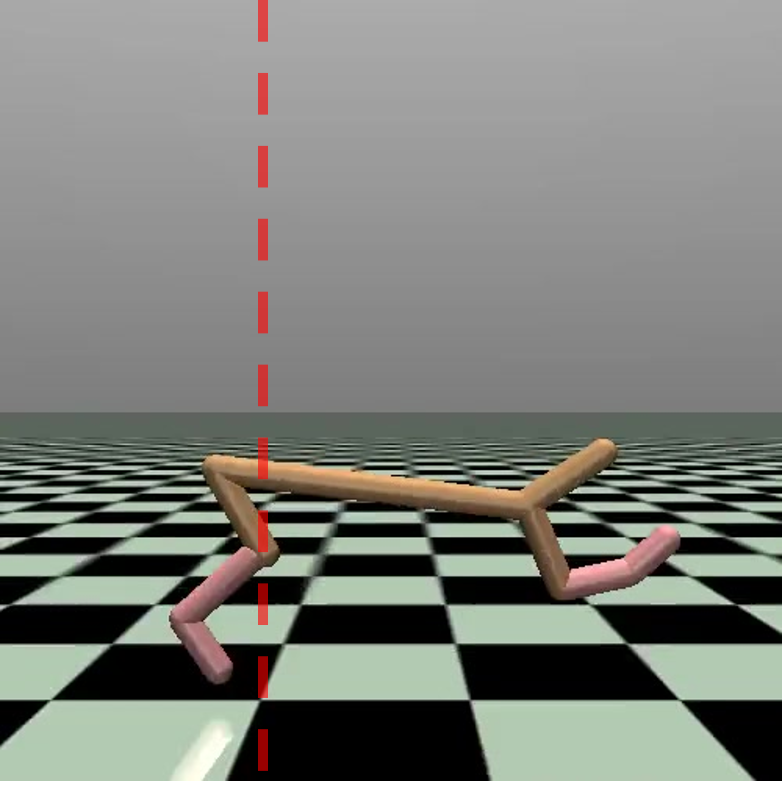}
	\end{subfigure}
	\caption{Consecutive Frames of Half-Cheetah trained with TRPO (top row) and FPG (bottom row) respectively under $\delta=100$ attack. The dashed red line serves as a stationary reference object. TRPO was fooled to learn a ''running backward'' policy, contrasted with the normal ''running forward'' policy learned by FPG.}
	\label{fig:cheetah}
\end{figure*}
\begin{figure*}[h!]
	\begin{subfigure}{0.305\textwidth}
		\centering
		\includegraphics[width=1\columnwidth]{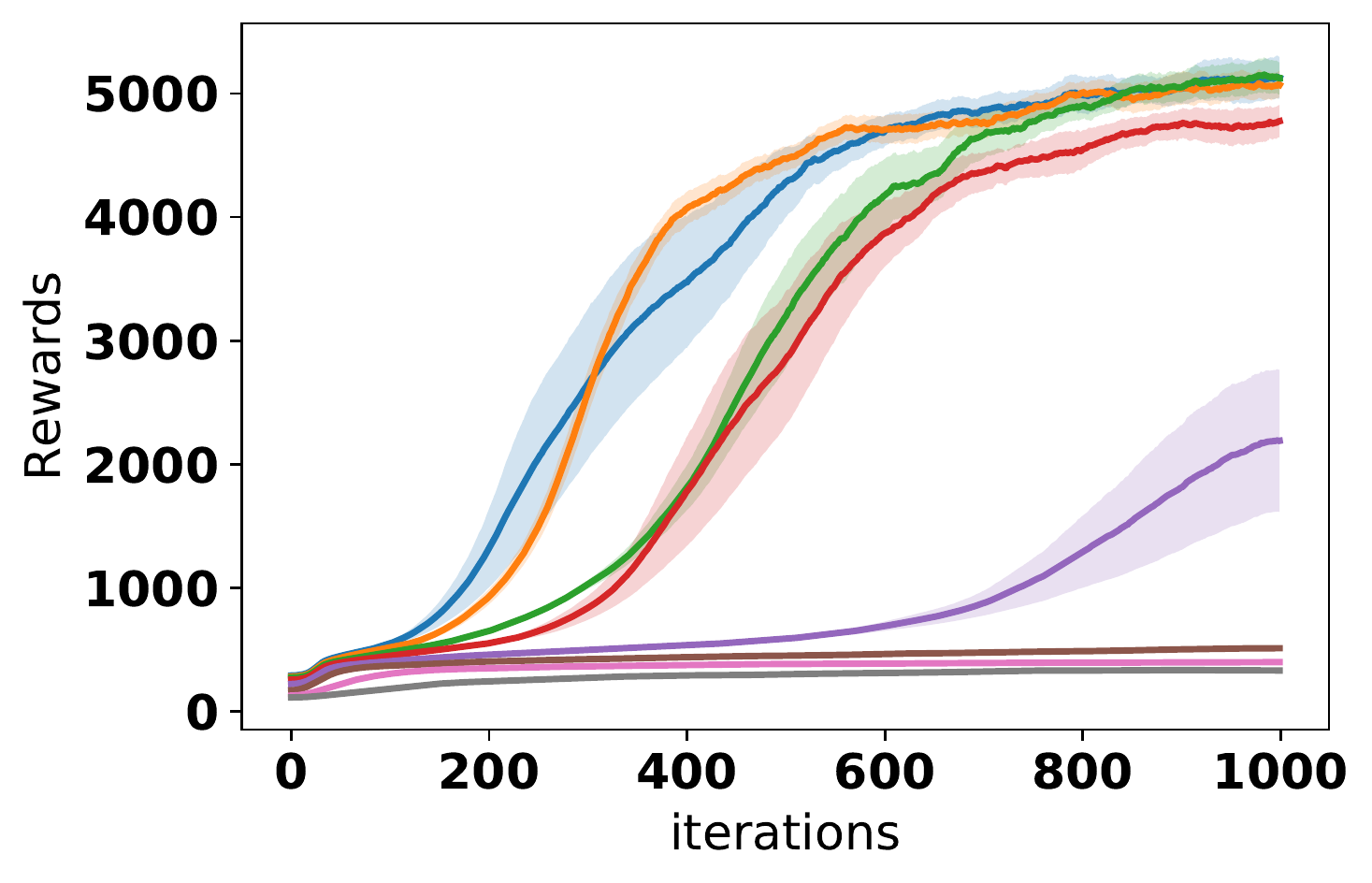}
		\caption{TRPO Rewards}
		\label{fig:TRPO_humanoid}
	\end{subfigure}
	\begin{subfigure}{0.305\textwidth}
		\centering
		\includegraphics[width=1\columnwidth]{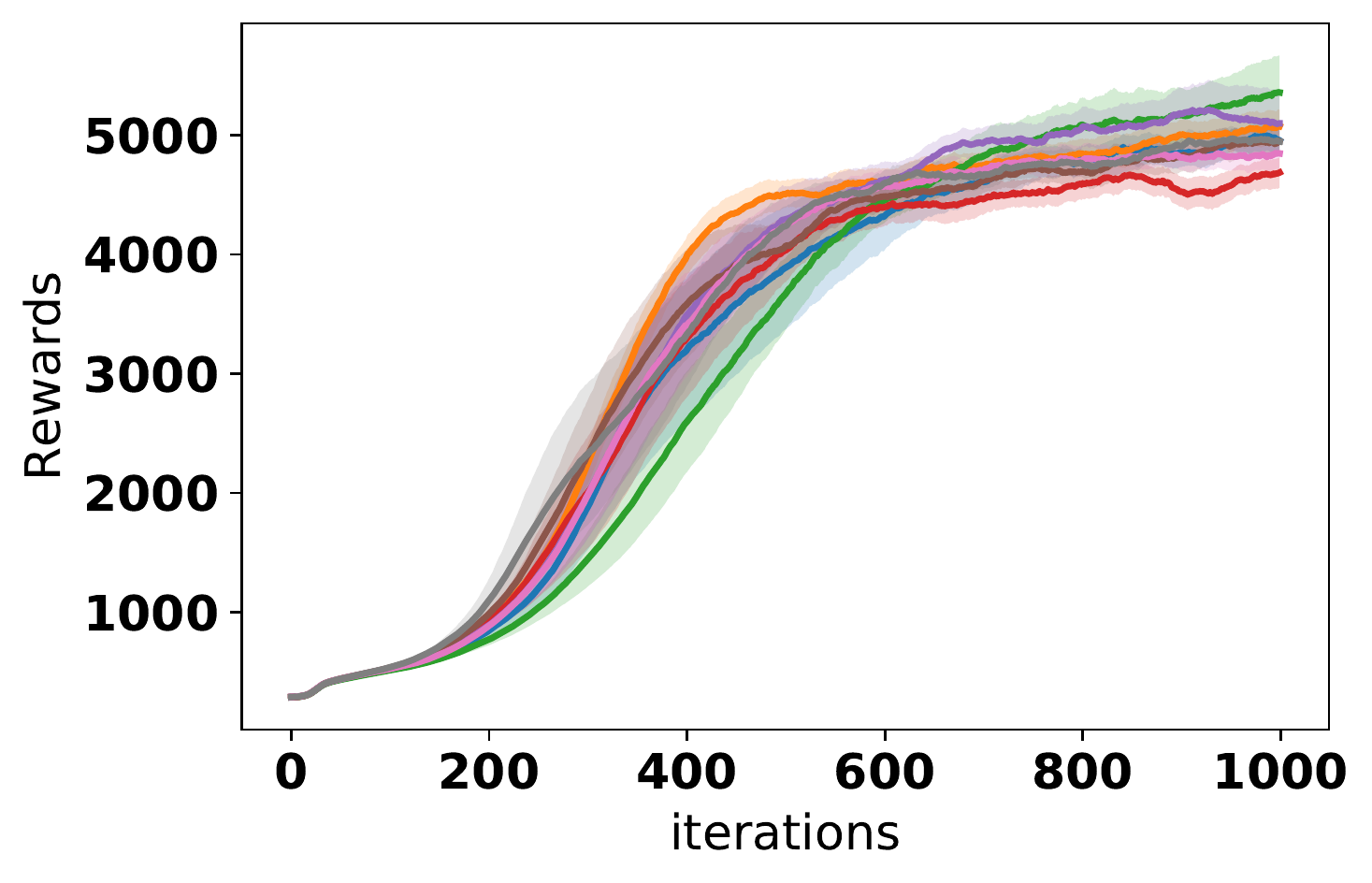}
		\caption{FPG Rewards}
		\label{fig:FPG_humanoid}
	\end{subfigure}
	\begin{subfigure}{0.39\textwidth}
		\centering
		\includegraphics[width=1\columnwidth]{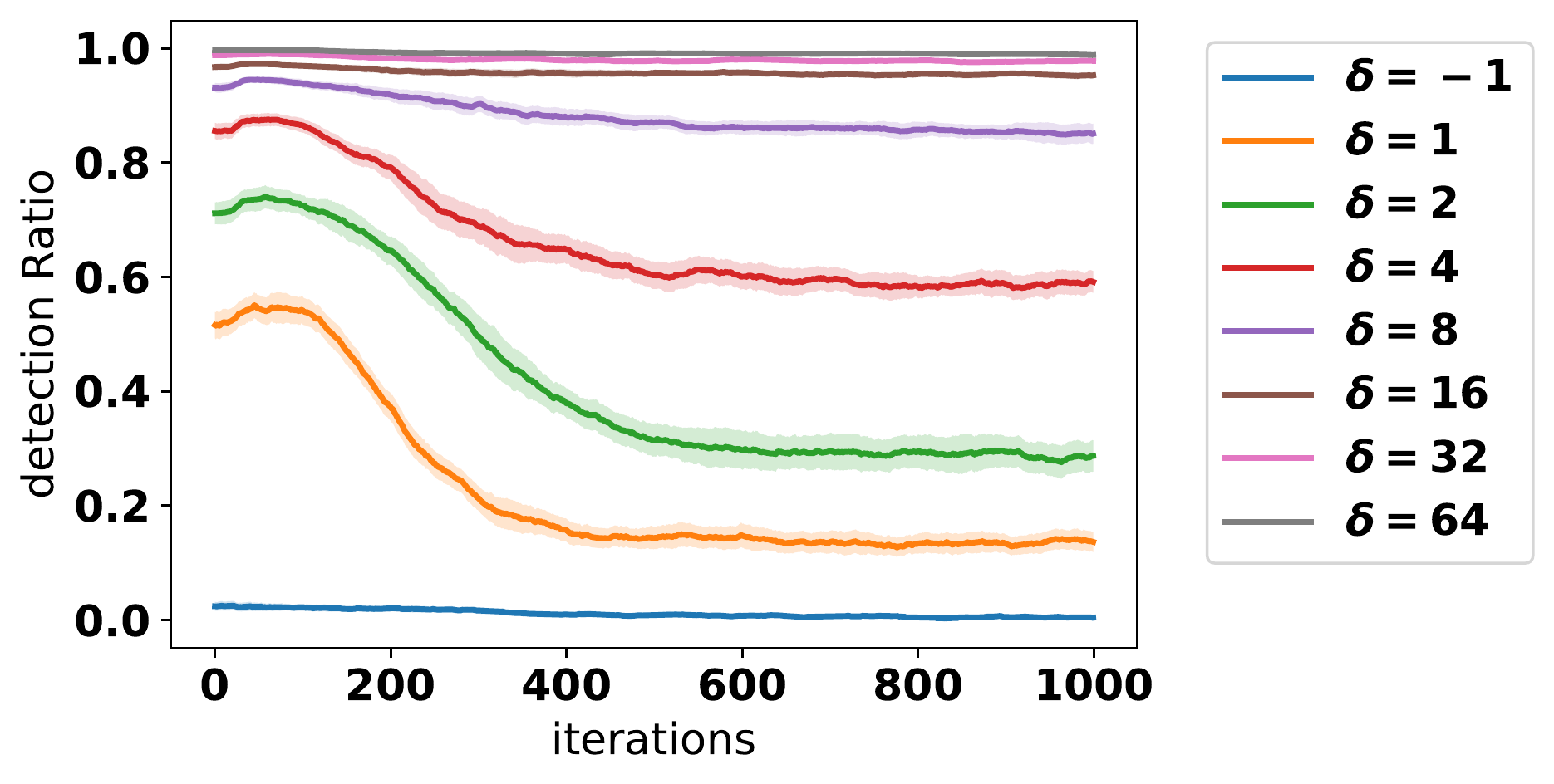}
		\caption{FPG Detection Ratio}
		\label{fig:FPG_detection_ratio}
	\end{subfigure}
	\caption{Detailed Results on Humanoid-v3.}
	\label{fig:humanoid_result}
\end{figure*}

\paragraph{Attack mechanism:} While designing and calculating the \textit{optimal} attack strategy against a deep RL algorithm is still a challenging problem and active area of research \cite{ma2019policy,zhang2020adaptive}, here we describe the poisoning strategy used in our empirical evaluation, which, despite being simple, can fool non-robust RL algorithms with ease. Conceptually, policy gradient methods can be viewed as a stochastic gradient ascent method, where each iteration can be simplified as:
\begin{eqnarray}
	\theta^{(t+1)} = \theta^{(t)} + g^{(t)}
\end{eqnarray}
where $g^{(t)}$ is a gradient step that ideally points in the direction of fastest policy improvement. Assuming that $g^{(t)}$ is a good estimate of the gradient direction, then a simple attack strategy is to try to perturb $ g^{(t)}$ to point in the $-g^{(t)}$ direction, in which case the policy, rather than improving, will deteriorate as learning proceed. A straightforward way to achieve this is to flip the rewards and multiply them by a big constant $\delta$ in the adversarial episodes. In the linear regression subproblem of Alg. \ref{alg:q_npg_sample}, this would result in a set of $(x,y)$ pairs whose $y$ becomes $-\delta y$. This in expectation will make the best linear regressor $w$ point to the opposite direction, which is precisely what we want. 

This attack strategy is therefore parameterized by a single parameter $\delta$, which guides the magnitude of the attack, and is \textbf{adaptively tuned} against each learning algorithm in the experiments: Throughout the experiment, we set the contamination level $\epsilon=0.01$, and tune $\delta$ among the values of $[1,2,4,8,16,32,64]$ to find the most effective magnitude against each learning algorithm. All experiments are repeated with 3 random seeds and the mean and standard deviations are plotted in the figures.

\paragraph{Results:} The experiment results are shown in Figure \ref{fig:exp_result}. Consistent patterns can be observed across all environments: vanilla \texttt{TRPO} performs well without attack but fails completely under the adaptive attack (which choose $\delta=64$ in all environments). \texttt{FPG}, on the other hand, matches the performance of vanilla \texttt{TRPO} with or without attack. Figure \ref{fig:cheetah} showcase two half-cheetah control policies learned by \texttt{TRPO} and \texttt{FPG} under attack with $\delta=100$. Interestingly, due to the large negative adversarial rewards, TRPO actually learns the ``running backward'' policy, showing that our attack strategy indeed achieves what it's designed for. In contrast, \texttt{FPG} is still able to learn the ''running forward'' policy despite the attack.

Figure \ref{fig:humanoid_result} shows the detailed performances of \texttt{TRPO} and \texttt{FPG} across different $\delta$'s on the hardest \textit{Humanoid} environment. 
One can observe that \texttt{TRPO} actually learns robustly under attacks of small magnitude ($\delta=1,2,4$) and achieves similar performances to itself in clean environments, verifying our theoretical result in Theorem \ref{thm:npg_natural}.
In contrast, \texttt{FPG} remains robust across all values of $\delta$'s. Figure \ref{fig:FPG_detection_ratio} shows the proportion of adversary data detected and removed by \texttt{FPG}'s filtering subroutine throughout the learning process. One can observe that as the attack norm $\delta$ increases, the filtering algorithm also does a better job detecting the adversarial data and thus protect the algorithm from getting inaccurate gradient estimates. Similar patterns can be observed in all the other environments, and we defer the additional figures to the appendix.

\section{Discussions}
To summarize, in this work we present a robust policy gradient algorithm \texttt{FPG}, and show theoretically and empirically that it can learn in the presence of strong data corruption. Despite our results, many open questions remain unclear and are interesting directions to pursue further:
\begin{enumerate}
	\item \texttt{FPG} does not handle exploration and relies on an exploratory initial distribution. Can we design algorithms that achieve the same \textit{dimension-free} robustness guarantee without such assumptions?
	\item Our $O(\epsilon^{1/4})$ upper-bound and $O(\epsilon)$ lower-bound are not tight. Information theoretically, what is the best robustness guarantee one can achieve under $\epsilon$-contamination?
	\item The \texttt{SEVER} algorithm requires computing the top eigenvalue of an $n\times d$ matrix, which is memory and time consuming when using large neural networks (large $d$). More computationally efficient robust learning method will be extremely valuable to make FPG truly scale.
	\item In the experiment, we focus on TRPO as the closest variant of NPG. Can other policy gradient algorithm, such as PPO and SAC, be robustified in similar fashions and achieve strong empirical performance?
\end{enumerate}
We believe that answering these questions will be important steps towards more robust reinforcement learning. 

\section{Acknowledgements}
We would like to thank Ankit Pensia and Ilias Diakonikolas for valuable discussions on SEVER and other robust statistics techniques. Xiaojin Zhu acknowledges NSF grants 1545481, 1704117, 1836978, 2041428, 2023239 and MADLab AF CoE FA9550-18-1-0166. Xuezhou Zhang is supported in part by NSF Award DMS-2023239.

\bibliography{robustPG}

\begin{thebibliography}{72}
\providecommand{\natexlab}[1]{#1}
\providecommand{\url}[1]{\texttt{#1}}
\expandafter\ifx\csname urlstyle\endcsname\relax
  \providecommand{\doi}[1]{doi: #1}\else
  \providecommand{\doi}{doi: \begingroup \urlstyle{rm}\Url}\fi

\bibitem[Agarwal et~al.(2019)Agarwal, Kakade, Lee, and
  Mahajan]{agarwal2019optimality}
Agarwal, A., Kakade, S.~M., Lee, J.~D., and Mahajan, G.
\newblock On the theory of policy gradient methods: Optimality, approximation,
  and distribution shift.
\newblock \emph{arXiv preprint arXiv:1908.00261}, 2019.

\bibitem[Agarwal et~al.(2020{\natexlab{a}})Agarwal, Henaff, Kakade, and
  Sun]{agarwal2020pc}
Agarwal, A., Henaff, M., Kakade, S., and Sun, W.
\newblock Pc-pg: Policy cover directed exploration for provable policy gradient
  learning.
\newblock \emph{arXiv preprint arXiv:2007.08459}, 2020{\natexlab{a}}.

\bibitem[Agarwal et~al.(2020{\natexlab{b}})Agarwal, Kakade, Krishnamurthy, and
  Sun]{agarwal2020flambe}
Agarwal, A., Kakade, S., Krishnamurthy, A., and Sun, W.
\newblock Flambe: Structural complexity and representation learning of low rank
  mdps.
\newblock \emph{Advances in Neural Information Processing Systems}, 33,
  2020{\natexlab{b}}.

\bibitem[Akkaya et~al.(2019)Akkaya, Andrychowicz, Chociej, Litwin, McGrew,
  Petron, Paino, Plappert, Powell, Ribas, et~al.]{akkaya2019solving}
Akkaya, I., Andrychowicz, M., Chociej, M., Litwin, M., McGrew, B., Petron, A.,
  Paino, A., Plappert, M., Powell, G., Ribas, R., et~al.
\newblock Solving rubik's cube with a robot hand.
\newblock \emph{arXiv preprint arXiv:1910.07113}, 2019.

\bibitem[Auer et~al.(2009)Auer, Jaksch, and Ortner]{auer2009near}
Auer, P., Jaksch, T., and Ortner, R.
\newblock Near-optimal regret bounds for reinforcement learning.
\newblock In \emph{Advances in neural information processing systems}, pp.\
  89--96, 2009.

\bibitem[Ayoub et~al.(2020)Ayoub, Jia, Szepesvari, Wang, and
  Yang]{ayoub2020model}
Ayoub, A., Jia, Z., Szepesvari, C., Wang, M., and Yang, L.~F.
\newblock Model-based reinforcement learning with value-targeted regression.
\newblock \emph{arXiv preprint arXiv:2006.01107}, 2020.

\bibitem[Azar et~al.(2017)Azar, Osband, and Munos]{azar2017minimax}
Azar, M.~G., Osband, I., and Munos, R.
\newblock Minimax regret bounds for reinforcement learning.
\newblock In \emph{International Conference on Machine Learning}, pp.\
  263--272, 2017.

\bibitem[Berner et~al.(2019)Berner, Brockman, Chan, Cheung, D{\k{e}}biak,
  Dennison, Farhi, Fischer, Hashme, Hesse, et~al.]{berner2019dota}
Berner, C., Brockman, G., Chan, B., Cheung, V., D{\k{e}}biak, P., Dennison, C.,
  Farhi, D., Fischer, Q., Hashme, S., Hesse, C., et~al.
\newblock Dota 2 with large scale deep reinforcement learning.
\newblock \emph{arXiv preprint arXiv:1912.06680}, 2019.

\bibitem[Brafman \& Tennenholtz(2002)Brafman and Tennenholtz]{brafman2002r}
Brafman, R.~I. and Tennenholtz, M.
\newblock R-max-a general polynomial time algorithm for near-optimal
  reinforcement learning.
\newblock \emph{Journal of Machine Learning Research}, 3\penalty0
  (Oct):\penalty0 213--231, 2002.

\bibitem[Bubeck \& Cesa-Bianchi(2012)Bubeck and Cesa-Bianchi]{bubeck2012regret}
Bubeck, S. and Cesa-Bianchi, N.
\newblock Regret analysis of stochastic and nonstochastic multi-armed bandit
  problems.
\newblock \emph{arXiv preprint arXiv:1204.5721}, 2012.

\bibitem[Cai et~al.(2019)Cai, Yang, Jin, and Wang]{cai2019provably}
Cai, Q., Yang, Z., Jin, C., and Wang, Z.
\newblock Provably efficient exploration in policy optimization.
\newblock \emph{arXiv preprint arXiv:1912.05830}, 2019.

\bibitem[Charikar et~al.(2017)Charikar, Steinhardt, and
  Valiant]{charikar2017learning}
Charikar, M., Steinhardt, J., and Valiant, G.
\newblock Learning from untrusted data.
\newblock In \emph{Proceedings of the 49th Annual ACM SIGACT Symposium on
  Theory of Computing}, pp.\  47--60, 2017.

\bibitem[Cheung et~al.(2019)Cheung, Simchi-Levi, and Zhu]{cheung2019non}
Cheung, W.~C., Simchi-Levi, D., and Zhu, R.
\newblock Non-stationary reinforcement learning: The blessing of (more)
  optimism.
\newblock \emph{Available at SSRN 3397818}, 2019.

\bibitem[Dann \& Brunskill(2015)Dann and Brunskill]{dann2015sample}
Dann, C. and Brunskill, E.
\newblock Sample complexity of episodic fixed-horizon reinforcement learning.
\newblock In \emph{Advances in Neural Information Processing Systems}, pp.\
  2818--2826, 2015.

\bibitem[Dann et~al.(2017)Dann, Lattimore, and Brunskill]{dann2017unifying}
Dann, C., Lattimore, T., and Brunskill, E.
\newblock Unifying pac and regret: Uniform pac bounds for episodic
  reinforcement learning.
\newblock In \emph{Advances in Neural Information Processing Systems}, pp.\
  5713--5723, 2017.

\bibitem[Derman et~al.(2020)Derman, Mankowitz, Mann, and
  Mannor]{derman2020bayesian}
Derman, E., Mankowitz, D., Mann, T., and Mannor, S.
\newblock A bayesian approach to robust reinforcement learning.
\newblock In \emph{Uncertainty in Artificial Intelligence}, pp.\  648--658.
  PMLR, 2020.

\bibitem[Diakonikolas \& Kane(2019)Diakonikolas and
  Kane]{diakonikolas2019recent}
Diakonikolas, I. and Kane, D.~M.
\newblock Recent advances in algorithmic high-dimensional robust statistics.
\newblock \emph{arXiv preprint arXiv:1911.05911}, 2019.

\bibitem[Diakonikolas et~al.(2016)Diakonikolas, Kamath, Kane, Li, Moitra, and
  Stewart]{diakonikolas2016robust}
Diakonikolas, I., Kamath, G., Kane, D., Li, J., Moitra, A., and Stewart, A.
\newblock Robust estimators in high dimensions without the computational
  intractability.
\newblock In \emph{2016 IEEE 57th Annual Symposium on Foundations of Computer
  Science (FOCS)}, pp.\  655--664, 2016.

\bibitem[Diakonikolas et~al.(2017)Diakonikolas, Kamath, Kane, Li, Moitra, and
  Stewart]{diakonikolas2017being}
Diakonikolas, I., Kamath, G., Kane, D.~M., Li, J., Moitra, A., and Stewart, A.
\newblock Being robust (in high dimensions) can be practical.
\newblock \emph{arXiv preprint arXiv:1703.00893}, 2017.

\bibitem[Diakonikolas et~al.(2019)Diakonikolas, Kamath, Kane, Li, Steinhardt,
  and Stewart]{diakonikolas2019sever}
Diakonikolas, I., Kamath, G., Kane, D., Li, J., Steinhardt, J., and Stewart, A.
\newblock Sever: A robust meta-algorithm for stochastic optimization.
\newblock In \emph{International Conference on Machine Learning}, pp.\
  1596--1606, 2019.

\bibitem[Diakonikolas et~al.(2020)Diakonikolas, Kane, and
  Pensia]{diakonikolas2020outlier}
Diakonikolas, I., Kane, D.~M., and Pensia, A.
\newblock Outlier robust mean estimation with subgaussian rates via stability.
\newblock \emph{Advances in Neural Information Processing Systems}, 33, 2020.

\bibitem[Domingues et~al.(2020)Domingues, M{\'e}nard, Pirotta, Kaufmann, and
  Valko]{domingues2020kernel}
Domingues, O.~D., M{\'e}nard, P., Pirotta, M., Kaufmann, E., and Valko, M.
\newblock A kernel-based approach to non-stationary reinforcement learning in
  metric spaces.
\newblock \emph{arXiv preprint arXiv:2007.05078}, 2020.

\bibitem[Du et~al.(2019)Du, Luo, Wang, and Zhang]{du2019provably}
Du, S.~S., Luo, Y., Wang, R., and Zhang, H.
\newblock Provably efficient q-learning with function approximation via
  distribution shift error checking oracle.
\newblock In \emph{Advances in Neural Information Processing Systems}, pp.\
  8060--8070, 2019.

\bibitem[Even-Dar et~al.(2009)Even-Dar, Kakade, and Mansour]{even2009online}
Even-Dar, E., Kakade, S.~M., and Mansour, Y.
\newblock Online markov decision processes.
\newblock \emph{Mathematics of Operations Research}, 34\penalty0 (3):\penalty0
  726--736, 2009.

\bibitem[Eykholt et~al.(2018)Eykholt, Evtimov, Fernandes, Li, Rahmati, Xiao,
  Prakash, Kohno, and Song]{eykholt2018robust}
Eykholt, K., Evtimov, I., Fernandes, E., Li, B., Rahmati, A., Xiao, C.,
  Prakash, A., Kohno, T., and Song, D.
\newblock Robust physical-world attacks on deep learning visual classification.
\newblock In \emph{Proceedings of the IEEE Conference on Computer Vision and
  Pattern Recognition}, pp.\  1625--1634, 2018.

\bibitem[Gupta et~al.(2019)Gupta, Koren, and Talwar]{gupta2019better}
Gupta, A., Koren, T., and Talwar, K.
\newblock Better algorithms for stochastic bandits with adversarial
  corruptions.
\newblock \emph{arXiv preprint arXiv:1902.08647}, 2019.

\bibitem[Jiang et~al.(2017)Jiang, Krishnamurthy, Agarwal, Langford, and
  Schapire]{jiang2017contextual}
Jiang, N., Krishnamurthy, A., Agarwal, A., Langford, J., and Schapire, R.~E.
\newblock Contextual decision processes with low bellman rank are
  pac-learnable.
\newblock In \emph{International Conference on Machine Learning}, pp.\
  1704--1713. PMLR, 2017.

\bibitem[Jin et~al.(2018)Jin, Allen-Zhu, Bubeck, and Jordan]{jin2018q}
Jin, C., Allen-Zhu, Z., Bubeck, S., and Jordan, M.~I.
\newblock Is q-learning provably efficient?
\newblock In \emph{Advances in Neural Information Processing Systems}, pp.\
  4863--4873, 2018.

\bibitem[Jin et~al.(2019)Jin, Netrapalli, Ge, Kakade, and Jordan]{jin2019short}
Jin, C., Netrapalli, P., Ge, R., Kakade, S.~M., and Jordan, M.~I.
\newblock A short note on concentration inequalities for random vectors with
  subgaussian norm.
\newblock \emph{arXiv preprint arXiv:1902.03736}, 2019.

\bibitem[Jin et~al.(2020{\natexlab{a}})Jin, Jin, Luo, Sra, and
  Yu]{jin2020learning}
Jin, C., Jin, T., Luo, H., Sra, S., and Yu, T.
\newblock Learning adversarial markov decision processes with bandit feedback
  and unknown transition.
\newblock In \emph{International Conference on Machine Learning}, pp.\
  4860--4869. PMLR, 2020{\natexlab{a}}.

\bibitem[Jin et~al.(2020{\natexlab{b}})Jin, Yang, Wang, and
  Jordan]{jin2020provably}
Jin, C., Yang, Z., Wang, Z., and Jordan, M.~I.
\newblock Provably efficient reinforcement learning with linear function
  approximation.
\newblock In \emph{Conference on Learning Theory}, pp.\  2137--2143. PMLR,
  2020{\natexlab{b}}.

\bibitem[Jin \& Luo(2020)Jin and Luo]{jin2020simultaneously}
Jin, T. and Luo, H.
\newblock Simultaneously learning stochastic and adversarial episodic mdps with
  known transition.
\newblock \emph{arXiv preprint arXiv:2006.05606}, 2020.

\bibitem[Kakade \& Langford(2002)Kakade and Langford]{kakade2002approximately}
Kakade, S. and Langford, J.
\newblock Approximately optimal approximate reinforcement learning.
\newblock In \emph{ICML}, volume~2, pp.\  267--274, 2002.

\bibitem[Kakade et~al.(2020)Kakade, Krishnamurthy, Lowrey, Ohnishi, and
  Sun]{kakade2020information}
Kakade, S., Krishnamurthy, A., Lowrey, K., Ohnishi, M., and Sun, W.
\newblock Information theoretic regret bounds for online nonlinear control.
\newblock \emph{Advances in Neural Information Processing Systems}, 33, 2020.

\bibitem[Kakade(2001)]{kakade2001natural}
Kakade, S.~M.
\newblock A natural policy gradient.
\newblock \emph{Advances in neural information processing systems},
  14:\penalty0 1531--1538, 2001.

\bibitem[Lai et~al.(2016)Lai, Rao, and Vempala]{lai2016agnostic}
Lai, K.~A., Rao, A.~B., and Vempala, S.
\newblock Agnostic estimation of mean and covariance.
\newblock In \emph{2016 IEEE 57th Annual Symposium on Foundations of Computer
  Science (FOCS)}, pp.\  665--674. IEEE, 2016.

\bibitem[Lee et~al.(2020)Lee, Luo, Wei, and Zhang]{lee2020bias}
Lee, C.-W., Luo, H., Wei, C.-Y., and Zhang, M.
\newblock Bias no more: high-probability data-dependent regret bounds for
  adversarial bandits and mdps.
\newblock \emph{Advances in Neural Information Processing Systems}, 33, 2020.

\bibitem[Lykouris et~al.(2018)Lykouris, Mirrokni, and
  Paes~Leme]{lykouris2018stochastic}
Lykouris, T., Mirrokni, V., and Paes~Leme, R.
\newblock Stochastic bandits robust to adversarial corruptions.
\newblock In \emph{Proceedings of the 50th Annual ACM SIGACT Symposium on
  Theory of Computing}, pp.\  114--122, 2018.

\bibitem[Lykouris et~al.(2019)Lykouris, Simchowitz, Slivkins, and
  Sun]{lykouris2019corruption}
Lykouris, T., Simchowitz, M., Slivkins, A., and Sun, W.
\newblock Corruption robust exploration in episodic reinforcement learning.
\newblock \emph{arXiv preprint arXiv:1911.08689}, 2019.

\bibitem[Ma et~al.(2019)Ma, Zhang, Sun, and Zhu]{ma2019policy}
Ma, Y., Zhang, X., Sun, W., and Zhu, J.
\newblock Policy poisoning in batch reinforcement learning and control.
\newblock In \emph{Advances in Neural Information Processing Systems}, pp.\
  14570--14580, 2019.

\bibitem[Neff \& Nagy(2016)Neff and Nagy]{neff2016automation}
Neff, G. and Nagy, P.
\newblock Automation, algorithms, and politics| talking to bots: Symbiotic
  agency and the case of tay.
\newblock \emph{International Journal of Communication}, 10:\penalty0 17, 2016.

\bibitem[Neu \& Olkhovskaya(2020)Neu and Olkhovskaya]{neu2020online}
Neu, G. and Olkhovskaya, J.
\newblock Online learning in mdps with linear function approximation and bandit
  feedback.
\newblock \emph{arXiv preprint arXiv:2007.01612}, 2020.

\bibitem[Neu et~al.(2010)Neu, Gy{\"o}rgy, and Szepesv{\'a}ri]{neu2010online}
Neu, G., Gy{\"o}rgy, A., and Szepesv{\'a}ri, C.
\newblock The online loop-free stochastic shortest-path problem.
\newblock In \emph{COLT}, volume 2010, pp.\  231--243. Citeseer, 2010.

\bibitem[Neu et~al.(2012)Neu, Gyorgy, and Szepesv{\'a}ri]{neu2012adversarial}
Neu, G., Gyorgy, A., and Szepesv{\'a}ri, C.
\newblock The adversarial stochastic shortest path problem with unknown
  transition probabilities.
\newblock In \emph{Artificial Intelligence and Statistics}, pp.\  805--813,
  2012.

\bibitem[Ornik \& Topcu(2019)Ornik and Topcu]{ornik2019learning}
Ornik, M. and Topcu, U.
\newblock Learning and planning for time-varying mdps using maximum likelihood
  estimation.
\newblock \emph{arXiv preprint arXiv:1911.12976}, 2019.

\bibitem[Ortner et~al.(2019)Ortner, Gajane, and Auer]{ortner2019variational}
Ortner, R., Gajane, P., and Auer, P.
\newblock Variational regret bounds for reinforcement learning.
\newblock In \emph{UAI}, pp.\ ~16, 2019.

\bibitem[Osband \& Van~Roy(2014)Osband and Van~Roy]{osband2014model}
Osband, I. and Van~Roy, B.
\newblock Model-based reinforcement learning and the eluder dimension.
\newblock \emph{Advances in Neural Information Processing Systems},
  27:\penalty0 1466--1474, 2014.

\bibitem[Osband \& Van~Roy(2016)Osband and Van~Roy]{osband2016lower}
Osband, I. and Van~Roy, B.
\newblock On lower bounds for regret in reinforcement learning.
\newblock \emph{arXiv preprint arXiv:1608.02732}, 2016.

\bibitem[Petersen et~al.(2012)Petersen, Ugrinovskii, and
  Savkin]{petersen2012robust}
Petersen, I.~R., Ugrinovskii, V.~A., and Savkin, A.~V.
\newblock \emph{Robust Control Design Using H-$\infty$ Methods}.
\newblock Springer Science \& Business Media, 2012.

\bibitem[Pinto et~al.(2017)Pinto, Davidson, Sukthankar, and
  Gupta]{pinto2017robust}
Pinto, L., Davidson, J., Sukthankar, R., and Gupta, A.
\newblock Robust adversarial reinforcement learning.
\newblock In \emph{International Conference on Machine Learning}, pp.\
  2817--2826. PMLR, 2017.

\bibitem[Rosenberg \& Mansour(2019)Rosenberg and Mansour]{rosenberg2019online}
Rosenberg, A. and Mansour, Y.
\newblock Online convex optimization in adversarial markov decision processes.
\newblock \emph{arXiv preprint arXiv:1905.07773}, 2019.

\bibitem[Schulman et~al.(2015{\natexlab{a}})Schulman, Levine, Abbeel, Jordan,
  and Moritz]{schulman2015trust}
Schulman, J., Levine, S., Abbeel, P., Jordan, M., and Moritz, P.
\newblock Trust region policy optimization.
\newblock In \emph{International conference on machine learning}, pp.\
  1889--1897, 2015{\natexlab{a}}.

\bibitem[Schulman et~al.(2015{\natexlab{b}})Schulman, Moritz, Levine, Jordan,
  and Abbeel]{schulman2015high}
Schulman, J., Moritz, P., Levine, S., Jordan, M., and Abbeel, P.
\newblock High-dimensional continuous control using generalized advantage
  estimation.
\newblock \emph{arXiv preprint arXiv:1506.02438}, 2015{\natexlab{b}}.

\bibitem[Schulman et~al.(2017)Schulman, Wolski, Dhariwal, Radford, and
  Klimov]{schulman2017proximal}
Schulman, J., Wolski, F., Dhariwal, P., Radford, A., and Klimov, O.
\newblock Proximal policy optimization algorithms.
\newblock \emph{arXiv preprint arXiv:1707.06347}, 2017.

\bibitem[Shalev-Shwartz et~al.(2011)]{shalev2011online}
Shalev-Shwartz, S. et~al.
\newblock Online learning and online convex optimization.
\newblock \emph{Foundations and trends in Machine Learning}, 4\penalty0
  (2):\penalty0 107--194, 2011.

\bibitem[Sun et~al.(2019)Sun, Jiang, Krishnamurthy, Agarwal, and
  Langford]{sun2019model}
Sun, W., Jiang, N., Krishnamurthy, A., Agarwal, A., and Langford, J.
\newblock Model-based rl in contextual decision processes: Pac bounds and
  exponential improvements over model-free approaches.
\newblock In \emph{Conference on Learning Theory}, pp.\  2898--2933. PMLR,
  2019.

\bibitem[Sutton et~al.(1999)Sutton, McAllester, Singh, and
  Mansour]{sutton1999policy}
Sutton, R.~S., McAllester, D.~A., Singh, S.~P., and Mansour, Y.
\newblock Policy gradient methods for reinforcement learning with function
  approximation.
\newblock In \emph{Advances in Neural Information Processing Systems},
  volume~99, pp.\  1057--1063, 1999.

\bibitem[Todorov et~al.(2012)Todorov, Erez, and Tassa]{todorov2012mujoco}
Todorov, E., Erez, T., and Tassa, Y.
\newblock Mujoco: A physics engine for model-based control.
\newblock In \emph{2012 IEEE/RSJ International Conference on Intelligent Robots
  and Systems}, pp.\  5026--5033. IEEE, 2012.

\bibitem[Tropp(2015)]{tropp2015introduction}
Tropp, J.~A.
\newblock An introduction to matrix concentration inequalities.
\newblock \emph{arXiv preprint arXiv:1501.01571}, 2015.

\bibitem[Williams(1992)]{williams1992simple}
Williams, R.~J.
\newblock Simple statistical gradient-following algorithms for connectionist
  reinforcement learning.
\newblock \emph{Machine learning}, 8\penalty0 (3-4):\penalty0 229--256, 1992.

\bibitem[Yadkori et~al.(2013)Yadkori, Bartlett, Kanade, Seldin, and
  Szepesv{\'a}ri]{yadkori2013online}
Yadkori, Y.~A., Bartlett, P.~L., Kanade, V., Seldin, Y., and Szepesv{\'a}ri, C.
\newblock Online learning in markov decision processes with adversarially
  chosen transition probability distributions.
\newblock In \emph{Advances in neural information processing systems}, pp.\
  2508--2516, 2013.

\bibitem[Yang \& Wang(2019{\natexlab{a}})Yang and Wang]{yang2019sample}
Yang, L. and Wang, M.
\newblock Sample-optimal parametric q-learning using linearly additive
  features.
\newblock In \emph{International Conference on Machine Learning}, pp.\
  6995--7004. PMLR, 2019{\natexlab{a}}.

\bibitem[Yang \& Wang(2019{\natexlab{b}})Yang and Wang]{yang2019reinforcement}
Yang, L.~F. and Wang, M.
\newblock Reinforcement learning in feature space: Matrix bandit, kernels, and
  regret bound.
\newblock \emph{arXiv preprint arXiv:1905.10389}, 2019{\natexlab{b}}.

\bibitem[Zanette et~al.(2020)Zanette, Brandfonbrener, Brunskill, Pirotta, and
  Lazaric]{zanette2020frequentist}
Zanette, A., Brandfonbrener, D., Brunskill, E., Pirotta, M., and Lazaric, A.
\newblock Frequentist regret bounds for randomized least-squares value
  iteration.
\newblock In \emph{International Conference on Artificial Intelligence and
  Statistics}, pp.\  1954--1964, 2020.

\bibitem[Zhang et~al.(2020{\natexlab{a}})Zhang, Hu, and Basar]{zhang2020policy}
Zhang, K., Hu, B., and Basar, T.
\newblock Policy optimization for h-2 linear control with h-$\infty$ robustness
  guarantee: Implicit regularization and global convergence.
\newblock In \emph{Learning for Dynamics and Control}, pp.\  179--190. PMLR,
  2020{\natexlab{a}}.

\bibitem[Zhang et~al.(2020{\natexlab{b}})Zhang, Hu, and
  Basar]{zhang2020stability}
Zhang, K., Hu, B., and Basar, T.
\newblock On the stability and convergence of robust adversarial reinforcement
  learning: A case study on linear quadratic systems.
\newblock \emph{Advances in Neural Information Processing Systems}, 33,
  2020{\natexlab{b}}.

\bibitem[Zhang et~al.(2021)Zhang, Zhang, Hu, and
  Ba{\c{s}}ar]{zhang2021derivative}
Zhang, K., Zhang, X., Hu, B., and Ba{\c{s}}ar, T.
\newblock Derivative-free policy optimization for risk-sensitive and robust
  control design: Implicit regularization and sample complexity.
\newblock \emph{arXiv preprint arXiv:2101.01041}, 2021.

\bibitem[Zhang et~al.(2020{\natexlab{c}})Zhang, Ma, Singla, and
  Zhu]{zhang2020adaptive}
Zhang, X., Ma, Y., Singla, A., and Zhu, X.
\newblock Adaptive reward-poisoning attacks against reinforcement learning.
\newblock \emph{arXiv preprint arXiv:2003.12613}, 2020{\natexlab{c}}.

\bibitem[Zhou et~al.(2020)Zhou, He, and Gu]{zhou2020provably}
Zhou, D., He, J., and Gu, Q.
\newblock Provably efficient reinforcement learning for discounted mdps with
  feature mapping.
\newblock \emph{arXiv preprint arXiv:2006.13165}, 2020.

\bibitem[Zhou \& Doyle(1998)Zhou and Doyle]{zhou1998essentials}
Zhou, K. and Doyle, J.~C.
\newblock \emph{Essentials of robust control}, volume 104.
\newblock Prentice hall Upper Saddle River, NJ, 1998.

\bibitem[Zimin \& Neu(2013)Zimin and Neu]{zimin2013online}
Zimin, A. and Neu, G.
\newblock Online learning in episodic markovian decision processes by relative
  entropy policy search.
\newblock In \emph{Advances in neural information processing systems}, pp.\
  1583--1591, 2013.

\bibitem[Zinkevich(2003)]{zinkevich2003online}
Zinkevich, M.
\newblock Online convex programming and generalized infinitesimal gradient
  ascent.
\newblock In \emph{Proceedings of the 20th international conference on machine
  learning (icml-03)}, pp.\  928--936, 2003.

\end{thebibliography}
\bibliographystyle{icml2021}

\newpage
\onecolumn
\appendix
\appendixpage

\section*{Table of Contents}
\begin{itemize}
	\item Section \ref{sec:lb} presents the proof for the lower-bound result.
	\item Section \ref{sec:prop} derives the mean and variance of the $Q$-sampler.
	\item Section \ref{sec:sec4} presents the proof for the natural robustness of NPG.
	\item Section \ref{sec:sever} presents the modified analysis of SEVER.
	\item Section \ref{sec:sec5} presents the proof for the robustness of FPG.
	\item Section \ref{sec:code} presents the pseudo-code of the TRPO variant of the FPG algorithm used in the experiments and additional experiment details.
\end{itemize}
\section{Proof for the lower-bound result}\label{sec:lb}
\begin{theorem}[Theorem \ref{thm:lb}]\label{thm:lb1}
	For any algorithm, there exists an MDP such that the algorithm fails to find an $\left(\frac{\epsilon}{2(1-\gamma)}\right)$-optimal policy under the $\epsilon$-contamination model with a probability of at least $1/4$.
\end{theorem}
\begin{proof}[\bf Proof of Theorem \ref{thm:lb1}]
	Consider two MDPs $M_1,M_2$, both with 3 states and 2 actions, defined as
	\begin{eqnarray}
		P_1(s_2|s_1,a_1) = \frac{1-\epsilon}{2}, P_1(s_3|s_1,a_1) = \frac{1+\epsilon}{2}, P_1(s_3|s_1,a_2) = P_1(s_3|s_1,a_2) = \frac{1}{2}\\
		P_2(s_2|s_1,a_1) = \frac{1+\epsilon}{2}, P_2(s_3|s_1,a_1) = \frac{1-\epsilon}{2}, P_2(s_3|s_1,a_2) = P_2(s_3|s_1,a_2) = \frac{1}{2}
	\end{eqnarray}
	and for both MDPs $s_2,s_3$ are absorbing states with constant reward $1$ and $0$, respectively. So for $M_1$, the optimal policy is $\pi_1^*(s_1)=a_2$, and for $M_2$, the optimal policy is $\pi_2^*(s_1)=a_1$.
	In both cases, choosing the alternative action in $s_1$ will incur a suboptimality gap of $\frac{\epsilon}{2(1-\gamma)}$.
	
	
	Let $N(\cdot)$ be the probability function of Bernoulli distribution on $\{s_2, s_3\}$:
	$
	N(x) = 
	\begin{cases}
		1 & \mbox{if $x = s_2$}\\
		0 & \mbox{if $x = s_3$}
	\end{cases}
	$.
	First of all, notice that an $2\epsilon$-\textit{oblivious adversary} can make the two MDPs $M_1, M_2$ indistinguishable by changing $P_1( \cdot \mid s_1, a_1)$ to be $(1-\frac{2\epsilon}{1+\epsilon}) P_1( \cdot \mid s_1, a_1) + \frac{2\epsilon}{1+\epsilon} N(\cdot)$, which is exactly $P_2(\cdot \mid s_1,a_1)$. Note that $\frac{2\epsilon}{1+\epsilon}\leq 2\epsilon$ and thus can be achieved by a $2\epsilon$-oblivious adversary.

	When the two MDPs are indistinguishable, any rollout has the same probability under both MDP, and thus conditioned on any roll-out, the learner can at best obtain an $\frac{\epsilon}{2(1-\gamma)}$-optimal policy with probability $1/2$ on both MDP.
	
	What remains to be shown is that with high probability, the $\epsilon$-contamination adversary can simulate the oblivious adversary.

	Let $X_i, Y_i$ be Bernoulli random variables s.t. 
	$X_i = 
	\begin{cases}
		s_2 & U \le \frac{1-\epsilon}{2} \\
		s_3 & \mbox{o.w.}
	\end{cases}
	$,
	$Y_i = 
	\begin{cases}
		s_2 & U \le \frac{1+\epsilon}{2} \\
		s_3 & \mbox{o.w.}
	\end{cases}
	$,
	where $U$ is picked uniformly random in $[0,1]$.
	Then $(X_i, Y_i)$ is a coupling with law:
	$
	P((X_i, Y_i) = (s_2, s_2)) = \frac{1-\epsilon}{2}
	$,
	$
	P((X_i, Y_i) = (s_2, s_3)) = 0
	$,
	$
	P((X_i, Y_i) = (s_3, s_2)) = \epsilon
	$,
	$
	P((X_i, Y_i) = (s_3, s_3)) = \frac{1-\epsilon}{2}
	$,
	$X_i$ and $Y_i$ can be thought as the outcome of $P_1( \cdot \mid s_1, a_1)$, $P_2(\cdot \mid s_1,a_1)$ respectively. 
	The $\epsilon$-contamination adversary can simulate the oblivious adversary by changing $X_i$ to $Y_i$ when $X_1 \neq Y_i$, which has probability $\epsilon$.
	This is possible when there are at most $\epsilon$ fraction of index $i$ s.t. $X_i \neq Y_i$. 
	Suppose there are $T$ episodes, then
	\begin{eqnarray}
		P\left(\sum_{i = 1}^T \1_{\{\mbox{$a_1$ is taken at $s_1$}\}}\1_{\{X_i \neq Y_i\}} \ge \epsilon T\right) 
		\le P(\sum_{i = 1}^T \1_{\{X_i \neq Y_i\}} \geq T\epsilon)\leq \frac{1}{2}
	\end{eqnarray} 
	because the median of Binomial$(n,p)$ is at most $\ceil{np}$. Therefore, the probability that the adaptive adversary can simulate the oblivious adversary throughout 
	$T$ episodes is at least $1/2$. Assuming that when the adversary fails to simulate, the learner automatically succeed in finding the optimal policy, then we've established that the learner will still fail to find an $\left(\frac{\epsilon}{2(1-\gamma)}\right)$-optimal policy with probability $1/4$ on both MDPs.
\end{proof}

\section{Property of $\hat Q(s,a)$ sampled from Algorithm \ref{alg:sampler_est}}\label{sec:prop}
To prepare for the analysis that follows, we first show that the $\hat Q(s,a)$ sampled from Algorithm \ref{alg:sampler_est} is unbiased and has bounded variance.
\begin{lemma}\label{lemma:var}
	$\E{}{\hat Q^{\pi}(s,a)} = Q^\pi(s,a)$, $\mbox{Var}(\hat Q^{\pi}(s,a)) \le \frac{\gamma}{(1-\gamma)^2} + \frac{\sigma^2}{1-\gamma}$. The bound for variance is tight.
\end{lemma}
\begin{proof}[\bf Proof of Lemma \ref{lemma:var}]
	In the following, we treat $(s_0, a_0)$ as deterministic.
	\begin{align*}
		\E{}{\hat Q^{\pi}(s_0,a_0)} = & \sum_{k=0}^{\infty} \E{}{\sum_{t = 0}^{T} r(s_t, a_t)\middle\vert T = k} P(T = k)\quad 
		\mbox{(by law of total expectation)} \\
		= & \sum_{k=0}^{\infty} \E{}{\sum_{t = 0}^{k} r(s_t, a_t)} (1-\gamma)\gamma^k \quad 
		\mbox{(each $r(s,a)$ is independent of $T$)} \\
		= & (1-\gamma) \sum_{k = 0}^\infty \frac{\gamma^k}{1 - \gamma}\E{}{r(a_k, s_k)} \\
		= & Q^{\pi}(s_0,a_0)
	\end{align*}
	Now, we upperbound the variance. Let $\bar r(s,a) := r(s,a) - e(s,a) $ be the expected reward over the zero-mean noise. Because the zero-mean noise is independent of state transition, we observe that:
	\begin{align*}
		\E{}{r(s,a)} = & \E{}{\bar r(s,a)}\\
		\E{}{r(s,a)^2} = & \E{}{(\bar r(s,a) + e(s,a))^2} = \E{}{\bar r(s,a)^2} + \E{}{e(s,a)^2} \le \E{}{\bar r(s,a)^2} + \sigma^2 \\
		\E{}{r(s_i, a_i)r(s_j,a_j)} = & \E{}{(\bar r(s_i,a_i) + e(s_i,a_i))(\bar r(s_j,a_j) + e(s_j,a_j))} = \E{}{\bar r(s_i,a_i) \bar r(s_j, a_j)},
	\end{align*} 
	for $i \neq j$. 
	
	Given the above observations, we can bound the variance as follows
	\begin{eqnarray*}
		&&\mbox{Var}(\hat Q^{\pi}(s_0,a_0)) \\
		&\le & \sigma^2 + \E{}{(\hat Q^{\pi}(s_0,a_0)- \bar r(s_0, a_0))^2} - \left(\E{}{\hat Q^{\pi}(s_0,a_0)}- \bar r(s_0, a_0)\right)^2 \quad \mbox{(separate the variance of $r(s_0,a_0)$)}\\
		&= & \sigma^2 + \sum_{k = 1}^\infty (1-\gamma)\gamma^k \E{}{\left(\sum_{t = 1}^{k} r(s_t, a_t)\right)^2}  - \left(\E{}{\hat Q^{\pi}(s_0,a_0)}- \bar r(s_0, a_0)\right)^2 \\
		&=& \sigma^2 + \sum_{k = 1}^\infty (1-\gamma)\gamma^k\left(\sum_{t=1}^k \E{}{r(s_t, a_t)^2} + 2\sum_{i=1}^k\sum_{j=i+1}^k \E{}{r(s_i,a_i)r(s_j,a_j)}\right) 
		- \left(\E{}{\hat Q^{\pi}(s_0,a_0)}- \bar r(s_0, a_0)\right)^2 \\
		&= & \sigma^2 + \sum_{t = 1}^\infty \gamma^t \E{}{r(s_t,a_t)^2} + 2\sum_{i=1}^\infty \sum_{j = i+1}^\infty \gamma^j \E{}{r(s_i,a_i)r(s_j,a_j)} 
		- \left(\E{}{\hat Q^{\pi}(s_0,a_0)}- \bar r(s_0, a_0)\right)^2 \\
		&\le & \frac{\sigma^2}{1-\gamma} + \sum_{t = 1}^\infty \gamma^t \E{}{\bar r(s_t,a_t)^2} + 2\sum_{i=1}^\infty \sum_{j = i+1}^\infty \gamma^j \E{}{\bar r(s_i,a_i)\bar r(s_j,a_j)} 
		- \left(\E{}{\hat Q^{\pi}(s_0,a_0)}- \bar r(s_0, a_0)\right)^2 \\
		&\le & \frac{\sigma^2}{1-\gamma} + \sum_{t = 1}^\infty \gamma^t \E{}{\bar r(s_t,a_t)} + 2\sum_{i=1}^\infty \sum_{j = i+1}^\infty \gamma^j \E{}{\bar r(s_i,a_i)} - \left(\E{}{\hat Q^{\pi}(s_0,a_0)}- \bar r(s_0, a_0)\right)^2 \\
		&= & \frac{\sigma^2}{1-\gamma} + \sum_{t = 1}^\infty \gamma^t \E{}{\bar r(s_t,a_t)} + 2\sum_{i=1}^\infty \frac{\gamma^{i+1}}{1-\gamma} \E{}{\bar r(s_i,a_i)} - \left(\E{}{\hat Q^{\pi}(s_0,a_0)}- \bar r(s_0, a_0)\right)^2 \\
		&= & \frac{\sigma^2}{1-\gamma} + \frac{1+\gamma}{1-\gamma} \sum_{t=1}^\infty \gamma^t\E{}{\bar r(s_t, a_t)}
		- \left(\sum_{t = 1}^\infty \gamma^t \E{}{\bar r(s_t, a_t)}\right)^2 \\
		&= & - \left(\sum_{t = 1}^\infty \gamma^t \E{}{\bar r(s_t, a_t)} - \frac{1+\gamma}{2(1-\gamma)}\right)^2 + \frac{(1+\gamma)^2}{4(1-\gamma)^2} + \frac{\sigma^2}{1-\gamma} \\
		& \le & - \left(\sum_{t = 1}^\infty \gamma^t  - \frac{1+\gamma}{2(1-\gamma)}\right)^2 + \frac{(1+\gamma)^2}{4(1-\gamma)^2} + \frac{\sigma^2}{1-\gamma} = \frac{\gamma}{(1-\gamma)^2} + \frac{\sigma^2}{1-\gamma} 
	\end{eqnarray*}
	The last line is because:
	\[
	\sum_{t = 1}^\infty \gamma^t \E{}{\bar r(s_t, a_t)}
	\le \sum_{t = 1}^\infty \gamma^t = \frac{\gamma}{1-\gamma}
	\le \frac{1+\gamma}{2(1-\gamma)}.
	\]
	The equality can be reached by the following reward setting: let $P(1 = \bar r(s_1, a_1) = \cdots = \bar r(s_t, a_t) = \cdots) = 1$ and therefore is tight.
\end{proof}

\section{Proofs for Section \ref{sec:npg}.}\label{sec:sec4}
\begin{lemma}[Lemma \ref{thm:robust_ols}]\label{thm:robust_ols1}
	Suppose the adversarial rewards are bounded in $[0,1]$, and in a particular iteration $t$, the adversary contaminates $\epsilon^{(t)}$ fraction of the episodes, then given M episodes, it is guaranteed that with probability at least $1-\delta$,
	\begin{align}
		\EE_{s,a\sim d^{(t)}}&\left[\left(Q^{\pi^{(t)}}(s,a)-\phi(s,a)^\top w^{(t)} \right)^2\right] \leq 4 \left(W^2+WH\right)\left(\epsilon^{(t)} + \sqrt{\frac{8}{M}\log\frac{4d}{\delta}}\right).\nonumber
	\end{align}
	where $H = (\log \delta-\log M)/\log\gamma$ is the effective horizon.
\end{lemma}
\begin{proof}[{\bf Proof of Lemma \ref{thm:robust_ols1}}]
	First of all, observe that since the adversarial reward is bounded in $[0,1]$, with probability $1-\delta$, 
	the $\hat Q(s,a)$ estimates collected in the adversarial episodes are bounded by $H \defeq (\log \delta-\log M)/\log\gamma$.
	
	Conditioned on the above event, consider three loss functions $\hat f$, $f^\dagger$ and $f$, representing the loss w.r.t. clean data, corrupted data and underlying distribution respectively, i.e.
	\begin{eqnarray}
		\hat f &=& \frac{1}{M}\sum_{i = 1}^M (y_i-x_i^\top w)^2\\
		f^\dagger &=& \frac{1}{M}\left[\sum_{i \in C} (y^\dagger_i-x^{\dagger\top}_i w)^2 + \sum_{i \notin C} (y_i-x_i^\top w)^2\right]\\
		f &=& \EE (y_i-x_i^\top w)^2
	\end{eqnarray}
	Then, for all $w$, we can make the following decomposition
	\begin{eqnarray}\label{eq:decomposition}
		||\nabla_w f^\dagger - \nabla_w f||
		\leq ||\nabla_w f^\dagger - \nabla_w \hat f|| + ||\nabla_w \hat f - \nabla_w f||.
	\end{eqnarray}
	We next bound each of the two terms in equation \ref{eq:decomposition}.
	For the first term,
	\begin{eqnarray}
		&&\|\nabla_w f^\dagger - \nabla_w \hat f\|\\
		&=& \left\|\frac{2}{M}\sum_{i \in C} \left[(x^\dagger_ix^{\dagger\top}_i-x_ix_i^\top)w + (y^\dagger_i x^\dagger_i-y_ix_i)\right]\right\|\\
		&\leq& 4\left(W+H\right)\epsilon^{(t)}
	\end{eqnarray}
	where the last step uses the fact that $|C|/M\leq \epsilon^{(t)}$, and $\|x\|\leq 1$, $|y^\dagger|\leq H$ and $\|w\|\leq W$.
	For the second term
	\begin{eqnarray}
		&&||\nabla_w \hat f - \nabla_w f||\\
		&\leq & 2\left\|\left(\EE[xx^\top]-\frac{1}{M}\sum_{i = 1}^Mx_ix_i^\top\right)w - \left(\EE[yx]-\frac{1}{M}\sum_{i = 1}^M y_ix_i\right)\right\|\\
		&\leq & 2\left(\frac{2}{3M}\log\frac{4d}{\delta} + \sqrt{\frac{2}{M}\log\frac{4d}{\delta}}\right)W +2\sqrt{\frac{2}{M}\log\frac{4d}{\delta}}\cdot 2H\label{eq:bern}\\
		&\leq & 4\sqrt{\frac{8}{M}\log\frac{4d}{\delta}} \left(W+H\right) \text{, for }M\geq 2\log\frac{4d}{\delta}.
	\end{eqnarray}
	where in step \eqref{eq:bern} we apply Matrix Bernstein inequality \cite{tropp2015introduction} on the first term  and vector Hoeffding's inequality \cite{jin2019short} on the second term. The constant in Corollary 7 of \cite{jin2019short} is instantiated to be $c = 1$, because boundedness means we always have condition 2 in Lemma 2 of \cite{jin2019short}. This condition is all we need throughout the proof for the vector Hoeffding.
	
	Now, let $M$ be sufficiently large, and instantiate $w$ to be $w^{t}$, i.e. the constrained linear regression solution w.r.t $f^\dagger$, then our result above implies that for any vector $v$ such that $||w+v|| \leq W$, we have $\nabla_w f^\dagger(w^{t})^\top v/||v|| \geq 0$,
	and thus 
	\begin{eqnarray}
		\nabla_w f(w^{t})^\top v/||v|| \geq -4 \left(W+H\right)\left(\epsilon^{(t)} + \sqrt{\frac{8}{M}\log\frac{4d}{\delta}}\right) 
	\end{eqnarray}
	which by Lemma B.8 of \cite{diakonikolas2019sever} implies that
	\begin{eqnarray}
		\epsilon_{stat}^{(t)}\leq 4 \left(W^2+HW\right)\left(\epsilon^{(t)} + \sqrt{\frac{8}{M}\log\frac{4d}{\delta}}\right)\mbox{, w.p. } 1-2\delta.
	\end{eqnarray}
\end{proof}
\begin{theorem}[Theorem \ref{thm:npg_natural}]\label{thm:npg_natural1}
	Under assumptions \ref{ass:linearMDP} (linear Q function) and \ref{assum:conditioning} (reset distribution with small $\kappa$), given a desired optimality gap $\alpha$, there exists a set of hyperparameters agnostic to the contamination level $\epsilon$, such that 
	Algorithm \ref{alg:q_npg_sample} guarantees with a $poly(1/\alpha, 1/(1-\gamma), |\A|,W,\sigma,\kappa)$ sample complexity that under $\epsilon$-contamination with adversarial rewards bounded in $[0,1]$, we have
	\begin{eqnarray}
		\EE\left[V^{*}(\mu_0) - V^{\hat\pi}(\mu_0)\right]
		\leq
		\tilde O\left(\max\left[\alpha, W\sqrt{\frac{|\Acal|\kappa\epsilon}{(1-\gamma)^3}} \mbox{ }\right]\right)\nonumber
	\end{eqnarray}
	where $\hat\pi$ is the uniform mixture of $\pi^{(1)}$ through $\pi^{(T)}$.
\end{theorem}

\begin{proof}[{\bf Proof of Theorem \ref{thm:npg_natural1}}] 
First note that $\epsilon_{stat} = \EE_{s,a\sim d^{(t)}}[\left(\phi(s,a)^\top (w^{(t)}-w^*)\right)^2]\leq 4W^2$, because $\|\phi(s,a)\|\leq 1$ and $\|w^{(t)}\|,\|w^*\|\leq W$.
	As a result, the high probability bound in Lemma \ref{thm:robust_ols} can be ready translate into an expected bound:
	\begin{eqnarray}
		\EE\left[\EE_{s,a\sim d^{(t)}}\left[\left(Q^{\pi^{(t)}}(s,a)-\phi(s,a)^\top w^{(t)} \right)^2\right]\right]\leq 4 \left(W^2+HW\right)\left(\epsilon^{(t)} + \sqrt{\frac{8}{M}\log\frac{4d}{\delta}}\right) + 8\delta W^2
	\end{eqnarray}
	where $\delta$ becomes a free parameter. Plugging this into Lemma \ref{thm:npg_regret}, we get
	\begin{eqnarray*}
		& &\EE\left[\frac{1}{T}\sum_{t=1}^T \{V^{*}(\mu_0) - V^{(t)}(\mu_0) \}\right]\\
		&\leq&
		\frac{W}{1-\gamma}\sqrt{\frac{2 \log |\Acal|}{T}}
		+\frac{1}{T}\sum_{t=1}^T \sqrt{ \frac{4|\Acal| \kappa \epsilon_{stat}^{(t)}}{(1-\gamma)^3}}\\
		&\leq&
		\frac{W}{1-\gamma}\sqrt{\frac{2 \log |\Acal|}{T}}
		+\frac{1}{T}\sum_{t=1}^T \sqrt{ \frac{16|\Acal| \kappa \left(\left(W^2+HW\right)\left(\epsilon^{(t)} + \sqrt{\frac{8}{M}\log\frac{4d}{\delta}}\right) + 2\delta W^2\right)}{(1-\gamma)^3}}\\
		&\leq&
		\frac{W}{1-\gamma}\sqrt{\frac{2 \log |\Acal|}{T}}
		+
		\frac{1}{T}\sum_{t=1}^T \sqrt{ \frac{16|\Acal| \kappa \left(\left(W^2+HW\right)\sqrt{\frac{8}{M}\log\frac{4d}{\delta}} + 2\delta W^2\right)}{(1-\gamma)^3}}
		+
		\frac{1}{T}\sum_{t=1}^T \sqrt{ \frac{16|\Acal| \kappa \left(W^2+HW\right)\epsilon^{(t)}}{(1-\gamma)^3}}\\
		&\leq&
		\frac{W}{1-\gamma}\sqrt{\frac{2 \log |\Acal|}{T}}
		+
		\sqrt{ \frac{16|\Acal| \kappa \left(\left(W^2+HW\right)\sqrt{\frac{8}{M}\log\frac{4d}{\delta}} + 2\delta W^2\right)}{(1-\gamma)^3}}
		+
		\sqrt{ \frac{16|\Acal| \kappa \left(W^2+HW\right)\epsilon}{(1-\gamma)^3}}
	\end{eqnarray*}
	where the last step is by Cauchy Schwarz and the fact that the attacker only has $\epsilon$ budget to distribute, which implies that $\sum_{t=1}^T \epsilon^{(t)} = T\epsilon$.
	Setting 
	\begin{eqnarray}
		T &=& \frac{2W^2\log |\Acal|}{\alpha^2(1-\gamma)^2}\\
		\delta &=& \frac{\alpha^2(1-\gamma)^3}{32W^2|\Acal|\kappa}\\
		M &=& \frac{512|\Acal|^2 W^2 (W+H)^2 \kappa^2}{\alpha^4(1-\gamma)^6}\log \frac{4d}{\delta},
	\end{eqnarray}
	we get
	\begin{eqnarray}
		& &\EE\left[\frac{1}{T}\sum_{t=1}^T \{V^{*}(\mu_0) - V^{(t)}(\mu_0) \}\right]
		\leq 3\alpha + \sqrt{ \frac{16|\Acal| \kappa \left(W^2+HW\right)\epsilon}{(1-\gamma)^3}}.
	\end{eqnarray}
	with sample complexity
	\begin{eqnarray}
		TM = \frac{1024|\Acal|^2\log |\Acal|W^4(W+H)^2\kappa^2}{\alpha^6(1-\gamma)^8}\log \frac{128W^2|\Acal|\kappa d}{\alpha^2(1-\gamma)^3}.
	\end{eqnarray}
\end{proof}
Next, we prove this tighter version of Theorem \ref{thm:npg_natural} in the special case of tabular MDPs.
\begin{corollary}[Corollary \ref{thm:npg_tabular}]\label{thm:npg_tabular1}
	Given a tabular MDP and assumption \ref{assum:conditioning}, given a desired optimality gap $\alpha$, there exists a set of hyperparameters agnostic to the contamination level $\epsilon$, such that 
	Algorithm \ref{alg:q_npg_sample} guarantees with a $poly(1/\alpha, 1/(1-\gamma), |\A|,W,\sigma,\kappa)$ sample complexity that under $\epsilon$-contamination with adversarial rewards bounded in $[0,1]$, we have
	\begin{eqnarray}
		\EE\left[V^{*}(\mu_0) - V^{\hat\pi}(\mu_0)\right]
		\leq
		\tilde O\left(\max\left[\alpha, \sqrt{\frac{|\Acal|\kappa\epsilon}{(1-\gamma)^5}} \mbox{ }\right]\right)
	\end{eqnarray}
	where $\hat\pi$ is the uniform mixture of $\pi^{(1)}$ through $\pi^{(T)}$.
\end{corollary}
The proof follows the exact same structure as the proof of Theorem \ref{thm:npg_natural1}, but with a tighter robustness bound of linear regression.
\begin{lemma}\label{thm:robust_ols2}
	Assume a tabular MDP and the adversarial rewards are bounded in $[0,1]$, and in a particular iteration $t$, the adversary contaminates $\epsilon^{(t)}$ fraction of the episodes, then given M episodes, it is guaranteed that with probability at least $1-\delta$,
	\begin{align}
		\EE_{s,a\sim d^{(t)}}&\left[\left(Q^{\pi^{(t)}}(s,a)-\phi(s,a)^\top w^{(t)} \right)^2\right]\leq H^2\epsilon^{(t)} + 3 \left(W^2+WH\right) \sqrt{\frac{\log 1/\delta}{M}}.
	\end{align}
	where $H = (\log \delta-\log M)/\log\gamma$ is the effective horizon.
\end{lemma}
\begin{proof}[\bf Proof of Lemma \ref{thm:robust_ols2}] The proof is largely  based on Lemma G.1 of \cite{agarwal2020pc}.
We assumed that the constrained linear regression problem is solved using Projected Online Gradient Descent \cite{zinkevich2003online} on the sequence of loss functions $(w^\top \phi_i - \hat Q_i)^2$, i.e.
\begin{eqnarray}
	w_{i+1} = \text{Proj}_{\|w\|\leq W} \left(w_i-\eta_i(w_i^\top \phi_i-\hat Q_i)\phi_i\right)\text{, for all }i\in [M],
\end{eqnarray}
where $\eta_i = W^2/((W+H)\sqrt{N})$ and we set $w^{(t)} = \frac{1}{M}\sum_{i=1}^{M}w_i$.

Using the projected online gradient descent regret guarantee, we have that:
\begin{align}
	\sum_{i\in C} (w_i^\top \phi^\dagger_i - \hat Q^\dagger_i)^2 + \sum_{i\notin C} (w_i^\top \phi_i - \hat Q_i)^2 
	\leq \sum_{i\in C}(w^{\star\top} \phi^\dagger_i -\hat Q^\dagger_i)^2 + \sum_{i\notin C}(w^{\star\top} \phi_i -\hat Q_i)^2 + \underbrace{W(W+H)}_{:=Q}\sqrt{M}.
\end{align}
which implies
\begin{align}
	&\sum_{i\in [M]} (w_i^\top \phi_i - \hat Q_i)^2 - \sum_{i\in [M]}(w^{\star\top} \phi_i -\hat Q_i)^2\\
	\leq& \sum_{i\in C}\left[(w^{\star\top} \phi^\dagger_i -\hat Q^\dagger_i)^2 - (w^{\star\top} \phi_i -\hat Q_i)^2\right] - \sum_{i\in C}\left[(w_i^{\top} \phi^\dagger_i -\hat Q^\dagger_i)^2 - (w_i^{\top} \phi_i -\hat Q_i)^2\right] + Q\sqrt{M}.\label{eq:35}
\end{align}
We now want to show by induction that $w_i^{\top} \phi\in [0,H]$ for any $i$ and $\phi$. $w_0 = 0$ which satisfies $w_0^{\top} \phi\in [0,H]$. Now, assume that $w_i^{\top} \phi\in [0,H]$, we want to show $w_{i+1}^{\top} \phi\in [0,H]$. In a tabular MDP, $\phi$ is an one-hot vector, and thus for $\phi\neq \phi_i$, $w_{i+1}^\top\phi = w_i^\top\phi\in [0,H]$. If $\phi = \phi_i$, then
\begin{equation}
	w_{i+1}^\top\phi = \left(w_i-\eta_i(w_i^\top \phi_i-\hat Q_i)\phi_i\right)^\top\phi_i \leq (1-\eta_i)w_i^\top\phi_i + \eta \hat Q_i\in [0,H]
\end{equation}
because both $w_i\top\phi_i$ (by induction hypothesis) and $\hat Q_i$ (by assumption on bounded attack) are in $[0,H]$. Therefore, we have shown that  $w_i^{\top} \phi\in [0,H]$ for any $i$ and $\phi$. Then, \eqref{eq:35} implies that 
\begin{align}
	\sum_{i\in [M]} (w_i^\top \phi_i - \hat Q_i)^2 
	\leq \sum_{i\in [M]}(w^{\star\top} \phi_i -\hat Q_i)^2 +  2H^2\epsilon^{(t)}M + Q\sqrt{M}.
\end{align}

Denote random variable $z_i = (\theta_i\cdot x_i - y_i)^2 - (\theta^\star\cdot x_i - y_i)^2$. Denote $\EE_{i}$ as the expectation taken over the randomness at step $i$ conditioned on all history $t=1$ to $i-1$. Note that for $\EE_{i}[z_i]$, we have:
\begin{align}
	&\EE_{i} \left[ (\theta_i\cdot x - y)^2 - (\theta^\star\cdot x - y)^2 \right]\\
	& = \EE_{i} \left[ (\theta_i\cdot x - \EE[y|x])^2\right] \\
	& \qquad \qquad - \EE_{i}\left[2(\theta_i\cdot x - \EE[y|x])( \EE[y|x] - y ) - (\theta^\star\cdot x - \EE[y|x])^2 + 2(\theta^\star\cdot x - \EE[y|x])(\EE[y|x] - y) ) \right]\\
	& =  \EE_{i}\left[ (\theta_i\cdot x - \EE[y|x])^2 - ( \theta^\star\cdot x - \EE[y|x])^2 \right],
\end{align} where we use $\EE[\EE[y|x] - y] = 0$.
Also for $|z_i|$, we can show that for $|z_i|$ we have:
\begin{align}
	\left\lvert z_i\right\rvert = \left\lvert (\theta_i\cdot x_i - \theta^\star\cdot x_i)(\theta_i\cdot x_i +\theta^\star\cdot x_i - 2y_i) \right\rvert \leq W( 2W + 2H ) = 2W(W+H).
\end{align}
Note that $z_i$ forms a Martingale difference sequence. Using Azuma-Hoeffding's inequality, we have that with probability at least $1-\delta$:
\begin{align}
	\left\lvert\sum_{i=1}^M  z_i - \sum_{i=1}^M \EE_{i}\left[ (\theta_i \cdot x - \EE[y|x])^2  - (\theta^\star \cdot x - \EE[y|x])^2\right]\right\rvert  \leq 2W(W+H) \sqrt{{\ln(1/\delta)}{M}},
\end{align} which implies that:
\begin{align}
	&\sum_{i=1}^M \EE_{i}\left[ (\theta_i \cdot x - \EE[y|x])^2  - (\theta^\star \cdot x - \EE[y|x])^2\right] \leq \sum_{i=1}^M z_i + 2W(W+H) \sqrt{{\ln(1/\delta)}{M}} \\
	&\leq 2W(W+H) \sqrt{{\ln(1/\delta)}{M}} + 2H^2M\epsilon^{(t)} + Q\sqrt{M}.
\end{align}

Apply Jensen's inequality on the LHS of the above inequality, we have that:
\begin{align}
	\EE\left( \hat{\theta}\cdot x - \EE[y|x]\right)^2 \leq \EE\left(\theta^\star\cdot x - \EE[y|x]\right)^2  + 2H^2\epsilon^{(t)} + (Q+2W(W+H)) \sqrt{\frac{\ln(1/\delta)}{M}}.
\end{align}
\end{proof}

\section{A modified analysis for SEVER}\label{sec:sever}
In this section, we will derive an expected error bound for SEVER~\cite{diakonikolas2019sever} when applied to a linear regression problem.
The high level idea is to use the results of~\cite{diakonikolas2020outlier} to show the existence of a stable set and change the probabilistic argument in~\cite{diakonikolas2019sever} to an expectation argument. We note that the original result in \cite{diakonikolas2019sever} works only with probability $9/10$, and there is no direct way of translating it into either a high-probability argument or an expectation argument.

In the following, we consider a robust linear regression problem. We observe pairs $(X_i, Y_i) \in \R^d \times \R$ for $i \in [n]$, 
where $X_i$'s are drawn i.i.d. from a distribution $D_x$ and $Y_i = w^{*\top}X_i + e_i$ for some unknown $w^* \in \R^d$. $e_i$'s are i.i.d, noise from some distribution $D_{e\mid x}$.
Note that here $e_i$ and $X_i$ may not be independent.
We let $D_{xy}$ be the joint distribution of $(X,Y)$.
Let $f_i(w) = (Y_i - w^\top X_i)^2$.
Given a multiset of observations $\{(X_i,Y_i)\}_{i=1}^n$, our goal is to minimize the objective function
\begin{equation}
\bar f(w) = \EE_{(X,Y)\sim D_{xy}}[(Y-w^\top X)^2]
\end{equation}
on a convex feasible set $\cal{H}$.
Let $r:= \max_{w\in\cal{H}}\|w\|$ be the $\ell_2$-radius of $\cal{H}$.
In the following, we use $\|\cdot \|$ to denote the spectral norm of a matrix and the $2$-norm of a vector.
We use $\Cov$ to denote the covariance matrix of a random vector: $\Cov[X] = \EE\left[(X - \EE X)(X - \EE X)^\top\right]$.
When $S$ is a set, we use $\EE_S$ and $\Cov_S$ to denote the expectation and covariance over the empirical distribution on $S$.
We allow for an $\epsilon$-fraction of the observations to be arbitrary outliers. 
The $\epsilon$-corruption model is defined in more detail in the Appendix A of~\cite{diakonikolas2019sever}. 

Due to our application, we make assumptions on the linear regression model that is slight different from Assumption E.1 in~\cite{diakonikolas2019sever}:
\begin{assumption}
	\label{ass:sever_lr}
	Given the model for linear regression described above, assume the following conditions for $D_{e \mid x}$ and $D_x$:
	\begin{itemize}
		\item $\EE\left[ e \middle\vert X\right] = 0$;
		\item $\EE\left[ e^2 \middle\vert X\right] \le \xi$;
		\item $\EE_{X\sim D_x}[XX^\top] \preceq s^2 I$ for some $s > 0$;
		\item There is a constant $C > 0$, such that for all unit vectors $v$, $\EE_{X\sim D_x}[\langle v,X\rangle^4] \le Cs^4$.
	\end{itemize}
\end{assumption}
In~\cite{diakonikolas2019sever}, the noise term $e$ and $X$ are independent. We weaken the assumption on $e$ and bound its first and second moments conditional on $X$.

\subsection{Stability with subgaussian rate}
We first note that the gradient of $f_i$, $\nabla f_i(w)$ has bounded covariance matrix. 
We will show this by following the proof of Lemma E.3 in~\cite{diakonikolas2019sever}, but make minor changes as we do not assume $e$ and $X$ are independent:
\begin{lemma}[A variant of Lemma E.3 in~\cite{diakonikolas2019sever}]
	\label{lem:sever_lr_cov_gradf}
	Suppose $D_{xy}$ satisfies the conditions of Assumption~\ref{ass:sever_lr}.
	Then for all unit vectors $v \in \R^d$, we have
	\begin{equation}
	v^\top \Cov_{(X_i,Y_i) \sim D_{xy}}[\nabla f_i(w)]v \le 4 s^2 \xi + 4 C s^4 \|w^* - w\|^2.
	\end{equation}
\end{lemma}
\begin{proof}[\bf Proof of Lemma \ref{lem:sever_lr_cov_gradf}]
	We first note that $f_i(w) = (Y_i - w^\top X_i)^2$ and $\nabla f_i(w) = -2 ((w^* - w)^\top X_i + e_i)X_i$.
	By the property of conditional expectation, for any function $g(\cdot), h(\cdot)$, we have 
	$\EE\left[g(X)h(e)\right] = \EE_{X}\left[\EE_{h(e)\mid X}\left[g(X)h(e) \middle \vert X\right]\right] = \EE_{X}\left[g(X)\EE_{h(e)\mid X}\left[h(e) \middle \vert X\right]\right]$.
	Then
	\begin{eqnarray}
		&&\EE\left[\nabla f_i(w) \nabla f_i(w)^\top \right] = 4\EE\left[((w^* - w)^\top X_i + e_i)^2 X_iX_i^\top\right] \\
		&=& 4\EE\left[((w^* - w)^\top X_i)^2 X_iX_i^\top\right] 
		+ 4 \EE\left[ e_i^2 X_iX_i^\top \right]
		+ 4 \EE\left[2(w^* - w)^\top X_i e_i X_iX_i^\top\right]\\
		& = & 4\EE\left[((w^* - w)^\top X_i)^2 X_iX_i^\top\right] 
		+ 4 \EE\left[X_iX_i^\top \EE\left[e_i^2 \middle \vert X_i\right]  \right]
	\end{eqnarray}
	By Assumption~\ref{ass:sever_lr}, for all unit vectors $v \in \R^d$, we have
	\begin{eqnarray}
		v^\top \EE\left[((w^* - w)^\top X_i)^2 X_iX_i^\top\right] v &=& \EE\left[((w^* - w)^\top X_i)^2 (v^\top X_i)^2\right] \\
		&\le& \sqrt{\EE\left[((w^* - w)^\top X_i)^4\right]\EE\left[(v^\top X_i)^4\right]} \\
		&\le& Cs^4 \|w^* - w\|^2
	\end{eqnarray}
	and
	\begin{equation}
		v^\top \EE\left[X_iX_i^\top \EE\left[e_i^2 \middle \vert X_i\right]  \right] v
		\le \xi v^\top \EE\left[X_iX_i^\top \right] v \le s^2\xi
	\end{equation}
	Thus for all unit vectors $v \in \R^d$, we have
	\begin{eqnarray}
		v^\top \Cov_{(X_i,Y_i) \sim D_{xy}}[\nabla f_i(w)]v \le v^\top \EE\left[\nabla f_i(w) \nabla f_i(w)^\top \right]v
		\le 4 s^2 \xi + 4 C s^4 \|w^* - w\|^2.
	\end{eqnarray}
\end{proof}
We then use the following Theorem~\ref{thm:stab_w_subG_rate} to show that the observations $f_1, \ldots, f_n$ satisfies the Assumption~\ref{ass:sever_stab} with high probability:

\begin{theorem}[Theorem 1.4 in~\cite{diakonikolas2020outlier}]
	\label{thm:stab_w_subG_rate}
	Fix any $0 < \tau < 1$. 
	Let $S$ be a multiset of $n$ i.i.d. samples from a distribution on $
	\R^d$ with mean $\mu$ and covariance $\Sigma$.
	Let $\epsilon' = \tilde{C}\left(\log(1/\tau) / n + \epsilon\right) = O(1)$, for some constant $\tilde{C}>0$.
	Then, with probability at least $1-\tau$, there exists a subset $S' \subseteq S$ such that $|S'| \ge (1-\epsilon')n$ 
	and for every $S'' \subseteq S'$ with $|S''| \ge (1-2\epsilon')|S'|$, the following conditions hold: (i) $\|\mu_{S''} - \mu\| \le \sqrt{\|\Sigma\|} \delta$, and (ii) $\|\overline{\Sigma}_{S''} - \|\Sigma\|I\| \le \|\Sigma\|\delta^2/(2\epsilon')$, for $\delta = O\left(\sqrt{(d\log d)/n} + \sqrt{\epsilon} + \sqrt{\log(1/\tau) / n}\right)$.
\end{theorem}
where 
$\mu_{S''} = \frac{1}{|S''|} \sum_{x \in S''}x$ and 
$\overline{\Sigma}_{S''} = \frac{1}{|S''|} \sum_{x\in S''}(x-\mu)(x - \mu)^\top$.

We use a notion of stability similar to that in~\cite{diakonikolas2019sever} but allow the parameter to depend on the confidence level and sample size: 
\begin{assumption}[A variant of Assumption B.1 in~\cite{diakonikolas2019sever}]
	\label{ass:sever_stab}
	Fix $0<\epsilon<1/2$. 
	With probability at least $1-\tau$, 
	there exists an unknown set $I_{good}\subseteq [n]$ with $|I_{good}| \ge (1-\epsilon)n$ of ``good'' functions $\{f_i\}_{i \in I_{good}}$ and parameters $\sigma$, $\alpha(\epsilon, n,\tau), \beta(\epsilon, n,\tau) \in \R_+$ such that for all $w\in\mathcal{H}$:
	\begin{equation}
		\left\| \frac{1}{|I_{good}|} \sum_{i\in I_{good}} \nabla f_i(w) - \nabla \bar f(w)\right\| 
		\le \sigma \alpha(\epsilon, n,\tau)
	\end{equation}
	and
	\begin{equation}
		\label{eq:sever_stab_cov}
		\left\|\frac{1}{|I_{good}|}(\nabla f_i(w) - \nabla \bar f(w))(\nabla f_i(w) - \nabla \bar f(w))^\top\right\| 
		\le \sigma^2 \beta(\epsilon, n,\tau)
	\end{equation}
\end{assumption}
We can then equivalently write Theorem \ref{thm:stab_w_subG_rate} as the following Proposition:
\begin{proposition}
	\label{prop:stab_lr}
	Given a linear regression model $f_i(w) = (Y_i - w^\top X_i)^2$ satisfying Assumption~\ref{ass:sever_lr}, $X_i \sim D_x$, $D_e \sim D_e$,
	with probability at least $1-\tau$, $\{f_i\}_{i\in[n]}$ satisfies Assumption~\ref{ass:sever_stab} with $\sigma = 2s\sqrt{\xi} + 2 \sqrt{C}s^2 \|w^* - w\|$, $\alpha(\epsilon, n, \tau) = O\left(\sqrt{(d\log d)/n} + \sqrt{\epsilon} + \sqrt{\log(1/\tau) / n}\right)$ and $\beta(\epsilon, n, \tau) = \left(\frac{d\log d}{\log(1/\tau) + n\epsilon} + 1\right)$.
\end{proposition}
\begin{proof}[\bf Proof of Proposition \ref{prop:stab_lr}]
	By Theorem~\ref{thm:stab_w_subG_rate} and Lemma~\ref{lem:sever_lr_cov_gradf}, with probability at least $1- \tau$, there exist an unknown set $I_{good} \subseteq [n]$ with $|I_{good}| \ge (1-\epsilon')n$, s.t.
	\begin{eqnarray}
		&&\left\|\frac{1}{|I_{good}|}(\nabla f_i(w) - \nabla \bar f(w))(\nabla f_i(w) - \nabla \bar f(w))^\top\right\| \\
		&\le& \left\|\frac{1}{|I_{good}|}(\nabla f_i(w) - \nabla \bar f(w))(\nabla f_i(w) - \nabla \bar f(w))^\top - \left\|\Cov_{f \in p^*} [\nabla f]\right\|I \right\| + \left\|\Cov_{f \in p^*} [\nabla f]\right\| \\
		&\le & \left(4 s^2 \xi + 4 C s^4 \|w^* - w\|^2\right) O\left(\frac{d\log d}{\log(1/\tau) + n\epsilon} + 1\right)\\
		&\le& \left(2s\sqrt{\xi} + 2 \sqrt{C}s^2 \|w^* - w\|\right)^2O\left(\frac{d\log d}{\log(1/\tau) + n\epsilon} + 1\right) =: \sigma^2 \beta(\epsilon, n ,\tau).
	\end{eqnarray}
	\begin{eqnarray}
		\|\nabla \hat f(w) - \nabla \bar f(w)\| 
		&\le& \sigma O\left(\sqrt{(d\log d)/n} + \sqrt{\epsilon} + \sqrt{\log(1/\tau) / n}\right) =: \sigma \alpha(\epsilon, n, \tau).
	\end{eqnarray}
\end{proof}

\subsection{The expected optimality gap}
In order to prove the expected optimality gap, we first state a slightly modified version of the main theorem in~\cite{diakonikolas2019sever} by specifying the probability of success;
\begin{theorem}[Theorem B.2 in~\cite{diakonikolas2019sever}]
	\label{thm:sever_main_thm}
	Let the corruption level $\epsilon \in [0, c]$, for some small enough $c > 0$.
	Suppose that the functions $f_1, \ldots, f_n, \bar f: \cal{H} \rightarrow \R$ are bounded below, and that Assumption~\ref{ass:sever_stab} is satisfied.
	Then SEVER applied to $f_1, \ldots, f_n$ returns a point $w \in \cal{H}$ that, fix $p \ge \sqrt{\epsilon}$, with probability at least $1-p$, is a $O\left(\sigma\left(\alpha(\epsilon, n,\tau) + \sqrt{\alpha(\epsilon, n, \tau)^2+\beta(\epsilon, n,\tau)}\sqrt{\epsilon/p}\right)\right)$-approximate critical point of $\bar f$, i.e. for all unit vectors $v$ where $w + \lambda v \in \cal{H}$ for arbitrarily small positive $\lambda$, we have that $v\cdot \nabla f(w) \ge -O\left(\sigma\left(\alpha(\epsilon, n,\tau) + \sqrt{\alpha(\epsilon, n, \tau)^2+\beta(\epsilon, n,\tau)}\sqrt{\epsilon/p}\right)\right)$.
\end{theorem}
if $\bar f$ is convex, we have the following optimality gap. Recall $r$ is the radius of the convex set $\mathcal{H}$ where $w^*$ belongs. 
\begin{corollary}[Corollary B.3 in~\cite{diakonikolas2019sever}]
	\label{coro:sever_opt_gap}
	Let the corruption level $\epsilon \in [0, c]$, for some small enough $c > 0$.
	For functions $f_1, \ldots, f_n: \cal{H} \rightarrow \R$, suppose that Assumption~\ref{ass:sever_stab} holds and that $\cal{H}$ is convex.
	Then, fix $p \ge \sqrt{\epsilon}$, with probability at least $1-p$, the output of SEVER satisfies the following:
	if $\bar f$ is convex, the algorithm finds a $w \in \cal{H}$ such that $\bar f(w) - \bar f(w^*) = O\left(r\sigma\left(\alpha(\epsilon, n,\tau) + \sqrt{\alpha(\epsilon, n, \tau)^2+\beta(\epsilon, n,\tau)}\sqrt{\epsilon/p}\right)\right)$
\end{corollary} 
Given Theorem \ref{coro:sever_opt_gap}, we can prove the following expected optimality gap:
\begin{theorem}[expected optimality gap]\label{thm:sever_main}
	Let the corruption level $\epsilon \in [0, c]$, for some small enough $c > 0$.
	Let $\cal{H}$ be a convex set.
	Given $n$ samples from a linear regression model $f(w) = (Y - w^\top X)^2$ satisfying Assumption~\ref{ass:sever_lr}, where $X \sim D_x$, $e \sim D_e$, $Y = w^{*\top}X + e$ for some unknown $w^* \in  \cal{H}$, SEVER will find a $w \in \cal{H}$, such that
	\begin{eqnarray}
		\EE\left[\bar f(w) - \bar f(w^*)\right] = O\left(\left(sr\sqrt{\xi} + s^2r^2\right)\left(\tau + \sqrt{(d\log d)/n} + \sqrt{\epsilon} + \sqrt{\log(1/\tau) / n}\right)\right).
	\end{eqnarray}
	where the expectation above is over both the randomness of SEVER and $(X_i,Y_i)$ pairs.
\end{theorem}
\begin{proof}[\bf Proof of Theorem \ref{thm:sever_main}]
	In the following, we use $\alpha$ and $\beta$ as shorthands of $\alpha(\epsilon, n, \tau)$ and $\beta(\epsilon, n, \tau)$.
	We first show that $\bar f(w) - \bar f(w^*)$ is upper bounded:
	\begin{eqnarray}
		\bar f(w) - \bar f(w^*) &=& \EE_{X,Y}\left[(Y - w^\top X)^2 - (Y - w^{*\top} X)^2\right] \\
		&=& \EE_{X,e}\left[(w^{*} - w)^\top X + e)^2 - e^2\right] \\
		&=& (w^*-w)^\top \EE_X [XX^\top](w^* - w) \le s^2 (w-w^*)^2 \le 4s^2r^2.
	\end{eqnarray}
	For some constant $M > 0$,
	define $x_1 := M r\sigma\left(\alpha/\sqrt{\epsilon} + \sqrt{\alpha^2+\beta}\right)\sqrt{\epsilon}$.
	Let $A_1$ be the event of $\left\{\mbox{Assumption~\ref{ass:sever_stab} holds}\right\}$.
	Let $A_2$ be the event of $\left\{\mbox{SEVER removes less than $(1+1/\sqrt{\epsilon})\epsilon n$ points}\right\}$.
	Let $A_3(p)$ be the event of $\left\{\bar f(w) - \bar f(w^*) > M r\sigma\left(\alpha + \sqrt{\alpha^2+\beta}\sqrt{\epsilon/p}\right)\right\}$.
	Then, $\forall 0\le p < \sqrt{\epsilon}$
	\begin{equation}
		P(A_2, A_3(p) \mid A_1) = 0.
	\end{equation}
	By Corollary~\ref{coro:sever_opt_gap}, $\forall \sqrt{\epsilon} \le p \le 1$
	\begin{equation}
		P(A_2, A_3(p)\mid A_1)\le p.
	\end{equation}
	By Proposition~\ref{prop:stab_lr},
	\begin{equation}
		P(A_1) \ge 1-\tau.
	\end{equation}
	By Lemma~\ref{lemma:6'},
	\begin{equation}
		P(A_2 \mid A_1) \ge 1-\sqrt{\epsilon},
	\end{equation}
	and thus
	\begin{equation}
		1 - P(A_1, A_2) = 1 - P(A_2\mid A_1)P(A_1) \le \tau + \sqrt{\epsilon}.
	\end{equation}
	Then, we have:
	\begin{align}
		&P\left(\bar f(w) - \bar f(w^*) > x_1/\sqrt{p} \mid A_1, A_2 \right) \\
		\le & P\left(A_3(p) \mid A_1, A_2 \right) = P(A_2,A_3(p) \mid A_1) / P(A_2 \mid A_1)\\
		\le &
		\begin{cases}
			0 & 0\le p < \sqrt{\epsilon} \\
			\frac{p}{1-\sqrt{\epsilon}}& \sqrt{\epsilon} \le p \le 1
		\end{cases}.
	\end{align}
	Let $x = x_1/\sqrt{p}$, we have:
	\begin{equation}
		P\left( \bar f(w) - \bar f(w^*) > x \middle \vert A_1,A_2 \right) \le 
		\begin{cases}
			0& x\ge x_1\epsilon^{-1/4}\\
			\frac{1}{1-\sqrt{\epsilon}} \frac{x_1^2}{x^2} & x_1 \le x<x_1\epsilon^{-1/4}\\
			1&0\le x < x_1
		\end{cases}.
	\end{equation}
	By Proposition~\ref{prop:stab_lr} and law of total expectation, we can bound the expected optimality gap by:
	\begin{eqnarray}
		\EE\left[\bar f(w) - \bar f(w^*)\right] & \le & \EE\left[\bar f(w) - \bar f(w^*) \middle \vert A_1, A_2 \right] P(A_1,A_2) + 4s^2r^2 (1-P(A_1,A_2)) \\
		&\le& \int_{0}^{\infty}P \left(\bar f(w) - \bar f(w^*) > x \middle\vert A_1,A_2\right) dx + 4s^2r^2(\tau + \sqrt{\epsilon})\\
		&=& \int_{0}^{x_1} 1 dx  +  \frac{1}{1-\sqrt{\epsilon}}\int_{x_1}^{x_1\epsilon^{-1/4}} \frac{x_1^2}{x^2} dx + 4s^2r^2(\tau + \sqrt{\epsilon})\\
		&\le&2x_1 + 4s^2r^2(\tau + \sqrt{\epsilon})\\
		&=&2M r\sigma\left(\alpha/\sqrt{\epsilon} + \sqrt{\alpha^2+\beta}\right)\sqrt{\epsilon} + 4s^2r^2(\tau + \sqrt{\epsilon}) \\
		&=& O\left(\left(sr\sqrt{\xi} + s^2r^2\right)\left(\tau + \sqrt{(d\log d)/n} + \sqrt{\epsilon} + \sqrt{\log(1/\tau) / n}\right)\right) 
	\end{eqnarray}
	Note that the expectation above is over both the randomness of SEVER and $(X_i,Y_i)$ pairs.
\end{proof}


\subsection{Proof of Theorem~\ref{thm:sever_main_thm}}
\label{sec:stationary-point}
In this proof, we mainly follow the steps in~\cite{diakonikolas2019sever} but use our notion of stability in Assumption~\ref{ass:sever_stab}. 
We also allow the success probability to vary, so that we can give an expected error bound later on.

We first restate the SEVER algorithm in Algorithm~\ref{alg:sever_sec_B} and Algorithm~\ref{alg:filter_sec_B}.
\begin{algorithm}[!t]
	\caption{\sever{}$(f_{1:n}, \sL, \sigma)$}
	\label{alg:sever_sec_B}
	\begin{algorithmic}[1]
		\State {\bfseries Input:} Sample functions $f_1, \ldots, f_n : \dom \to \bR$, bounded below on a closed domain $\dom$, $\gamma$-approximate learner $\sL$, and parameter $\sigma \in \R_+$.
		\State Initialize $S \gets \{1,\ldots,n\}$.
		\Repeat
		\State ${w} \gets \sL(\{f_i\}_{i \in S})$. $\triangleright$ Run approximate learner on points in $S$.
		\State Let $\widehat{\nabla} = \frac{1}{|S|} \sum_{i\in S} \nabla f_i(w)$.
		\State Let $G = [\nabla f_i({w}) - \widehat{\nabla}]_{i \in S}$ be the $|S| \times d$ matrix of centered gradients.
		\State Let $v$ be the top right singular vector of $G$.
		\State Compute the vector $\tau$ of \emph{outlier scores} defined via
		$\tau_i = \left((\nabla f_i({w}) - \widehat{\nabla}) \cdot v\right)^2$.
		\State $S' \gets S$
		\State \label{filter-step} $S \gets \filter(S', \tau, \sigma)$ $\triangleright$ Remove some $i$'s
		with the largest scores $\tau_i$ from $S$; see Algorithm~\ref{alg:filter_sec_B}.
		\Until{$S = S'$.}
		\State Return $w$.
	\end{algorithmic}
\end{algorithm}
\begin{algorithm}[ht]
	\caption{\filter$(S, \tau, \sigma)$}
	\label{alg:filter_sec_B}
	\begin{algorithmic}[1]
		\State {\bfseries Input:} Set $S \subseteq [n]$, vector $\tau$ of outlier scores, and parameter $\sigma \in \R_+$.
		\State If $\frac{1}{|S|}\sum_{i\in S} \tau_i \leq c_0 \cdot \sigma^2$, for some constant $c_0>1$, 
		return $S$ $\triangleright$ We only filter out points if the variance is larger than an appropriately chosen threshold.
		
		\State Draw $T$ from the uniform distribution on $[0,\max_i \tau_i]$.
		\State Return $\{i \in S: \tau_i < T \}$.
	\end{algorithmic}
\end{algorithm}
\noindent Throughout this proof we let $\goodset$ be as in Assumption~\ref{ass:sever_stab}.
We require the following three lemmas.
Roughly speaking, the first states that with high probability, we will not remove too many points throughtout the process,
the second states that on average, we remove more corrupted points than uncorrupted points, and the third states that at termination, and if we have not removed too many points, then we have reached a point at which the empirical gradient is close to the true gradient.
Formally:
\begin{lemma}\label{lem:bad-elts}
	If the samples satisfy Assumption~\ref{ass:sever_stab}, $|S|\geq c_1n$, and the filtering threshold is at least 
	\begin{eqnarray}
		\frac{2(1-\epsilon)\sigma^2}{c_1-2\epsilon}\left(\alpha(\epsilon,n,\tau)^2 + \beta(\epsilon,n,\tau)\right)
	\end{eqnarray} 
	then if $S'$ is the output of $\textsc{Filter}(S, \tau, \sigma)$, 
	we have that
	\begin{equation}
		\EE[|\goodset\cap (S\backslash S')|] \leq \EE[|([n]\backslash\goodset)\cap(S\backslash S')|].
	\end{equation}
\end{lemma}
\begin{lemma}[Revised version of Lemma 6 in \cite{diakonikolas2019sever}]\label{lemma:6'}
	Assume filtering threshold is $4(\alpha(\epsilon, n, \tau)^2+\beta(\epsilon, n,\tau))\sigma^2$, $\epsilon \leq 1/16$, then we have that for any given $p \geq \sqrt{\epsilon}$, with probability at least $1 - p$, $n-|S| \leq (1+ 1/p)\epsilon n$ when the filtering algorithm terminates.
\end{lemma}

\begin{lemma}\label{lem:final-set}
	If the samples satisfy Assumption~\ref{ass:sever_stab}, $\textsc{Filter}(S, \tau, \sigma) = S$, 
	and $n-|S| \leq (1+1/p)\epsilon n$, for $p \ge \sqrt{\epsilon}$, 
	then
	\begin{equation}
	\left\|\nabla \Ef({w}) - \frac{1}{|\goodset|}\sum_{i\in S} \nabla f_i({w})\right\|_2 \leq O\left(\sigma\left(\alpha(\epsilon, n,\tau) + \sqrt{\alpha(\epsilon, n, \tau)^2+\beta(\epsilon, n,\tau)}\sqrt{\epsilon/p}\right)\right)
	\end{equation}
\end{lemma}

Before we prove these lemmata, we show how together they imply Theorem~\ref{thm:sever_main_thm}.

\begin{proof}[{\bf Proof of Theorem~\ref{thm:sever_main_thm} assuming Lemma~\ref{lemma:6'} and Lemma~\ref{lem:final-set}}]
	First, we note that the algorithm must terminate in at most $n$ iterations.
	This is easy to see as each iteration of the main loop except for the last must decrease the size of $S$ by at least $1$.
	
	It thus suffices to prove correctness.
	Note that Lemma~\ref{lemma:6'} says that with probability at least $1-p$, SEVER will not remove too many points, 
	this will allow us to apply Lemma \ref{lem:final-set} to complete the proof, using the fact that $w$ is a critical point of $\frac{1}{|\goodset|}\sum_{i\in S} \nabla f_i({w})$.
\end{proof}

\noindent
Thus it suffices to prove these three lemmata.
\begin{proof}[{\bf Proof of Lemma \ref{lem:bad-elts}}]
	Let $\Sgood = S\cap \goodset$ and $\Sbad =S\backslash \goodset$. We wish to show that the expected number of elements thrown out of $\Sbad$ is at least the expected number thrown out of $\Sgood$.
	We note that our result holds trivially if $\textsc{Filter}(S, \tau, \sigma) = S$.
	Thus, we can assume that $\EE_{i\in S}[\tau_i] \geq \frac{2(1-\epsilon)\sigma^2}{c_1-2\epsilon}\left(\alpha(\epsilon,n,\tau)^2 + \beta(\epsilon,n,\tau)\right)$. 
	
	It is easy to see that the expected number of elements thrown out of $\Sbad$ is proportional to $\sum_{i\in \Sbad}\tau_i$, while the number removed from $\Sgood$ is proportional to $\sum_{i\in \Sgood}\tau_i$ (with the same proportionality).
	Hence, it suffices to show that $\sum_{i\in \Sbad}\tau_i \geq  \sum_{i\in \Sgood}\tau_i$.
	
	We first note that since $\Cov_{i\in \goodset} [ \nabla f_i(w) ] \preceq \sigma^2 I$, we have that
	\begin{align}
		\Cov_{i\in \Sgood} [ v\cdot \nabla f_i(w)] &
		\leq
		\frac{1-\epsilon}{c_1-\epsilon} \Cov_{i \in \goodset} [v \cdot \nabla f_i (w)] \quad \mbox{(since $|\Sgood| \geq \frac{c_1-\epsilon}{1-\epsilon} |\goodset|$)}\\
		&=\frac{1-\epsilon}{c_1-\epsilon} 
		\left(\frac{1}{|\goodset|}\sum_{i \in \goodset}(v \cdot (\nabla f_i (w) - \bar f(w)))^2 - (\bar f(w) - \EE_{i\in \goodset}[v\cdot \nabla f_i(w)])^2\right)\\
		&\leq \frac{(1-\epsilon)\sigma^2}{c_1-\epsilon}\left(\alpha(\epsilon,n,\tau)^2 + \beta(\epsilon,n,\tau)\right) \quad \mbox{(By Assumption~\ref{ass:sever_stab})} ,
	\end{align}
	
	Let $\mugood =\EE_{i\in \Sgood}[v\cdot \nabla f_i(w)]$ and $\mu=\EE_{i\in S}[v\cdot \nabla f_i(w)]$.
	Note that 
	\begin{equation}
		\label{eq:E_sgood bound}
		\EE_{i\in \Sgood } [\tau_i] = \Cov_{i\in \Sgood}[ v\cdot \nabla f_i(w)] + (\mu-\mugood)^2 \leq \frac{(1-\epsilon)\sigma^2}{c_1-\epsilon}\left(\alpha(\epsilon,n,\tau)^2 + \beta(\epsilon,n,\tau)\right) + (\mu-\mugood)^2 \; .
	\end{equation}
	\noindent We now split into two cases.
	
	Firstly, if 
	\begin{equation}
		\label{eq:ass case 1}
		(\mu-\mugood)^2 \geq \frac{\epsilon}{c_1-2\epsilon}\frac{(1-\epsilon)\sigma^2}{c_1-\epsilon}\left(\alpha(\epsilon,n,\tau)^2 + \beta(\epsilon,n,\tau)\right),
	\end{equation}
	we let $\mubad=\EE_{i\in \Sbad}[v\cdot \nabla f_i(w)]$, and note that $| \mu -\mubad | |\Sbad| = |\mu-\mugood||\Sgood|$. We then have that
	\begin{align}
		\EE_{i\in \Sbad} [\tau_i] &= \Cov_{i\in \Sbad}[ v\cdot \nabla f_i(w)] + (\mu-\mubad)^2 \geq (\mu-\mubad)^2 \\
		&= (\mu-\mugood)^2 \left( \frac{|\Sgood|}{|\Sbad|} \right)^2 \\
		&\geq \frac{|\Sgood|}{|\Sbad|}\frac{c_1-\epsilon}{\epsilon}(\mu-\mugood)^2 \quad \mbox{(because $|\Sgood| \ge (c_1-\epsilon)n$ and $|\Sbad| \le \epsilon n$)}\\
		& =  \frac{|\Sgood|}{|\Sbad|} \left( \frac{c_1-2\epsilon}{\epsilon} (\mu-\mugood)^2+(\mu-\mugood)^2\right)\\
		&\geq \frac{|\Sgood|}{|\Sbad|} \left( \frac{(1-\epsilon)\sigma^2}{c_1-\epsilon}\left(\alpha(\epsilon,n,\tau)^2 + \beta(\epsilon,n,\tau)\right)+(\mu-\mugood)^2\right) \quad \mbox{(by~\eqref{eq:ass case 1})}\\
		&\geq \frac{|\Sgood|}{|\Sbad|}\EE_{i\in \Sgood} [\tau_i] \quad \mbox{(by~\eqref{eq:E_sgood bound})}.
	\end{align}
	Hence, $\sum_{i\in \Sbad}\tau_i \geq  \sum_{i\in \Sgood}\tau_i$. 
	
	On the other hand, if $(\mu-\mugood)^2 \leq \frac{\epsilon}{c_1-2\epsilon}\frac{(1-\epsilon)\sigma^2}{c_1-\epsilon}\left(\alpha(\epsilon,n,\tau)^2 + \beta(\epsilon,n,\tau)\right)$, 
	then $\EE_{i\in \Sgood} [\tau_i] \leq \left(1+\frac{\epsilon}{c-2\epsilon}\right)\frac{(1-\epsilon)\sigma^2}{c_1-\epsilon}\left(\alpha(\epsilon,n,\tau)^2 + \beta(\epsilon,n,\tau)\right) \leq \EE_{i\in S} [\tau_i]/2$. 
	Therefore $\sum_{i\in \Sbad}\tau_i \geq  \sum_{i\in \Sgood}\tau_i$ once again.
	This completes our proof.
\end{proof}

\begin{proof}[\bf Proof of Lemma \ref{lemma:6'}]
	Define the event
	\begin{eqnarray}
		A = \{n-|S| \leq (1+ 1/p)\epsilon n\},
	\end{eqnarray}
	and we want to lower-bound $P(A)$. Given that $\epsilon\leq 1/16$, the threshold is $4(\alpha(\epsilon, n, \tau)^2+\beta(\epsilon, n,\tau))\sigma^2$ and $p\geq \sqrt{\epsilon}$, and conditioned on the event $A$, it can be verified that the asusumption of Lemma \ref{lem:bad-elts} is satisfied.
	In particular, simple calculation shows that given $c_1 = 1-(1+1/p)\epsilon$, $\epsilon\leq 1/16$, $p\geq \sqrt{\epsilon}$, we have
	\begin{eqnarray}
		4\sigma^2\geq \frac{2(1-\epsilon)\sigma^2}{c_1-2\epsilon}
	\end{eqnarray}
	
	And Lemma \ref{lem:bad-elts} implies that  $|([n]\backslash\goodset)\cap S| + |\goodset\backslash S|$ is a supermartingale. 
	Since its initial size is at most $\epsilon n$, with probability at least $1-p$, it never exceeds $\epsilon n/p$, 
	and therefore at the end of the algorithm, we must have that $n-|S| \leq \epsilon n + |\goodset\backslash S| \leq (1+1/p)\epsilon n$.
\end{proof}

\noindent
We now prove Lemma \ref{lem:final-set}.
\begin{proof}[{\bf Proof of Lemma \ref{lem:final-set}}]
	
	We note that
	\begin{align}
		& \left\| \sum_{i\in S} (\nabla f_i(w) -\nabla \Ef(w)) \right\|_2\\
		\leq & \left\|\sum_{i\in \goodset} (\nabla f_i(w) -\nabla \Ef(w)) \right\|_2 + \left\|\sum_{i\in (\goodset\backslash S)} (\nabla f_i(w) -\nabla \Ef(w)) \right\|_2 + \left\| \sum_{i\in (S\backslash \goodset)} (\nabla f_i(w) -\nabla \Ef(w)) \right\|_2\\
		\leq & \left\|\sum_{i\in (\goodset\backslash S)} (\nabla f_i(w) -\nabla \Ef(w)) \right\|_2 + \left\|\sum_{i\in (S\backslash \goodset)} (\nabla f_i(w) -\nabla \Ef(w)) \right\|_2 + n\sigma \alpha(\epsilon, n,\tau).
	\end{align} 
	First we analyze
	\begin{equation}
	\left\|\sum_{i\in (\goodset\backslash S)} (\nabla f_i(w) -\nabla \Ef(w)) \right\|_2.
	\end{equation}
	This is the supremum over unit vectors $v$ of
	\begin{equation}
	\sum_{i\in (\goodset\backslash S)} v\cdot (\nabla f_i(w) -\nabla \Ef(w)).
	\end{equation}
	However, we note that
	\begin{equation}
	\sum_{i\in \goodset} (v\cdot (\nabla f_i(w) -\nabla \Ef(w)))^2 \le n \sigma^2 \beta(\epsilon, n,\tau).
	\end{equation}
	Since $|\goodset\backslash S| \le (1+1/p)\epsilon n$, we have by Cauchy-Schwarz that
	\begin{equation}
	\sum_{i\in (\goodset\backslash S)} v\cdot (\nabla f_i(w) -\nabla \Ef(w)) = \sqrt{(n\sigma^2\beta(\epsilon, n,\tau))((1+1/p)\epsilon n)} = n\sigma\sqrt{\beta(\epsilon, n,\tau)(1+1/p)\epsilon},
	\end{equation}
	as desired.
	
	Let
	\begin{equation}
	\Delta := \left\| \sum_{i\in S} (\nabla f_i(w) -\nabla \Ef(w)) \right\|_2 .
	\end{equation}
	
	Because our Filter algorithm terminates with $n-|S| \leq (1+1/p)\epsilon n$,
	and the stopping condition is set as $\|\frac{1}{|S|}\sum_{i\in S} (\nabla f_i(w) -\nabla \hat{f}(w))(\nabla f_i(w) -\nabla \hat{f}(w))^\top\| \le 4(\alpha(\epsilon, n, \tau)^2+\beta(\epsilon, n,\tau))\sigma^2$,
	we note that since for any such $v$ that 
	
	\begin{eqnarray}
		&&\sum_{i\in S} (v\cdot (\nabla f_i(w) -\nabla \Ef(w)))^2 = \sum_{i\in S} (v\cdot (\nabla f_i(w) -\nabla \hat{f}(w)))^2 + |S|(v\cdot (\nabla \hat{f}(w) - \nabla \Ef(w)))^2\\
		&\le& \sum_{i\in S} (v\cdot (\nabla f_i(w) -\nabla \hat{f}(w)))^2 + \Delta^2/|S| \le n 4(\alpha(\epsilon, n, \tau)^2+\beta(\epsilon, n,\tau))\sigma^2 + \Delta^2 / ((1-(1+1/p)\epsilon)n)
	\end{eqnarray}
	and since $|S\backslash \goodset| \le (1+1/p)\epsilon n$, and so we have similarly that
	\begin{eqnarray}
\left\|\sum_{i\in (S\backslash \goodset)} \nabla f_i(w) -\nabla \Ef(w)\right\|_2 
\le 2n \sigma \sqrt{\alpha(\epsilon, n, \tau)^2+\beta(\epsilon, n,\tau)}\sqrt{(1+1/p)\epsilon}+\Delta\sqrt{\frac{(1+1/p)\epsilon}{1-(1+1/p)\epsilon}}.
	\end{eqnarray}
	Combining with the above we have that
	\begin{equation}
	\frac{\Delta}{n} \le \sigma \alpha(\epsilon, n,\tau) + \sigma\sqrt{\beta(\epsilon, n,\tau)(1+1/p)\epsilon} + 2\sigma \sqrt{\alpha(\epsilon, n, \tau)^2+\beta(\epsilon, n,\tau)}\sqrt{(1+1/p)\epsilon}+\frac{\Delta}{n}\sqrt{\frac{(1+1/p)\epsilon}{1-(1+1/p)\epsilon}}, 
	\end{equation}
	Thus
	\begin{equation}
	\frac{\Delta}{n} \le \frac{1}{1-\sqrt{\frac{(1+1/p)\epsilon}{1-(1+1/p)\epsilon}}}
	\left(
	\sigma \alpha(\epsilon, n,\tau) + 6\sigma \sqrt{\alpha(\epsilon, n, \tau)^2+\beta(\epsilon, n,\tau)}\sqrt{\epsilon/p}
	\right)
	\end{equation}
	and therefore, $\frac{\Delta}{n}=O\left(\sigma\left(\alpha(\epsilon, n,\tau) + \sqrt{\alpha(\epsilon, n, \tau)^2+\beta(\epsilon, n,\tau)}\sqrt{\epsilon/p}\right)\right)$ as desired.
\end{proof}


\section{Proofs for Section \ref{sec:fpg}}\label{sec:sec5}
\begin{lemma}[Lemma \ref{thm:sever_result}]\label{thm:sever_result1}
	Suppose the adversarial rewards are unbounded, and in a particular iteration $t$, the adversarial contaminate $\epsilon^{(t)}$ fraction of the episodes, then given M episodes, it is guaranteed that if $\epsilon^{(t)}\leq c$, for some absolute constant c, and any constant $\tau\in[0,1]$, we have 
	\begin{align}
		&\EE \left[\EE_{s,a\sim d^{(t)}}\left[\left(Q^{\pi^{(t)}}(s,a)-\phi(s,a)^\top w^{(t)} \right)^2\right]\right]
		\\
		&\leq
		O\left(\left(W^2+\frac{\sigma W}{1-\gamma}\right)\left(\sqrt{\epsilon^{(t)}}+
		f(d,\tau)M^{-\frac{1}{2}} + \tau\right)\right).\nonumber
	\end{align}
	where $f(d,\tau) = \sqrt{d\log d}+\sqrt{\log(1/\tau)}$.
\end{lemma}
\begin{proof}[\bf Proof of Lemma \ref{thm:sever_result1}]
	The proof of Lemma \ref{thm:sever_result} follows by instantiating Theorem \ref{thm:sever_main} to our specific linear regression problem instance. To specify the constants in Theorem \ref{thm:sever_main}, we make the following observations
	\begin{enumerate}[leftmargin=*]
		\item 

		By Lemma \ref{lemma:var}, we have that $\xi = \frac{1}{(1-\gamma)^2}+\frac{\sigma^2}{1-\gamma}$.
		\item Since $\|X\|\leq 1$, $\EE_{X\sim D_x}\left[XX^\top\right]\leq I$, and thus $s=1$.
		\item $\max_{\|v\|=1}\E{}{(v ^\top X)^4} \le \E{}{\|v\|^4\|X\|^4} \le 1$, thus $C = 1$.
	\end{enumerate}
	Plugging in the above instantiation to Theorem \ref{thm:sever_main} concludes the proof.
\end{proof}

\begin{theorem}[Theorem \ref{thm:fpg}]\label{thm:fpg1}
	Under assumptions \ref{ass:linearMDP} and \ref{assum:conditioning}, given a desired optimality gap $\alpha$, there exists a set of hyperparameters agnostic to the contamination level $\epsilon$, such that 
	Algorithm \ref{alg:q_npg_sample}, using Algorithm \ref{alg:sever} as the linear regression solver, guarantees with a  $poly(1/\alpha, 1/(1-\gamma), |\A|,W,\sigma,\kappa)$ sample complexity that under $\epsilon$-contamination, we have
	\begin{align}
		\EE&\left[V^{*}(\mu_0) - V^{\hat\pi}(\mu_0)\right]\\
		&\qquad\leq
		\tilde O\left(\max\left[\alpha, \sqrt{ \frac{|\Acal| \kappa \left(W^2+\sigma W\right)}{(1-\gamma)^4}}\epsilon^{1/4}\right]\right).\nonumber
	\end{align}
	where $\hat\pi$ is the uniform mixture of $\pi^{(1)}$ through $\pi^{(T)}$.
\end{theorem}

\begin{proof}[{\bf Proof of Theorem \ref{thm:fpg1}}] 
	Denote $z := 2W$ and again $\epsilon_{stat}\leq (2W)^2 = z^2$. Denote $\left(W^2+\frac{\sigma W}{1-\gamma}\right)=b$.
	Notice that Lemma \ref{thm:sever_result} only holds when $\epsilon^{(t)}\leq c$ for some absolute constant $c$, and there are at most $\epsilon T/c$ iterations in which $\epsilon^{(t)} > c$, which incurs at most $\epsilon_{stat}\leq z^2$ error. Given this observation we can 
	now plugging Lemma \ref{thm:sever_result} into Lemma \ref{thm:npg_regret}, and we get
	\begin{eqnarray}
		& &\EE\left[\frac{1}{T}\sum_{t=1}^T \{V^{*}(\mu_0) - V^{(t)}(\mu_0) \}\right]\\
		&\leq&
		\frac{W}{1-\gamma}\sqrt{\frac{2 \log |\Acal|}{T}}
		+\frac{1}{T}\sum_{t=1}^T \sqrt{ \frac{4|\Acal| \kappa \epsilon_{stat}^{(t)}}{(1-\gamma)^3}}\\
		&\leq&
		\frac{W}{1-\gamma}\sqrt{\frac{2 \log |\Acal|}{T}} + \frac{z^2}{c}\epsilon + 
		\frac{1}{T}\sum_{t=1}^T \sqrt{ \frac{4|\Acal| \kappa b \left(\sqrt{\epsilon^{(t)}} +\sqrt{(d\log d)/M} +  \sqrt{\log(1/\tau) / M}+\tau\right)}{(1-\gamma)^3}}\\
		&\leq&
		\frac{W}{1-\gamma}\sqrt{\frac{2 \log |\Acal|}{T}} + \frac{z^2}{c}\epsilon + 
		\sqrt{ \frac{4|\Acal| \kappa b \left(\sqrt{(d\log d)/M} +  \sqrt{\log(1/\tau) / M}+\tau\right)}{(1-\gamma)^3}}
		+ 
		\frac{1}{T}\sum_{t=1}^T \sqrt{ \frac{4|\Acal| \kappa b \sqrt{\epsilon^{(t)}}}{(1-\gamma)^3}}
		\\
		&\leq&
		\frac{W}{1-\gamma}\sqrt{\frac{2 \log |\Acal|}{T}} + \frac{z^2}{c}\epsilon + 
		\sqrt{ \frac{4|\Acal| \kappa b \left(\sqrt{(d\log d)/M} +  \sqrt{\log(1/\tau) / M}+\tau\right)}{(1-\gamma)^3}}
		+ 
		\sqrt{ \frac{4|\Acal| \kappa b}{(1-\gamma)^3}}\epsilon^{1/4}
	\end{eqnarray}
	where the last two steps are by Cauchy Schwarz and the fact that the attacker only has $\epsilon$ budget to distribute, which implies that $\sum_{t=1}^T \epsilon^{(t)} = T\epsilon$.
	Setting 
	\begin{eqnarray}
		T &=& \frac{2W^2\log |\Acal|}{\alpha^2(1-\gamma)^2}\\
		\tau &=& \frac{\alpha^2(1-\gamma)^3}{4|\Acal|b\kappa} \\
		M &=& \frac{16|\Acal|^2b^2\kappa^2}{\alpha^4(1-\gamma)^6}\max\left[d\log d,\log(1/\tau)\right]
	\end{eqnarray}
	we get
	\begin{eqnarray}
		& &\EE\left[\frac{1}{T}\sum_{t=1}^T \{V^{*}(\mu_0) - V^{(t)}(\mu_0) \}\right]
		\leq 
		O\left(\alpha + \sqrt{ \frac{|\Acal| \kappa b}{(1-\gamma)^3}}\epsilon^{1/4}\right).
	\end{eqnarray}
	with sample complexity
	\begin{eqnarray}
		TM = \frac{32W^2|\Acal|^2\log |\Acal|b^2\kappa^2}{\alpha^6(1-\gamma)^8}\max\left[d\log d,\log(1/\tau)\right]. 
	\end{eqnarray}
\end{proof}

\section{Proof of Theorem \ref{thm:pcpg}}
\begin{algorithm}[!t]
	\begin{algorithmic}[1]
		\State \hspace*{-0.1cm}\textbf{Input}
		$\rho^n_{\text{cov}}$, $b^n$, learning rate $\eta$, sample size $M$ for critic fitting, iterations $T$
		\State Define $\Kcal^n = \{s: \forall a\in\Acal, b^n\sa = 0\}$ \label{line:known}
		\State Initialize policy $\pi^0: \Scal\to\Delta(\Acal)$, such that  
		\begin{align*}
			\pi^0(\cdot | s) = \begin{cases}
				\text{Uniform}(\Acal) &  s\in\Kcal^n \\
				\text{Uniform}(\{a\in\Acal: b^n\sa > 0\}) & s\not\in\Kcal^n.
			\end{cases}
		\end{align*}
		\For{$ t = 0 \to T-1$} 
		\State Draw $M$ i.i.d samples $\left\{s_i,a_i, \widehat{Q}^{\pi^t}(s_i,a_i; r+b^n)\right\}_{i=1}^M$ with $s_i,a_i\sim {\rho^n_{\text{cov}}}$ (see Alg~\ref{alg:sampler_est})
		\State \textbf{Critic} fit: Call Algorithm~\ref{alg:sever} to solve for the robust linear regression problem\label{line:learn_critic}
		\begin{align*}
			\theta^t = \argmin_{\|\theta\|\leq W} \sum_{i=1}^M \left( \theta\cdot  \phi(s_i,a_i) - \left(\widehat{Q}^{\pi^t}(s_i,a_i;r + b^n) - b^n(s_i,a_i)\right) \right)^2
		\end{align*}
		\State \textbf{Actor} update 
		\begin{equation}
			\pi^{t+1}(\cdot |s) \propto \pi^t(\cdot |s) \exp\left( \eta \left(b^n(s,\cdot) + \theta^t\cdot \phi(s,\cdot)  \right)    \one\{s\in\Kcal^n\}   \right)
			\label{eq:pg_update}
		\end{equation}
		\EndFor
		\State \Return $\pi^n \defeq \text{Uniform}\{\pi^0,...,\pi^{T-1}\}$.
	\end{algorithmic}
	\caption{Robust NPG Update}
	\label{alg:npg2}
\end{algorithm}
In \texttt{PC-PG}, aside from the robust linear regression step in Algorithm \ref{alg:npg2}, in step 4 of Algorithm \ref{alg:pcpg}, we also needs to robustly estimate the covariance matrix under $\epsilon$-contamination. Luckily, by assumption, $\phi(s,a)$ is bounded, and thus the current empirical mean estimation is already robust to adversarial contamination:
\begin{lemma}[Robust variant of Lemma G.3 in \cite{agarwal2020pc}]\label{thm:cov}
	Given $\nu \in \Delta(S\times A)$ and $K$ $\epsilon$-contaminated samples from $\nu$. 
	Denote $\Sigma = \E{(s,a)\sim\nu}{\phi(s,a)\phi(s,a)^\top}$.
	Then, with probability at least $1-\delta$, we have that under $\epsilon$-corruption
	\begin{eqnarray}
		\max_{\|x\|\leq 1}\left|x^\top\left(\sum_{i=1}^K \phi(s_i,a_i)\phi(s_i,a_i)^\top/K-\Sigma\right)x\right|\leq \sqrt{\frac{8\log(8d/\delta)}{K}}+2\epsilon.
	\end{eqnarray}
\end{lemma}
\begin{proof}
	Without contamination,  Lemma G.3 in \cite{agarwal2020pc} shows that
	\begin{eqnarray}
		\max_{\|x\|\leq 1}\left|x^\top\left(\sum_{i=1}^K \phi(s_i,a_i)\phi(s_i,a_i)^\top/K-\Sigma\right)x\right|\leq \frac{2\log(8d/\delta)}{3K}+\sqrt{\frac{2\log(8d/\delta)}{K}}.
	\end{eqnarray}
	Since both $x$ and $\phi(s,a)$ has norm bounded by $1$, the $\epsilon$ fraction of contaminated samples can only bias the estimate by at most $2\epsilon$, i.e. with $\epsilon$-contamination
	\begin{eqnarray}
		\max_{\|x\|\leq 1}\left|x^\top\left(\sum_{i=1}^K \phi(s_i,a_i)\phi(s_i,a_i)^\top/K-\Sigma\right)x\right|\leq \sqrt{\frac{8\log(8d/\delta)}{K}}+2\epsilon.
	\end{eqnarray}
\end{proof}

\begin{lemma}[Lemma G.4 in \cite{agarwal2020pc}]\label{thm:cov1}
	Denote $\eta(K) =  \sqrt{\frac{8\log(8d/\delta)}{K}}+2\epsilon$. Then, under $\epsilon$-contamination, $\phi(s,a)^\top(\Sigma_{cov}^n)^{-1}\phi(s,a)\leq \beta$ is guaranteed with probability $1-\delta$, if $N\eta(K)\leq \lambda/2$.
\end{lemma}

\begin{lemma}[variant of Lemma C.2 in \cite{agarwal2020pc}]\label{thm:pcpg-regret} Assuming that for all iterations $n$ but $m$ of them, we have $\phi(s,a)^\top(\Sigma_{cov}^n)^{-1}\phi(s,a)\leq \beta$ for $(s,a)\in\mathcal{K}^n$, then
	\begin{eqnarray}
		V^{*}-V^{\hat\pi}\leq \frac{1}{1-\gamma}\left(2W\sqrt{\frac{\log A}{T}}
		+2\sqrt{\beta\lambda W^2}
		+\frac{1}{NT}\sum_{n=0}^{N-1}\sum_{t=0}^{T-1}2\sqrt{\beta N \epsilon^{(n,t)}_{stat}}
		+\frac{2I_N(\lambda)}{\beta N} + 2Hm
		\right)
	\end{eqnarray}
\end{lemma}
\begin{proof}
	The proof follows exactly the proof of Lemma C.2 in \cite{agarwal2020pc}, except that we note that for iterations in which the assumption is not satisfied, the worst-case loss is bounded:
	\begin{eqnarray}
		\frac{1}{T}\sum_{t=0}^{T-1}\left(\EE_{(s,a)\sim\tilde d_{\mathcal M^n}}\left(A^t_{b_n}(s,a)-\hat A^t_{b_n}(s,a)\right)\mathbf{1}\{s\in\mathcal K^n\}\right)\leq 2H
	\end{eqnarray}
\end{proof}
\begin{proof}[\bf Proof of Theorem \ref{thm:pcpg}]
First of all, we need to upper-bound $m$. The condition in Lemma \ref{thm:cov1} is satisfied as long as $2\epsilon^{(n)}\leq \frac{\lambda}{4N}$ and $K \geq \frac{128N^2\log(8\tilde d/\delta)}{\lambda^2}$. Also note that $\sum_{n=0}^{N-1}\epsilon^{(n)}\leq N\epsilon$, and thus $m$ is at most $\frac{8N^2\epsilon}{\lambda}$.

Also, by Lemma \ref{thm:robust_ols1},
	\begin{align}
\varepsilon^{(n,t)}_{stat} \leq 4 \left(W^2+WH\right)\left(\epsilon^{(n,t)} + \sqrt{\frac{8}{M}\log\frac{4d}{\delta}}\right).
\end{align}
Plugging both into Lemma \ref{thm:pcpg-regret}, we get
	\begin{eqnarray}\label{eq:136}
	V^{*}-V^{\hat\pi}\leq \frac{1}{1-\gamma}\left(2W\sqrt{\frac{\log A}{T}}
	+2\sqrt{\beta\lambda W^2}
	+2\sqrt{4 \left(W^2+WH\right)\beta N \left(\epsilon + \sqrt{\frac{8}{M}\log\frac{4d}{\delta}}\right)}
	+\frac{2I_N(\lambda)}{\beta N} + \frac{16HN^2\epsilon}{\lambda}
	\right)
\end{eqnarray}
Let
\begin{eqnarray}
	T &=& \frac{4W^2\log A}{(1-\gamma)^2\alpha^2}, \qquad
	\lambda = 1, \qquad \beta = \frac{\alpha^2(1-\gamma)^2}{4W^2}, \qquad
	N  = \frac{4W^2d\log(N+1)}{\alpha^3(1-\gamma)^3}\\
	M &=& \frac{2d^2\log^2(N+1)(W^2+WH)^2\log(\frac{4d}{\delta})}{\alpha^6(1-\gamma)^6}, \qquad 
	K = 128N^2\log(8\tilde d/\delta)
\end{eqnarray}
Then, \eqref{eq:136} gives
\begin{eqnarray}
	V^{*}-V^{\hat\pi}\leq 4\alpha + \sqrt{\frac{16(W^2WH)d\log(N+1)}{\alpha(1-\gamma)^3}\epsilon}+\frac{256HW^4d^2\log^2(N+1)}{\alpha^6(1-\gamma)^6}\epsilon
\end{eqnarray}
Let $\alpha = \epsilon^{1/7}$, then
\begin{eqnarray}
	V^{*}-V^{\hat\pi}&\leq& 4\epsilon^{1/7} + \sqrt{\frac{16(W^2WH)d\log(N+1)}{(1-\gamma)^3}}\epsilon^{3/7}+\frac{256HW^4d^2\log^2(N+1)}{(1-\gamma)^6}\epsilon^{1/7}\\
	&\leq& \tilde O(d^2\epsilon^{1/7})
\end{eqnarray}
This concludes the proof.
\end{proof}
\begin{figure*}[ht!]
	\begin{subfigure}{0.305\textwidth}
		\centering
		\includegraphics[width=1\columnwidth]{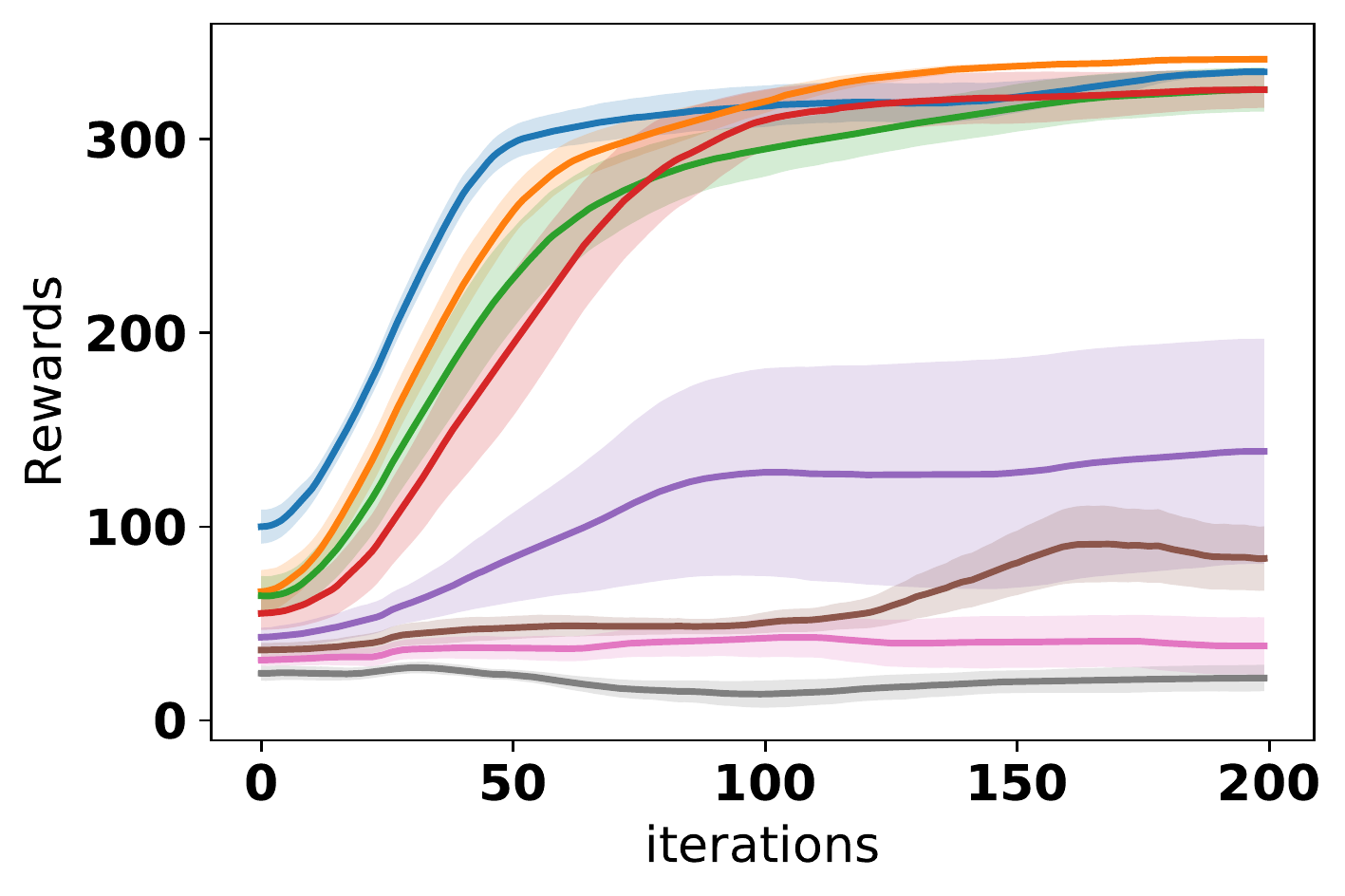}
		\caption{TRPO on Swimmer}
		\label{fig:TRPO_Swimmer}
	\end{subfigure}
	\begin{subfigure}{0.305\textwidth}
		\centering
		\includegraphics[width=1\columnwidth]{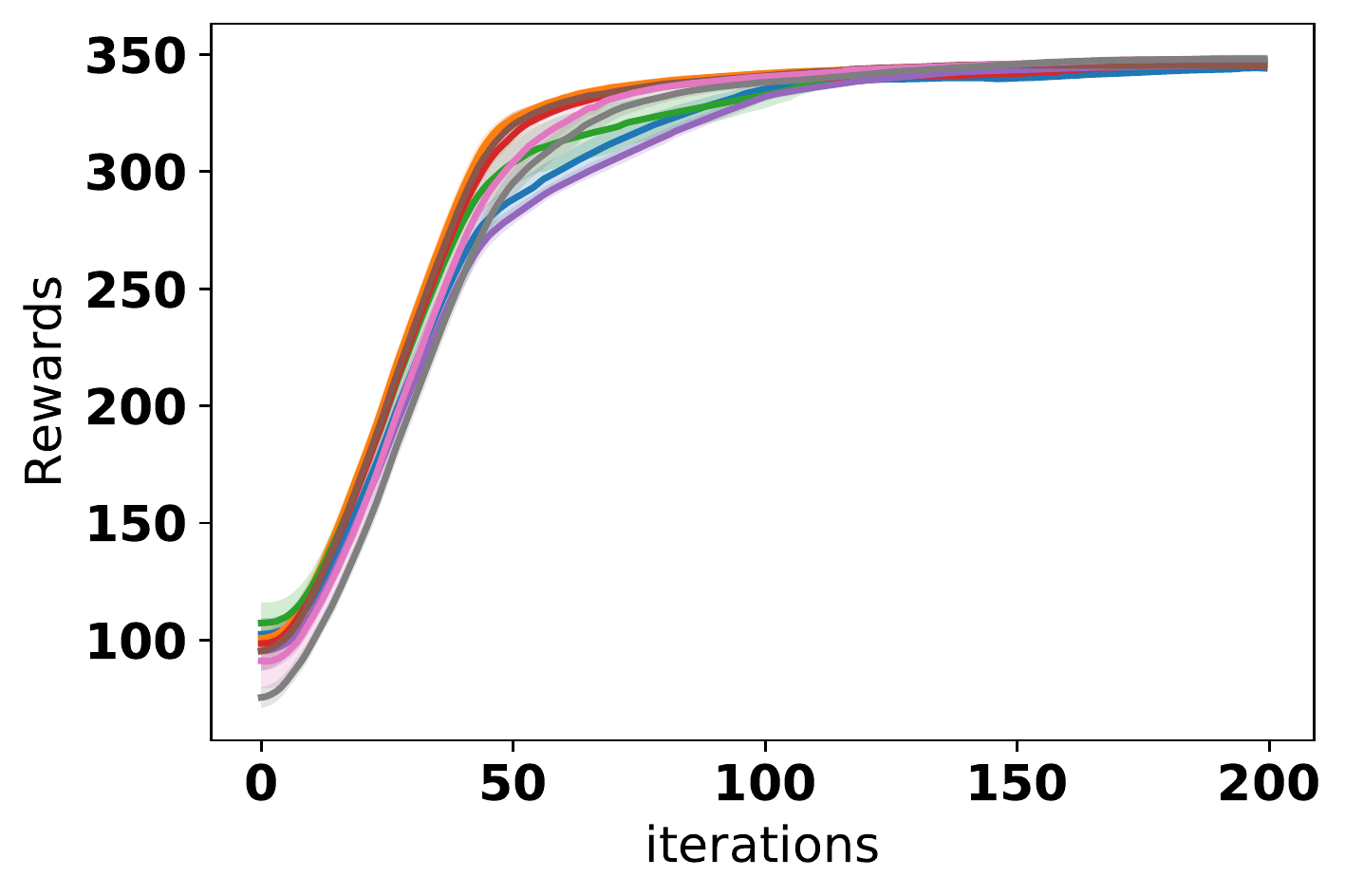}
		\caption{FPG on Swimmer}
		\label{fig:FPG_Swimmer}
	\end{subfigure}
	\begin{subfigure}{0.39\textwidth}
		\centering
		\includegraphics[width=1\columnwidth]{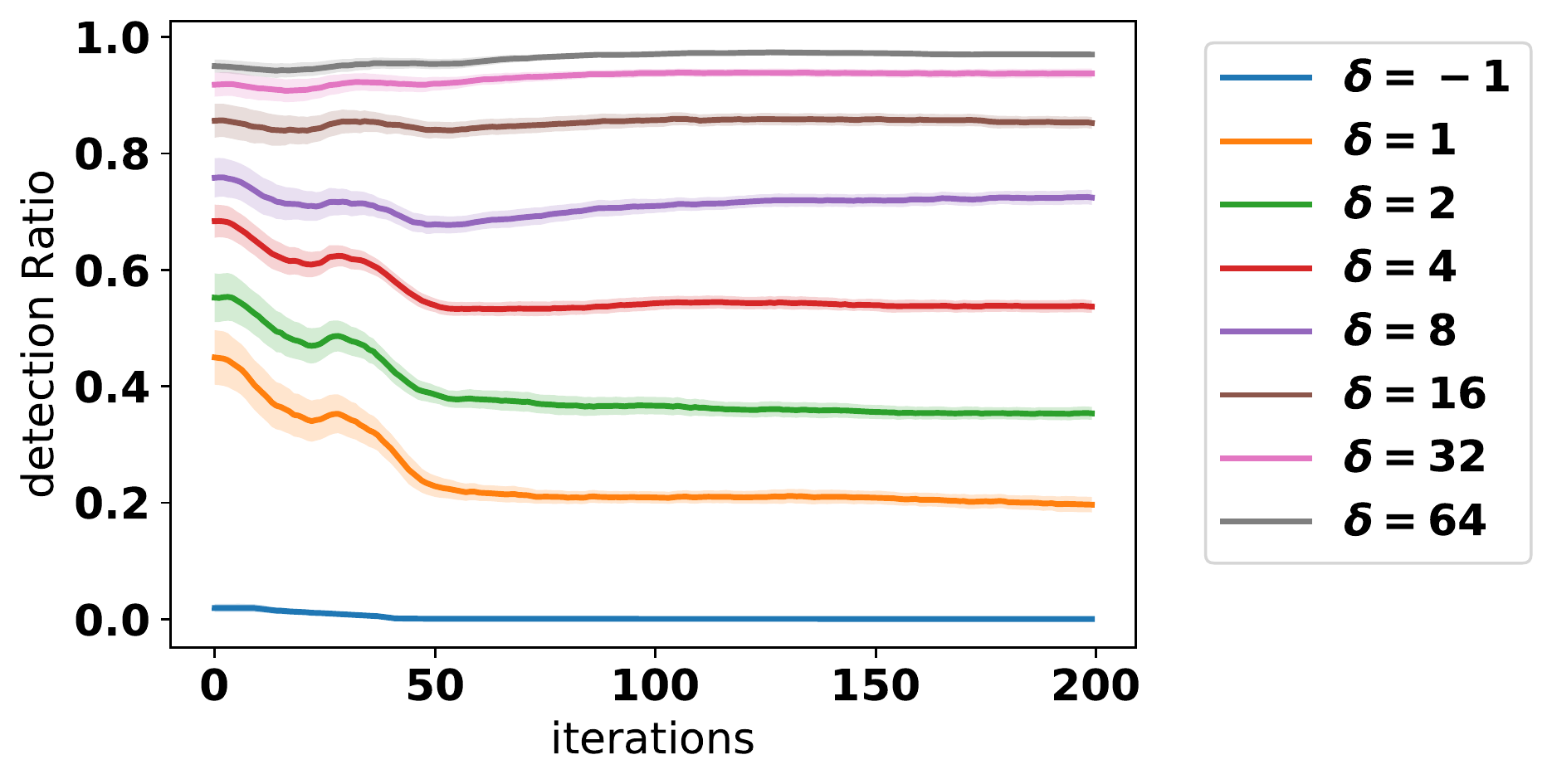}
		\caption{Detection Ratio on Swimmer}
		\label{fig:detection_ratio_Swimmer}
	\end{subfigure}
	\begin{subfigure}{0.305\textwidth}
		\centering
		\includegraphics[width=1\columnwidth]{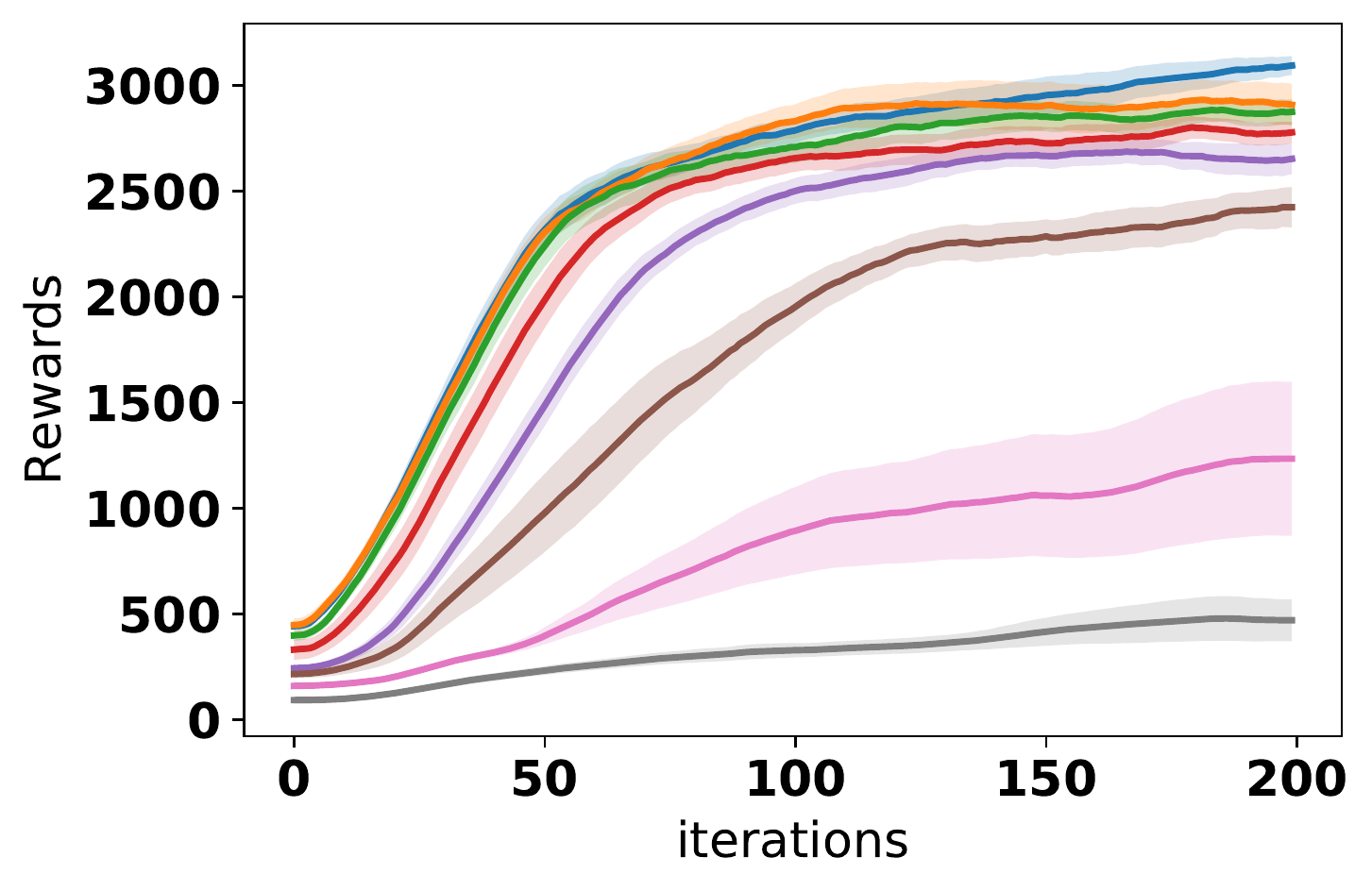}
		\caption{TRPO on Hopper}
		\label{fig:TRPO_Hopper}
	\end{subfigure}
	\begin{subfigure}{0.305\textwidth}
		\centering
		\includegraphics[width=1\columnwidth]{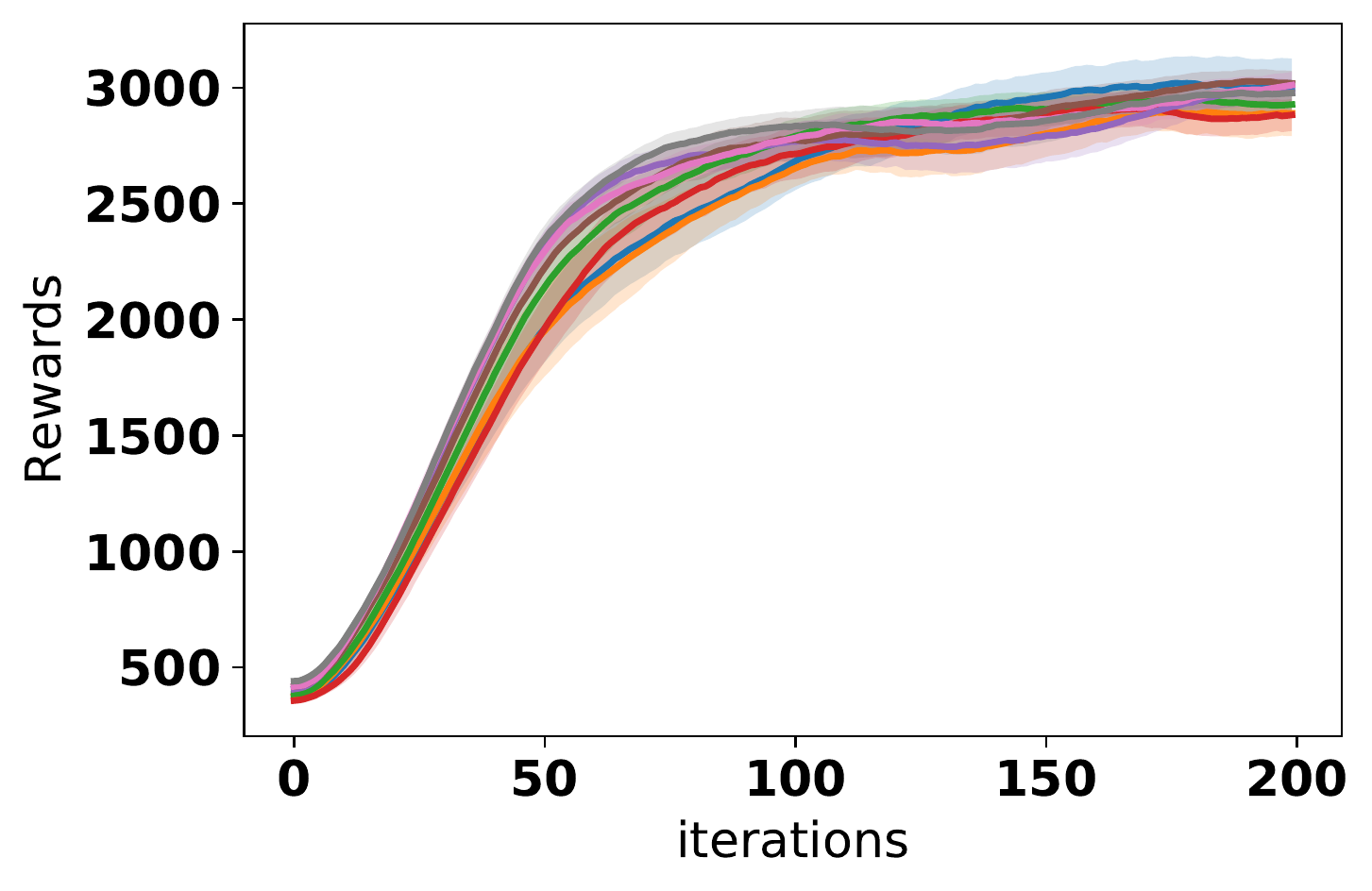}
		\caption{FPG on Hopper}
		\label{fig:FPG_Hopper}
	\end{subfigure}
	\begin{subfigure}{0.39\textwidth}
		\centering
		\includegraphics[width=1\columnwidth]{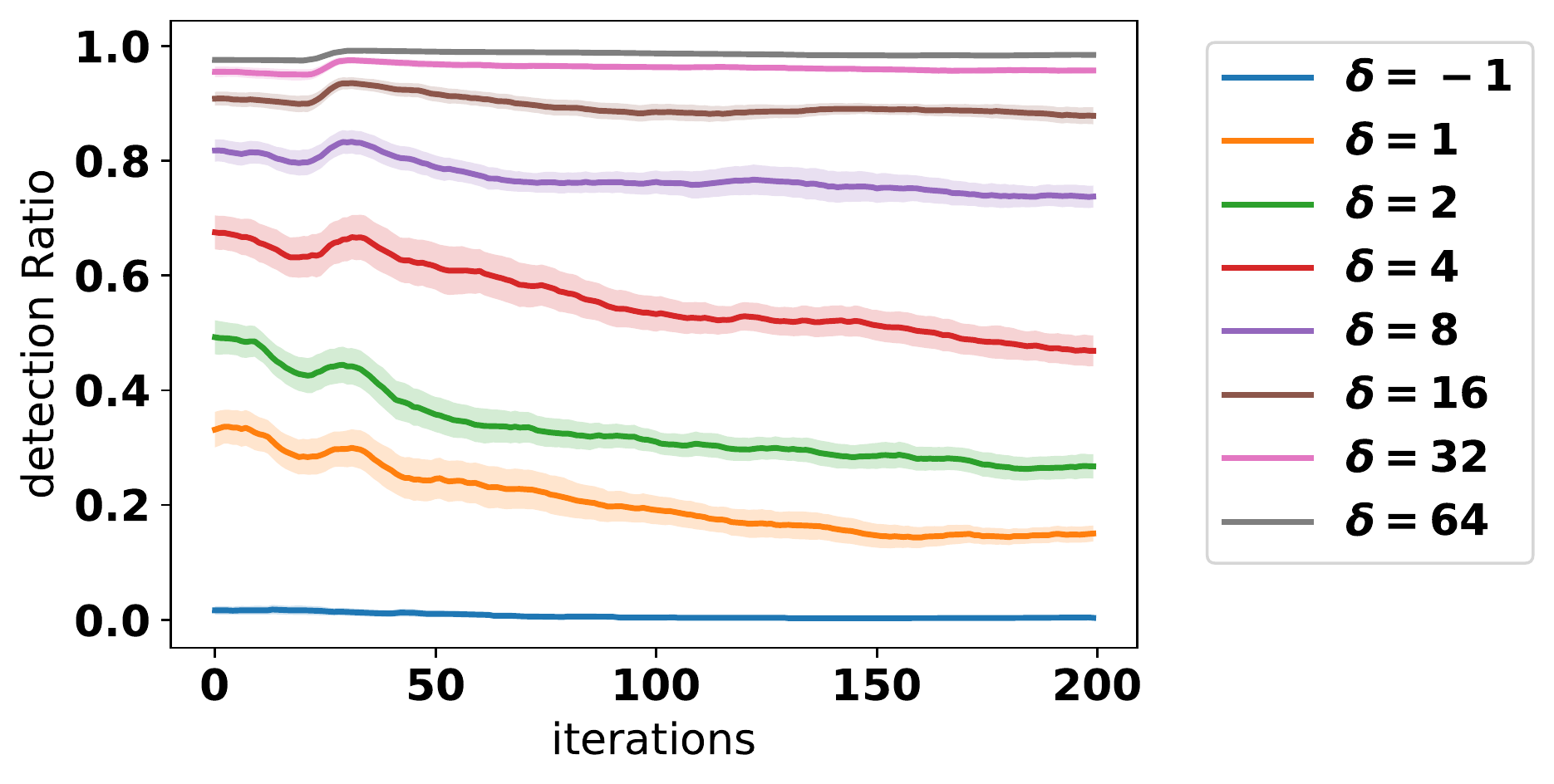}
		\caption{Detection Ratio on Hopper}
		\label{fig:detection_ratio_Hopper}
	\end{subfigure}
	\begin{subfigure}{0.305\textwidth}
		\centering
		\includegraphics[width=1\columnwidth]{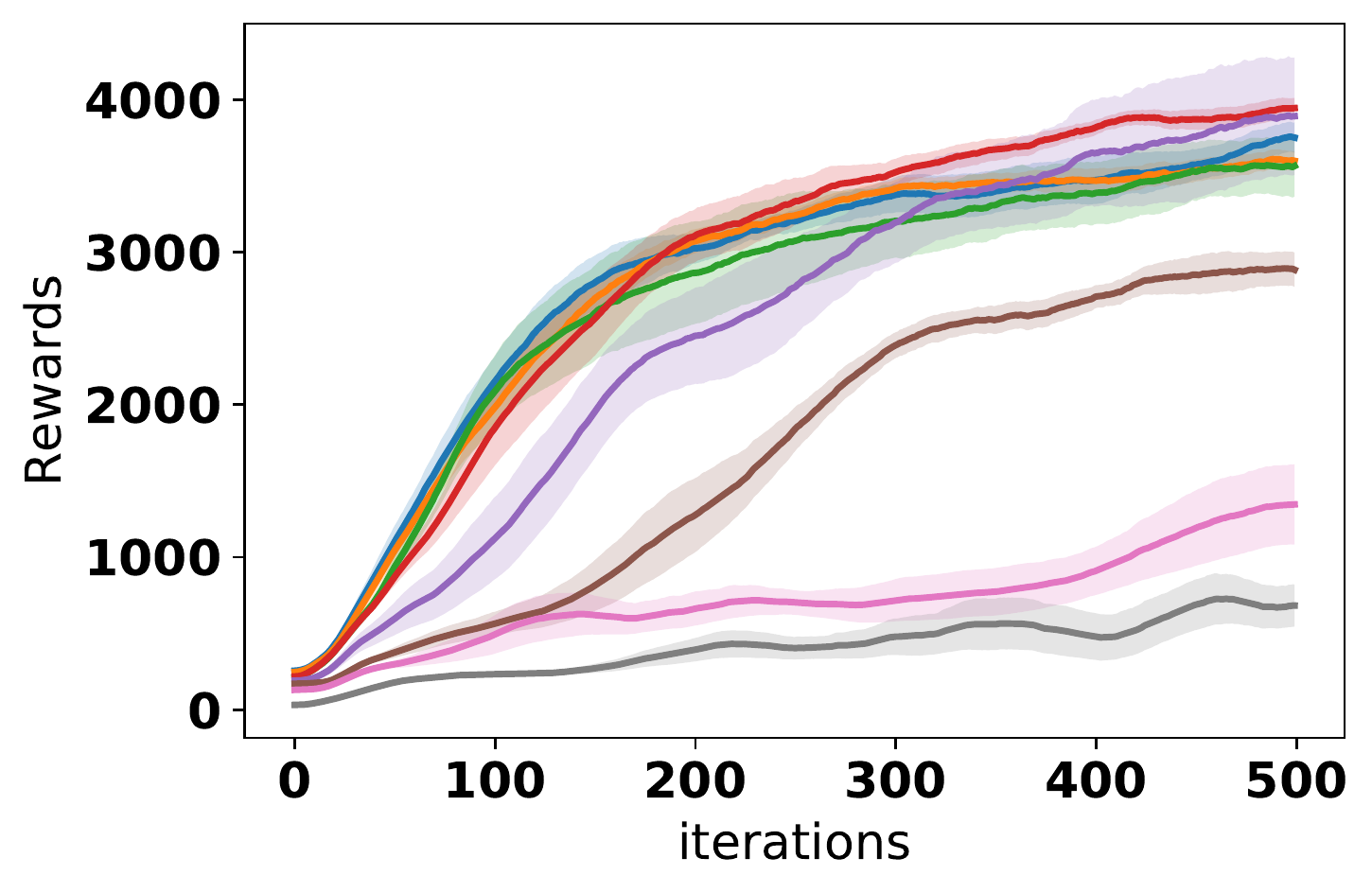}
		\caption{TRPO on Walker2d}
		\label{fig:TRPO_Walker2d}
	\end{subfigure}
	\begin{subfigure}{0.305\textwidth}
		\centering
		\includegraphics[width=1\columnwidth]{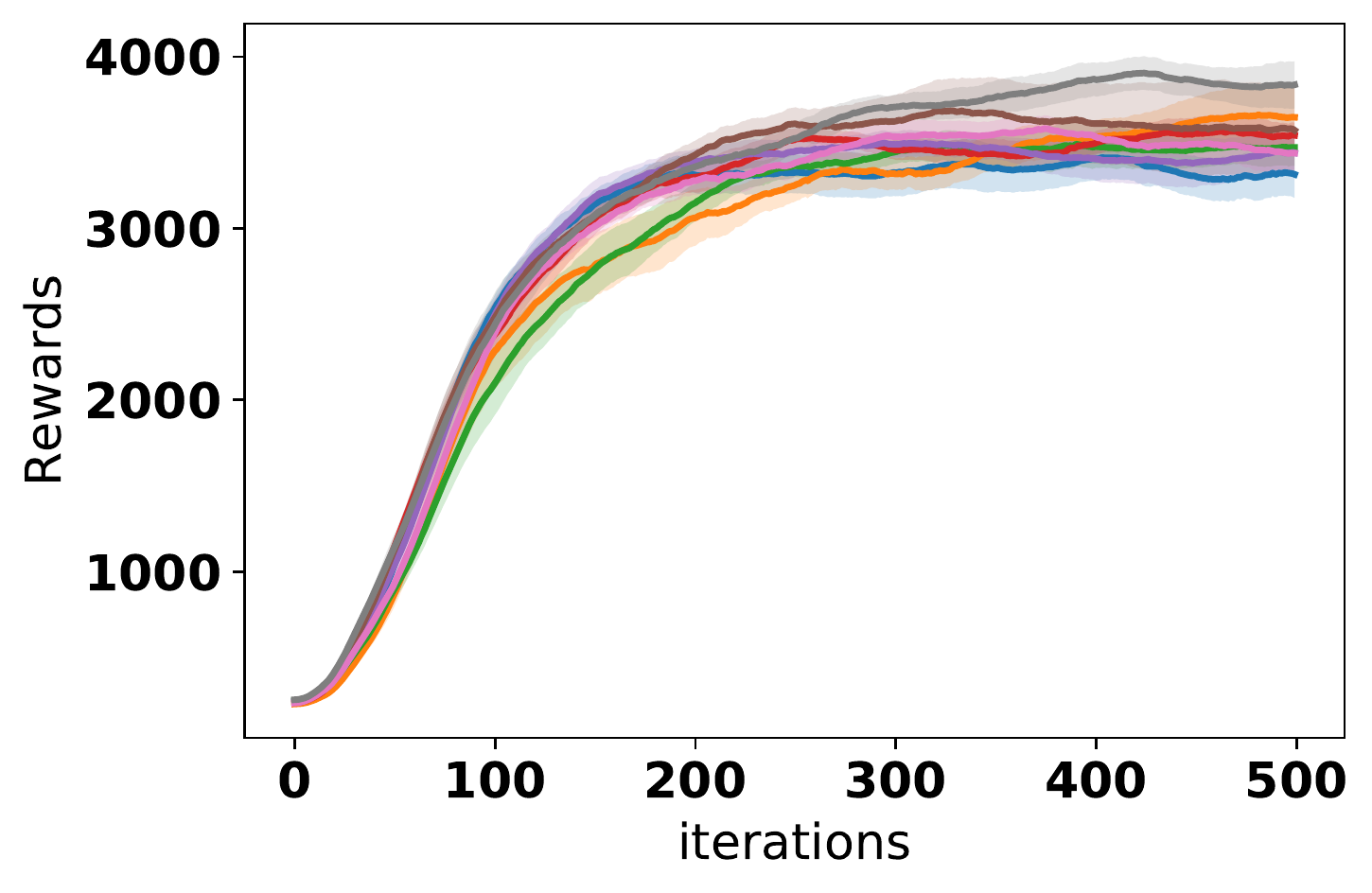}
		\caption{FPG on Walker2d}
		\label{fig:FPG_Walker2d}
	\end{subfigure}
	\begin{subfigure}{0.39\textwidth}
		\centering
		\includegraphics[width=1\columnwidth]{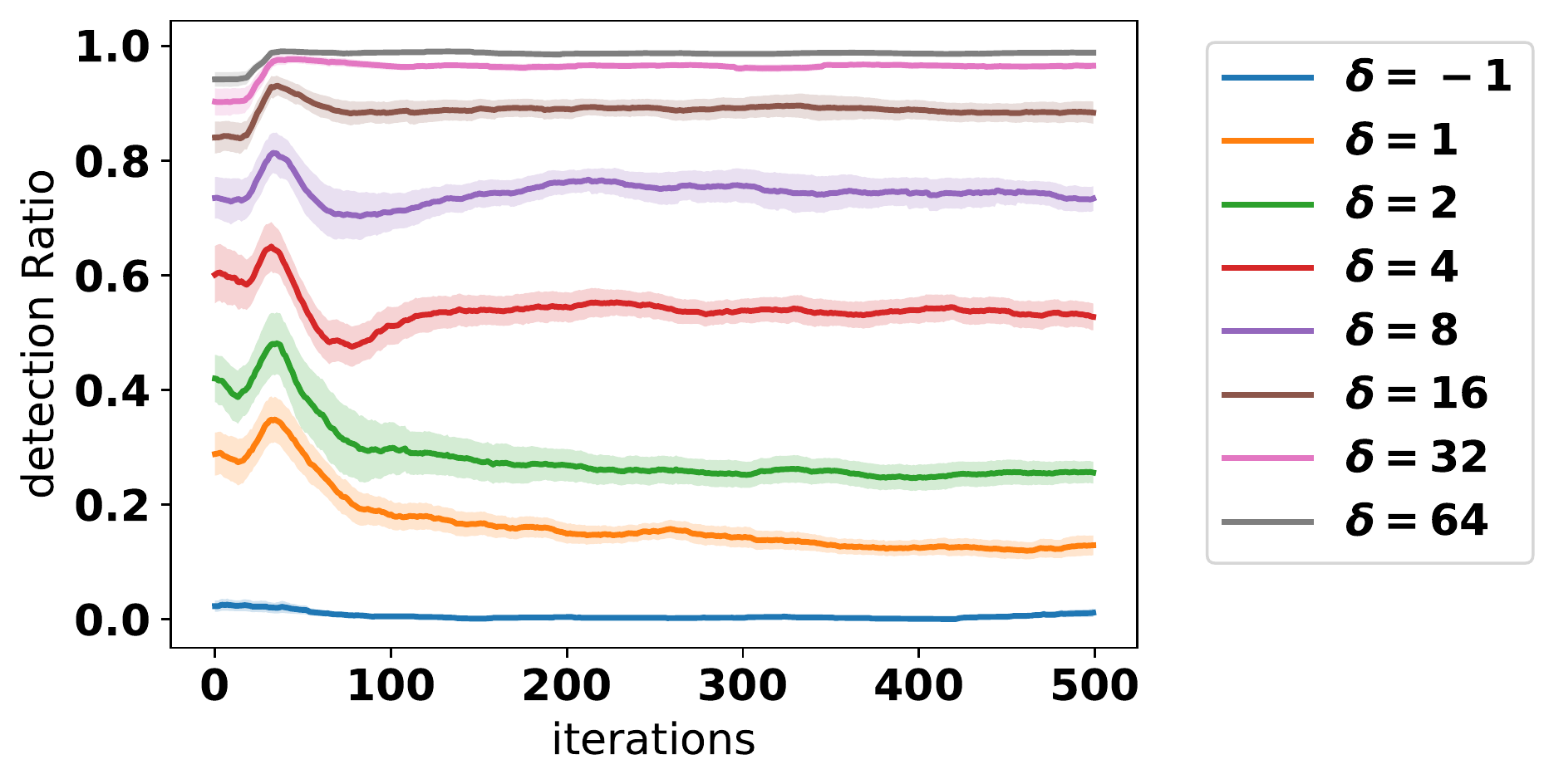}
		\caption{Detection Ratio on Walker2d}
		\label{fig:detection_ratio_Walker2d}
	\end{subfigure}
	\begin{subfigure}{0.305\textwidth}
		\centering
		\includegraphics[width=1\columnwidth]{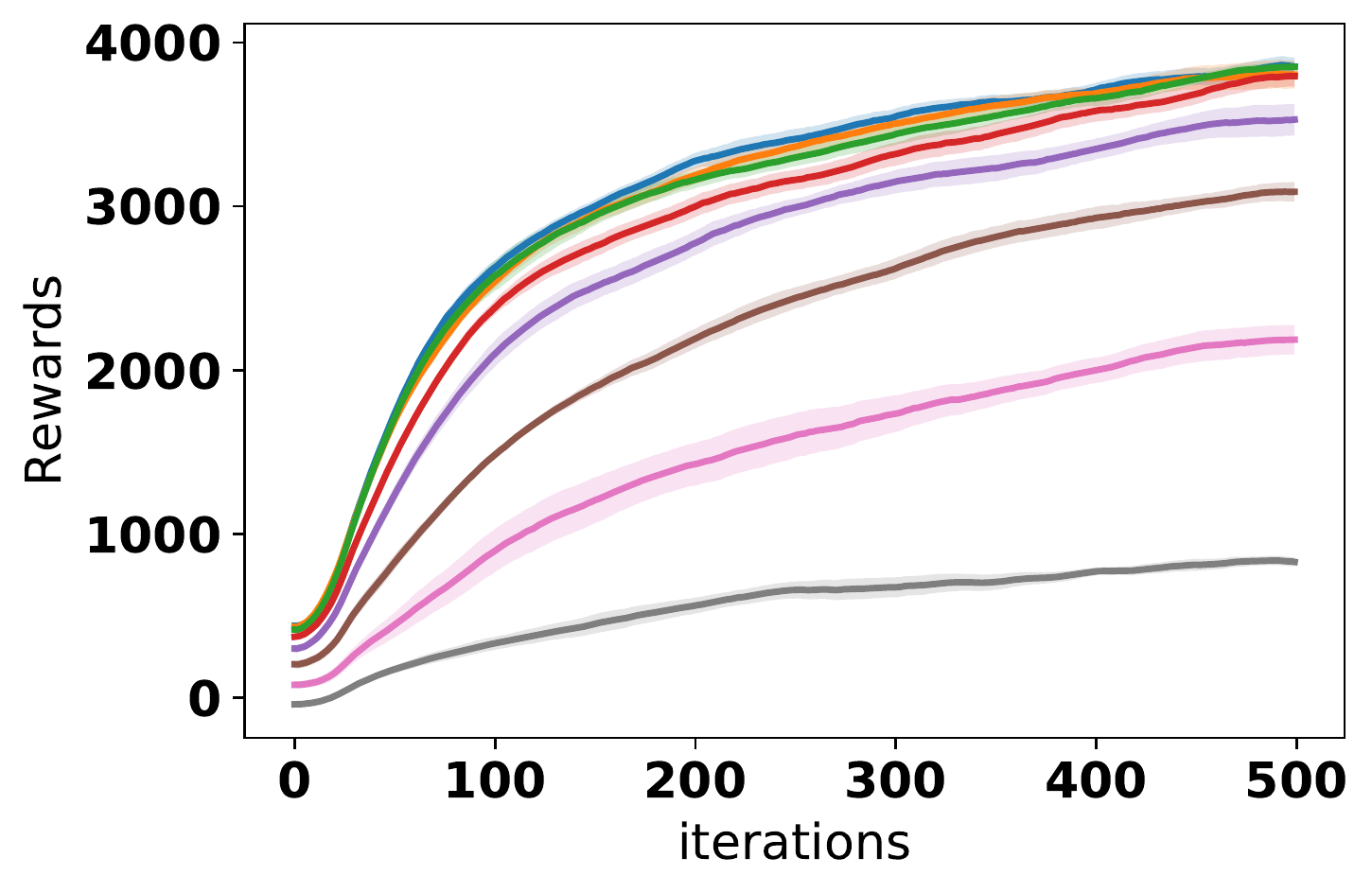}
		\caption{TRPO on HalfCheetah}
		\label{fig:TRPO_HalfCheetah}
	\end{subfigure}
	\begin{subfigure}{0.305\textwidth}
		\centering
		\includegraphics[width=1\columnwidth]{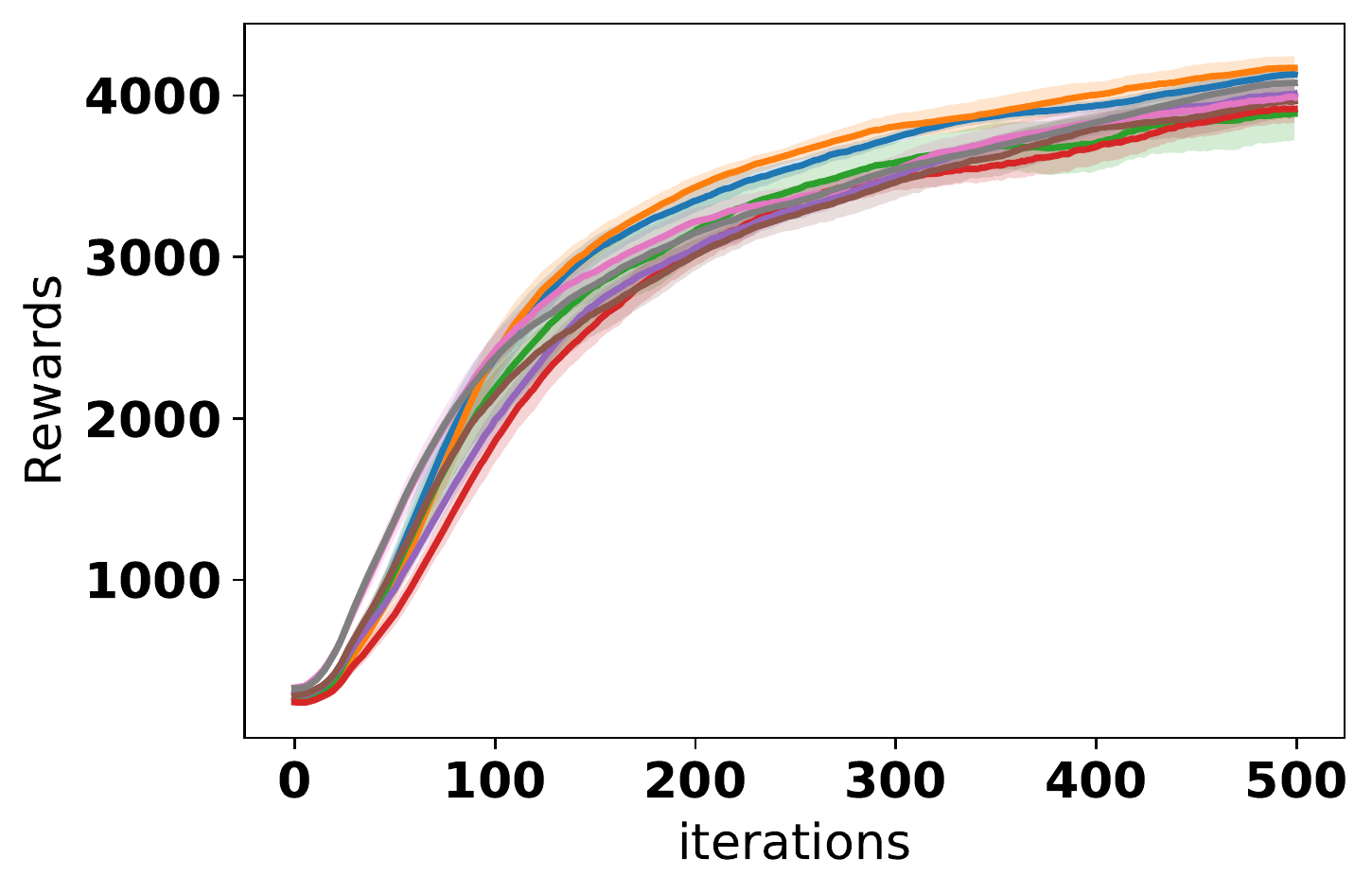}
		\caption{FPG on HalfCheetah}
		\label{fig:FPG_HalfCheetah}
	\end{subfigure}
	\begin{subfigure}{0.39\textwidth}
		\centering
		\includegraphics[width=1\columnwidth]{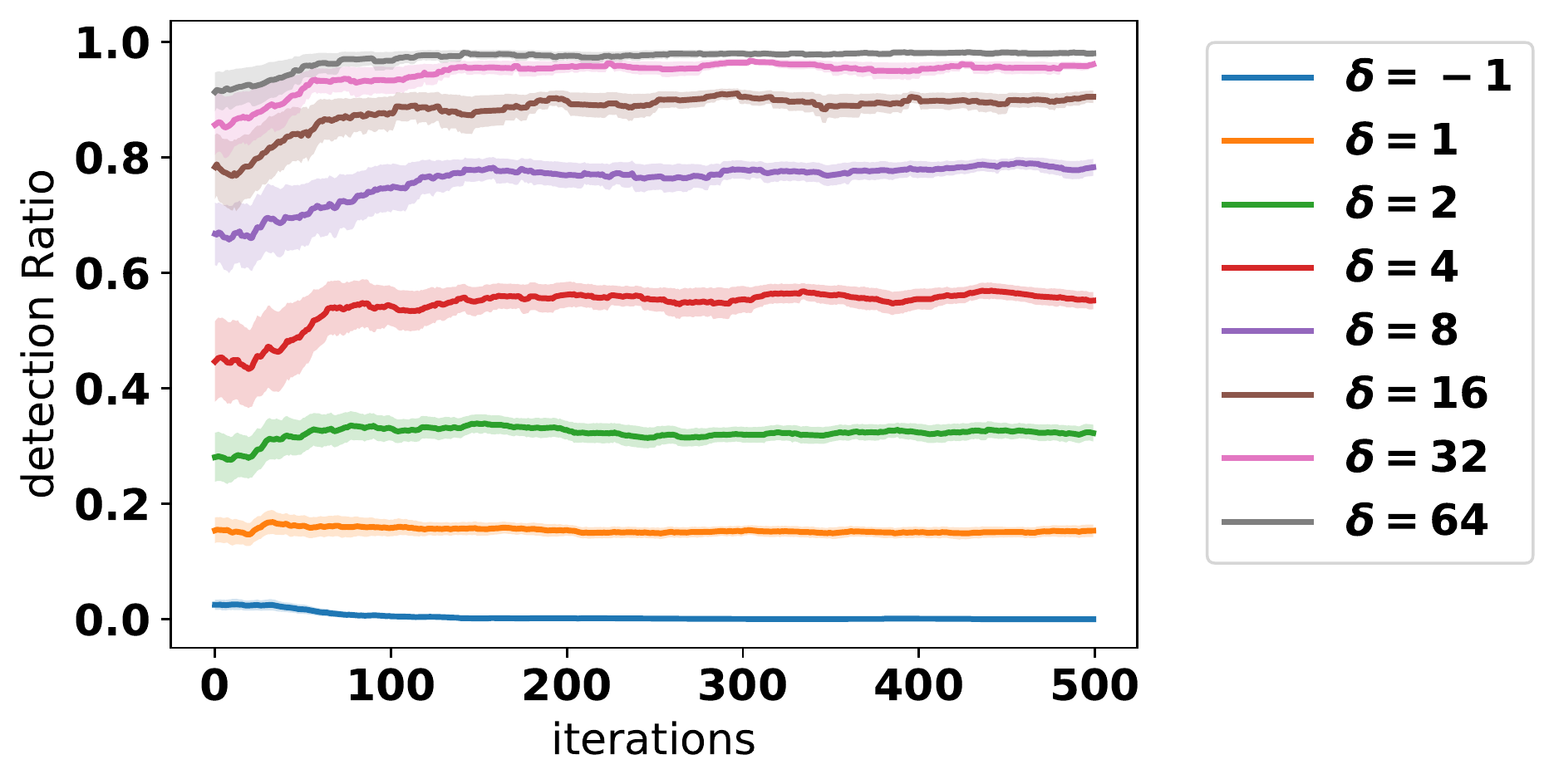}
		\caption{Detection Ratio on HalfCheetah}
		\label{fig:detection_ratio_HalfCheetah}
	\end{subfigure}
	\begin{subfigure}{0.305\textwidth}
		\centering
		\includegraphics[width=1\columnwidth]{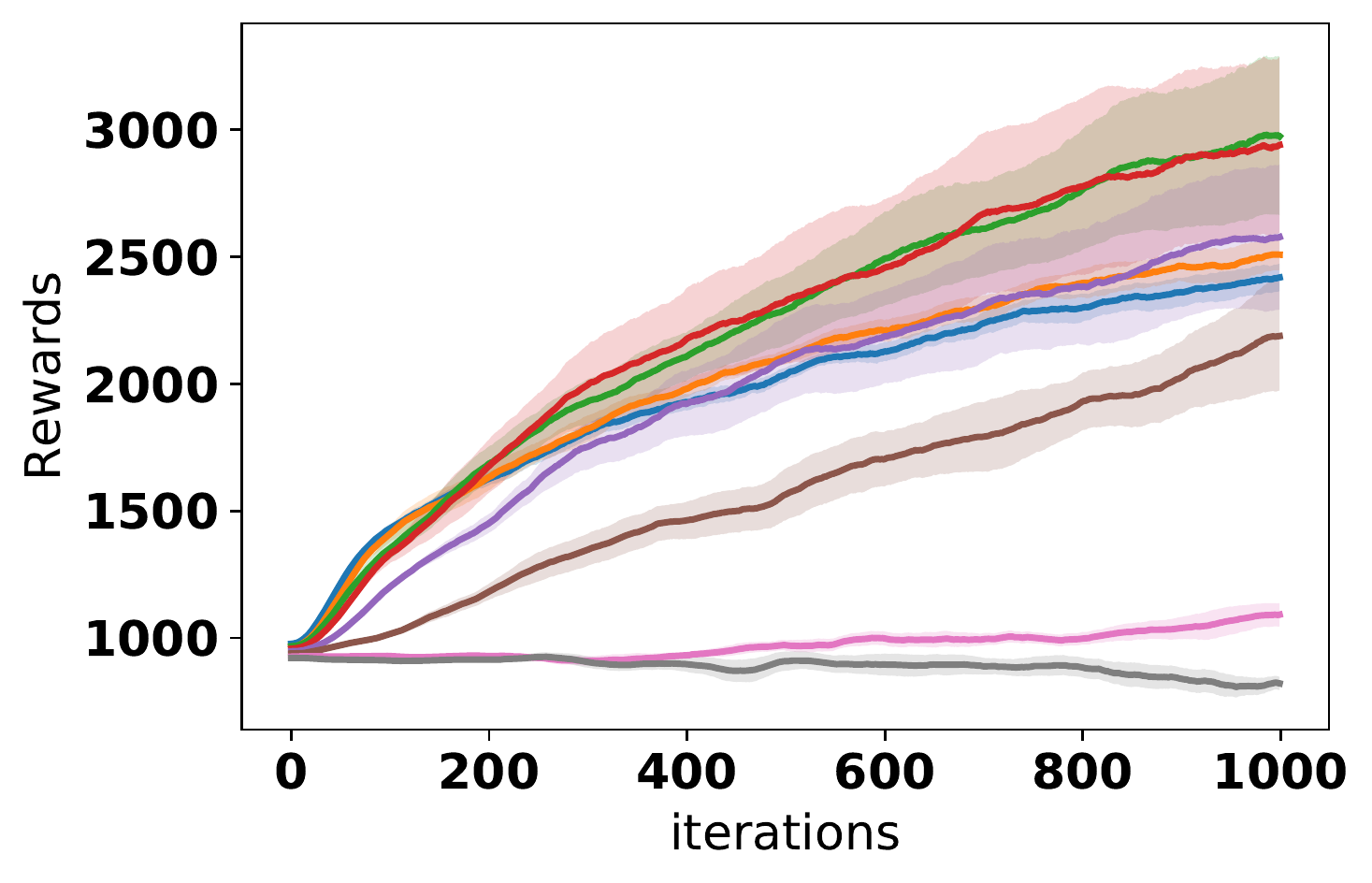}
		\caption{TRPO on Ant}
		\label{fig:TRPO_Ant}
	\end{subfigure}
	\begin{subfigure}{0.305\textwidth}
		\centering
		\includegraphics[width=1\columnwidth]{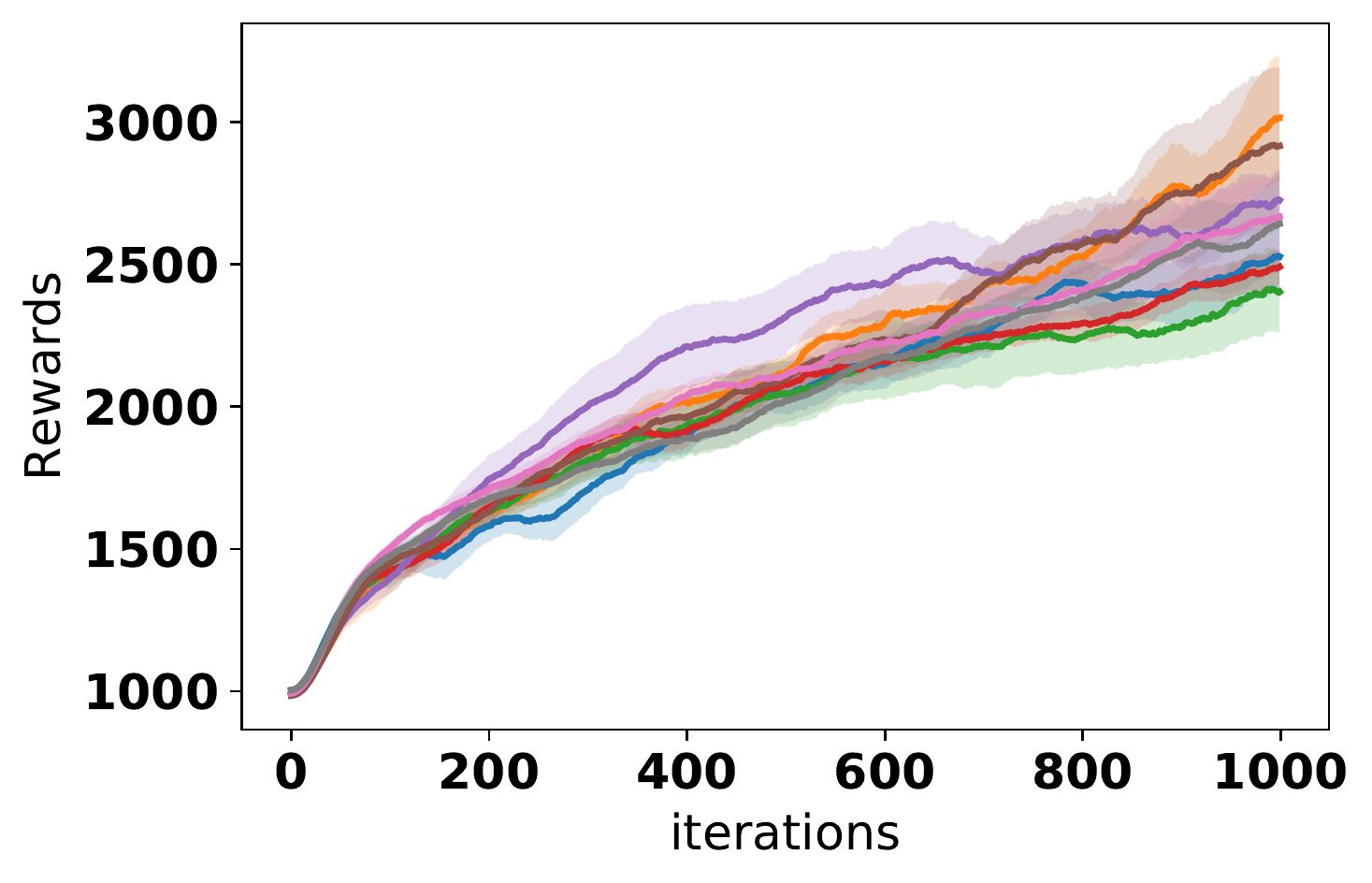}
		\caption{FPG on Ant}
		\label{fig:FPG_Ant}
	\end{subfigure}
	\begin{subfigure}{0.39\textwidth}
		\centering
		\includegraphics[width=1\columnwidth]{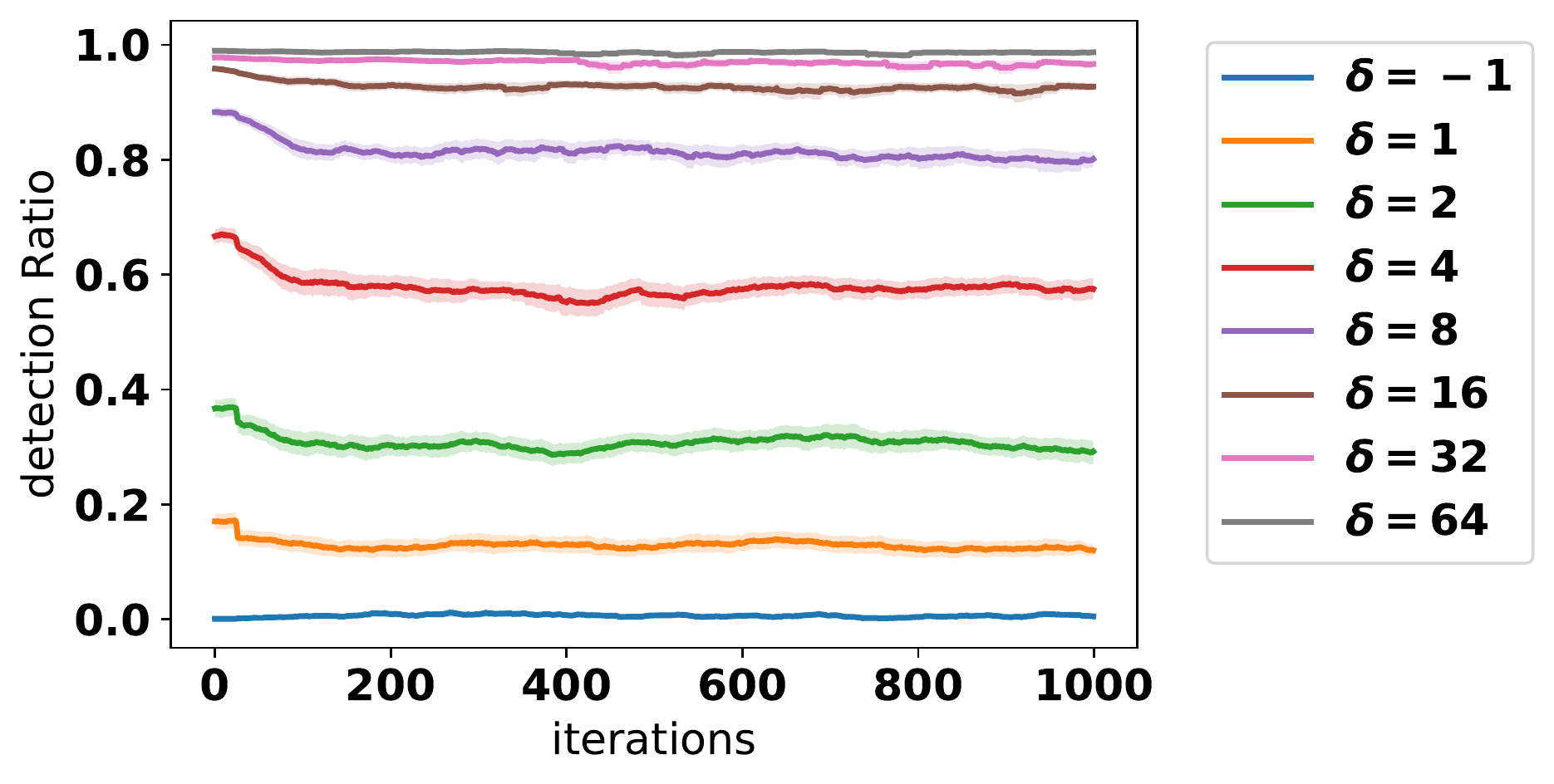}
		\caption{Detection Ratio on Ant}
		\label{fig:detection_ratio_Ant}
	\end{subfigure}
	\begin{subfigure}{0.305\textwidth}
		\centering
		\includegraphics[width=1\columnwidth]{figures/Humanoid-v3_TRPO_reward_plot}
		\caption{TRPO on Humanoid}
		\label{fig:TRPO_Humanoid}
	\end{subfigure}
	\begin{subfigure}{0.305\textwidth}
		\centering
		\includegraphics[width=1\columnwidth]{figures/Humanoid-v3_FPG_reward_plot}
		\caption{FPG on Humanoid}
		\label{fig:FPG_Humanoid}
	\end{subfigure}
	\begin{subfigure}{0.39\textwidth}
		\centering
		\includegraphics[width=1\columnwidth]{figures/Humanoid-v3_FPG_detection_plot}
		\caption{Detection Ratio on Humanoid}
		\label{fig:detection_ratio_Humanoid}
	\end{subfigure}
	\caption{Detailed Results on the MuJoCo benchmarks.}
	\label{fig:full_result}
\end{figure*}

\section{Implementation Details of \texttt{FPG-TRPO}}\label{sec:code}
\begin{algorithm}[!t]
	\begin{algorithmic}[1]	
		\State {\bf{Input:}} initial policy parameter $\theta_0$; initial value function parameter $\phi_0$.
		\State {\bf{Hyperparameters:}} KL-divergence limit $\delta$; backtracking coefficient $\alpha$; maximum number of backtracking steps $K$; upper-bound of corruption level $\epsilon$; episode length $H$; batch size $M$.
		\For{$k=0,1,\ldots$}
		\State Collect set of $M$ trajectories $D_k = \{\tau_i\}_{1:M}$ by running policy $\pi_k = \pi(\theta_k)$ in the environment.
		\State Compute rewards-to-go $\hat R_{t,i} = \sum_{h=t}^H \gamma^{h-t} r_{h,i}$. 
		\State Using GAE to compute advantage estimate $\hat A_{t,i}$ based on the current value function $V_{\phi_k}$.
		\State Compute and save $\hat g_{t,i} = \nabla_\theta \log \pi_\theta(a_{t,i},s_{t,i})\vert_{\theta_k}$ for all $t=1:H$ and $i=1:M$.
		\State Call the filtered conjugate gradient algorithm in Alg.~\ref{alg:fcg} to get $S_k\subset [M]\times [H], \hat x_k = FCG(\hat g_{t,i}, \hat A_{t,i})$.
		\State Compute policy gradient estimate $\hat g_k = \frac{1}{|S_k|}\sum_{(t,i)\in S_k}\hat g_{t,i}\hat A_{t,i}$.
		\State Update the policy by backtracking line search with
		\begin{equation}
			\theta_{k+1} = \theta_k + \alpha^j \sqrt{\frac{2\delta}{\hat x_k \hat g_k}} \hat x_k
		\end{equation}
		where $j\in \{0,1,2,...,K\}$ is the smallest value which improves the sample loss and satisfies the sample KL-divergence constraint.
		\State Fit the value function by regression on mean-squared error on the filtered trajectories $S_k$:
		\begin{equation}
			\phi_{k+1} = \argmin_\phi \frac{1}{|S_k|}\sum_{(t,i)\in S_k}\left(V_\phi(s_{t,i})-\hat R_{t,i}\right)^2
		\end{equation}
		In practice, one often only take a few gradient steps in each iteration $k$, instead of optimizing to convergence.
		\EndFor
	\end{algorithmic}
	\caption{\texttt{FPG-TRPO}}
	\label{alg:fpg_TRPO}
\end{algorithm}

\begin{algorithm}[!t]
	\begin{algorithmic}[1]	
		\State {\bf{Input:}} $\hat g_{t,i}, \hat A_{t,i}$
		\State {\bf{Hyperparameters:}} Number of iterations $r$ (default $r=4$), fraction of data filtered in each iteration $p$ (default $p = \epsilon/2$, i.e. filter out $2\epsilon$ data in total).
		\State Initialize $S = \{1,2,\ldots,M\}$.
		\For{$k=1,\ldots,r$}
		\State Call standard CG to solve for $\hat x = \hat F^{-1}\hat g$, where $\hat F = \frac{1}{S}\sum_{(t,i)\in S} \hat g_{t,i}\hat g_{t,i}^\top$ and $\hat g = \frac{1}{S}\sum_{(t,i)\in S} \hat g_{t,i}\hat A_{t,i}$.
		\State Compute the residues $r_{t,i} = \hat g_{t,i}\hat g_{t,i}^\top\hat x-\hat g_{t,i}\hat A_{t,i}$ for $(t,i)\in S$ and save in a matrix $G$ of size $d\times |S|$.
		\State Let $v$ be the top right singular vector of $G$.
		\State Compute the vector $\tau$ of \emph{outlier scores} defined via
		$\tau_{t,i} = \left(r_{t,i} ^\top v\right)^2$.
		\State Remove $(HMp)$ number of $(t,i)$ pair with the largest outlier scores from $S$.
		\EndFor
		\State Call standard CG one more time and return $(S, \hat x)$.
	\end{algorithmic}
	\caption{Filtered Conjugate Gradient (FCG)}
	\label{alg:fcg}
\end{algorithm}

In the experiment, we use a TRPO variant of FPG implementation, which differs from Alg.~\ref{alg:q_npg_sample} in several ways:
\begin{enumerate}
	\item Most existing TRPO implementation uses the conjugate gradient (CG) method instead of linear regression to solve for the matrix inverse vector product problem. We follow this convention and design FPG-TRPO to use a filtered conjugate gradient (FCG) subroutine to replace the standard CG produce. The FPG procedure is detailed in Alg.~\ref{alg:fcg}. At a high level FCG performs a filtering algorithm (a.k.a. outlier removal)  on the residues of CG with respect to each data point.
	\item Again following existing TRPO implementations, FPG-TRPO builds another network to estimate the value function for the purpose of variance reduction, effectively resulting in an actor-critic algorithm. Instead of performing robust learning procedure on both policy and value function learning, we perform the main filtering algorithm on the policy learning procedure (the CG step discussed above), which also returns a filtered subset of data as a by-product. We then use this filtered subset of data to perform the rest of the learning procedure, including value function update and the sample loss estimation in backtracking line search. This allows us to perform the robust learning procedure only once per PG iteration.
	\item FPG-TRPO uses a deterministic variant of the filtering algorithm suggested in \cite{diakonikolas2019sever}, which empirically performs better and is simpler to implement than the stochastic variant used for theoretical analysis. Specifically, the filtering algorithm will simply remove a fixed fraction of points with the largest deviation along the top singular value direction (step 9 of Alg.~\ref{alg:fcg}).
\end{enumerate} 
The pseudo-code of \texttt{FPG-TRPO} can be found in Alg.~\ref{alg:fpg_TRPO}. Similar to the NPG variant of FPG, the only difference between Alg.~\ref{alg:fpg_TRPO} and a standard TRPO implementation is the replacement of the CG subroutine with the FCG subroutine. This modular implementation allows one to easily replace Alg.~\ref{alg:fcg} with any state-of-the-art robust CG procedure in the future. Table \ref{table:hyperparameters} lists all the hyper-parameters we used in our experiments, which are taken from open-source implementations of TRPO tuned for the MuJoCo environments. Our code to reproduce the experiment result is included in the supplementary material and will be open-sourced. Finally, Figure \ref{fig:full_result} presents the detailed results on all experiments, completing the partial results shown in Figure \ref{fig:humanoid_result}.
\begin{table*}[t]
	\centering
	\begin{tabular}{| l | c | l |} 
		\hline
		Parameters & Values & Description\\ [0.5ex] 
		\hline\hline
		$\gamma$ & $0.995$ & discounting factor.\\
		$\lambda$ & 0.97 & GAE parameter \cite{schulman2015high}.\\
		l2-reg & $0.001$ & L2 regularization weight in value loss.\\
		$\delta$ & $0.01$ & KL constraint in TRPO.\\
		damping & $0.1$ & damping factor in conjugate gradient.\\
		batch-size & 25000 & number of time steps per policy gradient iteration.\\
		$\alpha$ & 0.5 & backtracking coefficient.\\
		$K$ & 10 & maximum number of backtracking steps.\\
		\hline
	\end{tabular}
	\caption{Hyperparameters for FPG-TRPO.}
	\label{table:hyperparameters}
\end{table*}
\end{document}